\documentclass[14pt]{article}
\usepackage{microtype}
\usepackage{graphicx}
\usepackage{subfigure}
\usepackage{booktabs} 
\usepackage{wrapfig}

\usepackage[colorlinks,
linkcolor=red,
anchorcolor=blue,
citecolor=blue
]{hyperref}

\usepackage{smile}
\usepackage{algpseudocode}
\usepackage{enumitem}
\usepackage{fullpage}
\usepackage{tikz}
\usepackage{makecell}
\usepackage{soul}
\usepackage{authblk}

\def\tv{\text{TV}}
\def\vareps{\varepsilon}
\def\eps{\epsilon}
\def\mis{\textsf{mis}}
\def\vf{\textsf{vf}}

\def\dr{\textsf{dr}}
\def\ci{\textsf{conf}}
\def\subopt{\textsf{SubOpt}}
\newcommand{\indp}{\perp\!\!\!\!\perp} 
\def\pdim{\mathfrak{C}}
\def\stat{\text{stat}}

\title{Offline Reinforcement Learning with Instrumental Variables in Confounded Markov Decision Processes}

\author[1]{Zuyue Fu}
\author[2]{Zhengling Qi}
\author[1]{Zhaoran Wang}
\author[3]{Zhuoran Yang}
\author[4]{Yanxun Xu}
\author[5]{Michael R. Kosorok}
\affil[1]{Northwestern University}
\affil[2]{George Washington University}
\affil[3]{Yale University}
\affil[4]{Johns Hopkins University}
\affil[5]{University of North Carolina at Chapel Hill}

\date{}

\begin{document}

\maketitle

\begin{abstract}
    We study the offline reinforcement learning (RL) in the face of unmeasured confounders. Due to the lack of online interaction with the environment, offline RL is facing the following two significant challenges: (i) the agent may be confounded by the unobserved state variables; (ii) the offline data collected a prior does not provide sufficient coverage for the environment. To tackle the above challenges, we study the policy learning in the confounded MDPs with the aid of instrumental variables. Specifically, we first establish value function (VF)-based and marginalized importance sampling (MIS)-based identification results for the expected total reward in the confounded MDPs. Then by leveraging pessimism and our identification results, we propose various policy learning methods with the finite-sample suboptimality guarantee of finding the optimal in-class policy under minimal data coverage and modeling assumptions. Lastly, our extensive theoretical investigations and one numerical study motivated by the kidney transplantation demonstrate the promising performance of the proposed methods. 

\end{abstract}

\section{Introduction}

Reinforcement learning (RL, \cite{sutton2018reinforcement}) with deep neural networks gains tremendous successes in practice, e.g., games \citep{silver2016mastering, OpenAI_dota}, robotics \citep{kalashnikov2018scalable}, precision medicine \citep{kosorok2019precision,cho2022}.  In many application domains, actively collecting data through interacting with the environment in an online fashion is usually either expensive or unethical, e.g.,  healthcare \citep{raghu2017continuous, komorowski2018artificial,gottesman2019guidelines} and autonomous driving \citep{shalev2016safe}). Therefore, 
a growing body of literature focus on designing RL methods in the offline setting, where the agent aims to learn an optimal policy $\pi^*$ in the infinite-horizon Markov decision process (MDP, \cite{puterman2014markov}) only through observational data, which consists of $N$ trajectories generated by a behavior policy $b$ with a finite horizon $T$. 

However, applying RL methods in the offline setting still possesses the following challenges: 
(i) The agent may be confounded by unmeasured variables (confounders) in the observational data. We refer to the MDP with such unmeasured confounders as confounded MDP. Such confounders usually come from private data or heuristic information not recorded \citep{brookhart2010confounding}. In the confounded MDP, the causal effects of actions on the transitions and rewards are not identifiable from the observational data, leaving most offline RL methods assuming unconfoundedness not applicable in our setting.  
(ii) To learn an optimal policy from observational data, many prior methods \citep{precup2000eligibility, antos2008learning, chen2019information} require a data coverage assumption for any policy $\pi$, i.e., 
the density ratio between the state-action visitation measure induced by $\pi$ and that induced by the behavior policy $b$ is uniformly upper bounded for any $\pi$. However, such a data coverage assumption is hardly satisfied in practice, especially when the state or action spaces are large. 
Further, many existing methods developed under this assumption are not stable or even do not converge when the assumption is violated \citep{wang2020statistical, wang2021instabilities}. 

To tackle the above challenge (i), we study the confounded MDP via instrumental variables (IVs, \cite{angrist1996identification}). 
Informally, IVs are variables that affect the transitions and rewards only through actions. 
With the aid of IVs, we introduce two types of identification results: value function (VF)-based and marginalized importance sampling (MIS)-based. 
Specifically, with only finite-horizon data, VF-based identification result helps identify the state-value function in the infinite-horizon confounded MDP and establish a new Bellman equation by leveraging IVs. On the other hand, we establish the MIS-based identification result for directly estimating the expected total reward $J(\pi)$ to be maximized. Both identification results rely on the memoryless assumption on the unmeasured confounders. The memoryless assumption rules out the existence of unmeasured confounders that can affect future rewards and unmeasured confounders but only the immediate reward. 

To tackle the above challenge (ii), we employ pessimism to achieve policy learning and systematically study its theoretical properties. 
Specifically, when using VF-based identification, we first formulate a min-max estimator of the state-value function via the newly established estimating equation. Then, we construct a confidence set of such a min-max estimator based on its loss function, so that the true state-value function lies within the confidence set with a high probability. Note that the construction of our confidence set does not require to develop uniform bands for related estimators, which is different from many existing pessimistic algorithms \citep[e.g.,][]{jin2021pessimism,rashidinejad2021bridging,yan2022efficacy}. Finally, we search for the best policy that maximizes the most conservative expected total reward associated with the estimated state-value function within the confidence set. 
As a theoretical contribution, under the data coverage assumption only for an optimal (in-class) policy $\pi^*$ and realizability assumption for all policies, we show that the suboptimality of the learned best policy is upper bounded by $O(\log(NT)(NT)^{-1/2})$, i.e., the regret of our algorithm in finding the optimal policy converges to $0$ as long as either the number of trajectories $N$ or the number of decision points at each trajectory $T$ goes to infinity.  It is worth noting that our theoretical analysis does not assume that the observational data are generated from stationary distribution or even independent, which has been widely imposed in related literature \citep{farahmand2016regularized, nachum2019dualdice, tang2019doubly, xie2021bellman, kallus2022efficiently}, and thus is more general than the aforementioned literature. Without imposing such a restrictive assumption, inspired by \cite{wang2021projected}, our convergence analysis relies on  novel concentration inequalities for geometrically ergodic sequences, which significantly increases the applicability of our results in practice. 
In the meanwhile, pessimism with MIS-based identification achieves a similar result by imposing a realizability assumption only for an optimal policy $\pi^*$ and data coverage assumption for all policies. 
Furthermore, by combining VF- and MIS-based identification results, we propose a doubly robust (DR) estimator for learning the optimal policy $\pi^*$. Theoretically, for such a DR estimator, we show a similar suboptimality at a rate of $O(\log(NT)(NT)^{-1/2})$, but only requiring that either the assumptions imposed in VF-based method or those imposed in MIS-based one hold.  Lastly, our proposed algorithms have different requirements on the identifiability of the expected total rewards $J(\pi)$. Specifically, for the VF-based algorithm we propose, we only require the offline data distribution can identify $J(\pi^*)$ uniquely rather than other policies. On the contrary, the MIS-based pessimistic algorithm requires all $J(\pi)$ be uniquely identified by our offline data distribution, which is thus stronger. Interestingly both approaches do not require the identifiability of all associated nuisance parameters. See more detailed discussion in Section \ref{sec: identification}.  In Table \ref{table:summary-assumptions} below,  we summarize our main assumptions/requirements for the theoretical guarantees of the proposed algorithms.

\begin{table}[h!]
\centering
\begin{tabular}{| c | c | c | c | } 
    \hline
    Methods & Data Coverage & Realizability & Identifiability \\ 
    \hline
    VF-based & $\|w^{\pi^*}\|_\infty \leq C_*$ & $w^{\pi^*}\in \cW$, $V^\pi\in \cV~\forall \pi\in \Pi$ & $J(\pi^*)$ is identifiable \\
    \hline
    MIS-based & $\|w^\pi\|_\infty \leq C_*~\forall\pi\in \Pi$ & $V^{\pi^*}\in \cV$, $w^\pi\in \cW~\forall \pi\in \Pi$ & $J(\pi)$ is identifiable $\forall\pi\in \Pi$ \\
    \hline
    DR-based & \multicolumn{3}{c|}{\makecell{$\|w^{\pi^*}\|_\infty \leq C_*$, $w^{\pi^*}\in \cW$, $V^\pi\in \cV~\forall\pi\in \Pi$, and $J(\pi^*)$ is identifiable; \\ \textit{\textbf{OR}} $\|w^\pi\|_\infty \leq C_*~\forall\pi\in \Pi$, $V^{\pi^*}\in \cV$, $w^\pi\in \cW~\forall \pi\in \Pi$, and $J(\pi)$ is identifiable $\forall\pi\in \Pi$}} \\
    \hline
\end{tabular}
\caption{Assumptions required by our VF-, MIS-, and DR-based methods, where $w^\pi$ is the density ratio between visitation measures induced by the policy $\pi$ and the behavior policy $b$ (see \eqref{eq:iv-ratio-def} for a detailed definition), and $V^\pi$ is the state-value function of the policy $\pi$. Here $\cV, \cW$ and $\Pi$ are function classes, and $C_*$ is some generic constant.}
\label{table:summary-assumptions}
\end{table}

\vskip5pt
\noindent\textbf{Contribution.}
As a summary, our contribution is three-fold. First, by leveraging IVs, we provide VF- and MIS-based identification results under the confounded MDP. Second, by employing pessimism based on the loss functions used for estimating nuisance parameters, we innovatively construct estimators of the optimal policy $\pi^*$ via VF- and MIS-based identification. Further, by combining VF- and MIS-based identification, we propose a DR-based algorithm for estimating $\pi^*$. Third, under mild conditions on data coverage and realizability, we show that the suboptimalities of the proposed algorithms in finding the optimal in-class policy are upper bounded by $O(\log(NT)(NT)^{-1/2})$, without requiring that the observational data are generated from stationary distribution or even independent. The success of our algorithm relies on novel constructions of confidence sets for related nuisance parameters and maximal inequalities of geometrically ergodic sequences over function classes, which may be of independent interest.

\vskip5pt
\noindent\textbf{Related Work.}
~Our work is related to the line of works that study RL under the presence of unmeasured confounders. \cite{zhang2019near} propose an online RL method to solve dynamic treatment regimes in a finite-horizon setting with the presence of confounded observational data. Their method relies on sensitivity analysis, which constructs a set of possible models based on the confounded observational data to obtain partial identification. Also, to incorporate the observational data into the finite-horizon RL, \cite{wang2021provably} propose deconfounded optimistic value iteration, which is an online algorithm with a provable regret guarantee. To ensure the identifiability through the observational data, they impose the backdoor criterion \citep{pearl2009causality, peters2017elements} when confounders are partially observed, and also the frontdoor criterion when confounders are unmeasured. Our work is also closely related to \cite{liao2021instrumental} and \cite{chen2021estimating}. Specifically, \cite{liao2021instrumental} propose an IV-aided value iteration algorithm to study confounded MDPs in the offline setting. It is worth mentioning that they only consider the finite-horizon MDP, where the transition dynamics is a linear function of some known feature map. To ensure identifiability, they assume that the unmeasured confounders are Gaussian noise, which does not affect the immediate reward and only affect the transition dynamics in an additive manner. In contrast, we consider an infinite-horizon confounded MDP without such restrictive assumptions on the structure of the model, which brings significant technical challenges. On the other hand, \cite{chen2021estimating} study the partial identification using IVs for improving dynamic treatment regimes and require the data coverage assumption on all the policies. In contrast, we establish point identification results with the help of IVs and our proposed algorithm is valid under only partial coverage, which is thus more appealing. In addition,  with unmeasured confounders, \cite{kallus2020confounding} study off-policy evaluation in the infinite-horizon setting based on sensitivity analysis, which imposes additional assumptions on how strong the unmeasured confounding can possibly be. In the meanwhile, to ensure identifiability, \cite{namkoong2020off} consider the case where the unmeasured confounders affect only one of the decisions made. Very recently, there is a stream of research focused on using proximal causal inference \citep{tchetgen2020introduction} for off-policy evaluation and learning in the partially observed MDP \citep{bennett2021off,shi2021minimax,lu2022pessimism}.

Our work is also related to the line of research on policy evaluation and policy learning in the offline setting assuming unconfoundedness. In terms of off-policy evaluation, most works either employs a VF-based method \citep{ernst2005tree,ertefaie2018constructing, shi2020statistical, liao2021off, uehara2021finite,zhou2021estimating}, or an MIS-based method \citep{liu2018breaking, nachum2019dualdice, zhang2020gendice, wang2021projected, uehara2021finite}. Our work is also related to those that propose DR estimators in off-policy evaluation. See \cite{jiang2016doubly, thomas2016data, tang2019doubly, kallus2020double, uehara2020minimax, kallus2022efficiently} for this line of research. As for policy learning in the offline setting, \cite{munos2008finite,antos2008learning} and \cite{luckett2019estimating} prove that fitted value and policy iterations converge to an optimal policy under the data coverage assumption and realizability assumption for all policies. 
By employing pessimism, \cite{xie2021bellman} guarantee a near-optimal policy under the realizability and the completeness assumptions for all policies, while \cite{jiang2020minimax} provide a similar guarantee under the data coverage assumption for the optimal policy and the realizability assumption for all policies. 
More recently, \cite{zhan2022offline} claims that they can learn a near-optimal policy under the data coverage and realizability assumptions for the optimal policy. Their method is built upon a regularized version of the LP formulation of MDPs and thus working on a class of regularized policies. However, due to regularization, the policy learned by \cite{zhan2022offline} is typically suboptimal even given infinite data. Moreover, their realizability assumption is imposed on the regularized value function, making it difficult to interpret and compare with other works. In contrast, our method is proposed in a non-regularized setting and is valid given standard realizability and minimal data coverage assumptions, even under the presence of unmeasured confounders and dependent observational data.

\vskip5pt
\noindent\textbf{Roadmap.} 
In Section \ref{sec:iv-background}, we introduce the background of confounded MDPs and their assumptions. 
In Section \ref{sec:iv-id}, we introduce VF- and MIS-based identification results. In Section \ref{sec:iv-pess}, by employing pessimism, we first introduce estimators of the optimal policy via VF- and MIS-based identification, then by combining both results, we introduce a DR estimator. Theoretical results upper bounding the suboptimalities of the proposed estimators are presented in Section \ref{sec:iv-theory}. In Section \ref{sec: dual form}, we provide dual formulations for our proposed algorithms under one additional assumption so that all estimated policy can be efficiently computed. Lastly, in Section \ref{sec: identification}, we discuss the implications of our algorithms on the identifiability related to total rewards and associated nuisance parameters. All technical proofs are provided in the Supplementary Material. In addition, in the Supplementary Material, we demonstrate the usefulness of the proposed methods by conducting a simulation study, in which we simulate a synthetic dataset that  mimics a real-world electronic medical record dataset for kidney transplantation patients \citep{hua2020personalized}.

\section{Confounded Markov Decision Processes}
\label{sec:iv-background}
In this section, we introduce the framework of confounded Markov decision processes with discrete instrumental variables. We aim to leverage the batch data to find an optimal in-class policy that maximizes the expected total rewards. 

\vskip5pt
\noindent\textbf{Confounded MDPs}
In a confounded MDP, we observe $\{S_t, A_t, R_t\}_{t\geq 0}$ for each trajectory, where $S_t$ is the observed state, $A_t$ is the action taken after observing $S_t$, and $R_t$ is the immediate reward received after making an action $A_t$ for $t\geq 0$. We denote by $\cS$ and $\cA$ the state and action spaces, respectively. 
Furthermore, we assume that at each decision point $t\geq 0$, there exist some unmeasured state variables $U_t \in \cU$, which may confound the effect of action $A_t$ on the rewards and future transitions. Due to such unobserved confounders, the (causal) effect of the action on the immediate and future rewards may not be non-parametrically identified and directly applying standard RL algorithms for MDPs will produce sub-optimal policies. 

To address this concern, we study the confounded MDP via the instrumental variable (IV) method \citep{angrist1995identification}, which has been widely used in the literature of causal inference (e.g., \cite{pearl2009causality,hernan2010causal}) to identify the causal effect of a treatment under unmeasured confounding. 
Specifically, at each decision point $t$, we further assume that we also observe a time-varying IV $Z_t\in \cZ$, which is independent of $U_t$ and does not have a direct effect on the immediate reward $R_t$ and all future states, actions, and rewards. 
With such an IV, we observe $\{S_t, Z_t, A_t, R_t\}_{t\geq 0}$ for each trajectory in the confounded MDP.


In this work, we consider finite action and IV spaces, i.e., $\cA = \{a_j\}_{j\in [K]}$ and $\cZ = \{z_j\}_{j\in [K]}$, where $K\geq 2$ is an integer. Furthermore, we consider a simplex encoding for both actions and IVs which enjoy a nice interpretation \citep{zhang2014multicategory}. Specifically, for any $j\in [K]$, we let
\#\label{eq:iv-simplex-encoding}
a_j = z_j = 
\begin{cases}
    (K-1)^{-1/2}\mathbf{1}_{K-1} & \text{if $j = 1$,}  \\
    \frac{(1 + \sqrt{K} \mathbf{1}_{K-1} )}{(K-1)^{3/2}} + \sqrt{\frac{K}{K-1}} e_{j-1} & \text{if $2\leq j \leq K$,}
\end{cases}
\#
where $\mathbf{1}_{K-1}\in \RR^{K-1}$ is an all-one vector and $e_j\in \RR^{K-1}$ is a vector with all elements $0$ except $1$ for $j$-th position. By the simplex encoding in \eqref{eq:iv-simplex-encoding}, one can see that $\sum_{j \in [K]}a_j = \sum_{j \in [K]}z_j = 0$ and $a_i^\top a_j = z_i^\top z_j = -\ind\{i\neq j\}/(K-1) + \ind\{i= j\}$ for any $i,j\in [K]$, where $\ind\{\cdot \}$ is an indicator function.  We remark that any other reasonable encoding mechanisms can be adopted here and our results still hold. 

\vskip5pt
\noindent\textbf{Value Function and Performance Metric.}
In the confounded MDP, we aim to find an optimal in-class policy $\pi^*\in \Pi$ such that $\pi^*$ maximizes the expected total rewards, where $\Pi$ is a class of time-homogeneous policies mapping from the observed state space $\cS$ into the probability distribution over the action space $\cA$. In particular, $\pi(a \given s)$ refers to the probability of choosing action $a \in \cA$ given the state value $s \in \cS$. Formally, for any $\pi\in \Pi$, we define the value function $V^\pi$ and the expected total reward $J(\pi)$ as follows, 
\#\label{eq:iv-val-func}
& V^\pi(s) = \EE_{\pi} \left[ \sum_{t = 0}^\infty \gamma^t R_t \Biggiven S_0 = s \right ], \qquad J(\pi) = (1-\gamma ) \cdot \EE_{S_0 \sim \nu} \left [ V^\pi(S_0) \right ], 
\#
where the expectation $\EE_\pi[\cdot]$ is taken with respect to the distribution such that the action $A_t\sim \pi(\cdot \given S_t)$ for any $t\geq 0$, and $\nu$ is a known reference distribution over $\cS$. Given the definition of $J(\pi)$ in \eqref{eq:iv-val-func}, our goal is to leverage the batch data to estimate $\pi^*$, where
$$
\pi^* \in \argmax_{\pi \in \Pi} J(\pi).
$$ 
Suppose the batch data we have collected consist of $N$ independent and identically distributed copies of $\{S_t, Z_t, A_t, R_t\}_{t \geq 0}$ with a total number $T$ of decision points for each trajectory. Then we can summarize our batch data as $\cD = \{\{S_t^i, Z_t^i, A_t^i, R_t^i, S_{t+1}^i\}_{t = 0}^{T-1}\}_{i\in [N]}$. Assuming that the number of decision points is the same at each trajectory is for simplicity. Indeed the proposed method and theoretical results below remain valid as long as the number of decision points at each trajectory stays within a suitable range and at the same order asymptotically. 
We define the performance metric as
\$
\subopt(\pi) = J(\pi^*) - J(\pi), 
\$
which characterizes the suboptimality of a policy $\pi$ compared with the optimal in-class policy $\pi^*$.

\vskip5pt
\noindent\textbf{Why is Confounded MDP Challenging?} In the standard MDP \citep{sutton2018reinforcement}, all states are assumed fully observed and the trajectory $\{S_t, A_t, R_t\}_{t\geq 0}$ satisfies the Markovian property. By leveraging the celebrated Bellman equation, under some mild conditions, one can non-parametrically identify $J(\pi)$ for any $\pi\in \Pi$, which serves as a foundation for many existing RL algorithms. However, in the confounded MDP, due to the unobserved states, the effect of actions on the rewards and future states cannot be identified even if we include all past history information at each decision point $t \geq 0$. Therefore, additional assumptions are needed to identify $J(\pi)$ and in this work, we rely on the IV to deal with such challenges.

\vskip5pt
\noindent\textbf{Notation.} 
Throughout the paper, we denote by $c$ a positive absolute constant, which may vary from lines to lines. 
Without further explanation, we denote by $\EE_\pi[\cdot]$ the expectation taken with respect to the trajectory generated by the policy $\pi$, $\EE[\cdot]$ the expectation taken with respect to the trajectory generated by the behavior policy, and $\hat \EE[\cdot]$ the empirical average across all $N$ trajectories.

\section{Assumptions and Identification Results}
In this section, we introduce several assumptions to help us identify $J(\pi)$ for any $\pi \in \Pi$ by using an IV. The first assumption is related to the trajectory $\{S_t, U_t, A_t, R_t\}_{t\geq 0}$, where we model it by a time-homogenous MDP.
\begin{assumption}\label{ass:MDP}
    The following statements hold. 
    \begin{enumerate}[label=(\alph*)]
        \item For any $t \geq 1$, we have $(S_{t+1}, U_{t+1})\indp \{S_j, U_j, A_j\}_{0\leq j<t} \given (S_t, U_t, A_t)$ and the transition probability is time-homogeneous;  
        \item \label{ass:iv-reward} For any $t\geq 0$, we have $R_t = R(U_t, S_t, A_t, S_{t+1}, U_{t+1})$ for some deterministic function $R\colon \cU\times \cS\times\cA\times\cS\times\cU\to \RR$. Also, we assume $|R_t| \leq 1$ almost surely for any $t\geq 0$; 
        \item The offline dataset $\cD$  is generated by an unknown initial distribution $\zeta$ over $\cS$ and a stationary policy $b$, which is a function mapping from $\cS \times \cU \times \cZ$ into a probability distribution over $\cA$. 
    \end{enumerate}
\end{assumption}
Here $b$ is often called behavior policy in RL literature. Assumption \ref{ass:MDP} is standard in the literature of RL, which is mild as $\{U_t\}_{t\geq0}$ is unobserved. The known reward structure in Assumption \ref{ass:MDP}~\ref{ass:iv-reward} is always satisfied as one can put the observed reward $R_t$ as a part of information in the next state. The uniformly bounded assumption on the reward $R_t$ is used to simplify the technical analysis and can be relaxed by imposing some high-order moment condition on $R_t$ instead. Due to the unobserved state variables $U_t$, we make the following IV assumptions.
\begin{assumption}\label{ass:iv-common}
    The following statements hold. 
    \begin{enumerate}[label=(\alph*)]
        \item \label{ass:4}
        For any $t\geq 0$, we have $(S_{t+1}, U_{t+1})\indp Z_t\given (S_t, U_t, A_t)$; 
        \item \label{ass:5}
        For any $a\in \cA$ and $t\geq 0$, we have $\PP(A_t = a\given S_t, Z_t)\neq \PP(A_t = a\given S_t)$; 
        \item \label{ass:iv-zu-ind}
        For any $t \geq 0$, we have $Z_t \indp U_t \given S_t$ and the probability distribution of $Z_t$ given $S_t$ is time-homogeneous. 
        \item \label{ass:iv-compliance}
        For any $t\geq 0$, we have the behavior policy satisfy that
        \$
        & b(A_t = a\given S_t, U_t, Z_t = a) - \frac{1}{K-1}\sum_{z\in \cZ, z\neq a} b(A_t = a\given S_t, U_t, Z_t = z) \\
        & \qquad = b(A_t = a\given S_t, Z_t = a) - \frac{1}{K-1}\sum_{z\in \cZ, z\neq a} b(A_t = a\given S_t, Z_t = z) = \Delta^*(S_t, a), 
        \$
        i.e., the compliance $\Delta^*(S_t,a)$ defined above is independent of the unobserved confounder $U_t$ almost surely. 
    \end{enumerate}
\end{assumption}

Assumption \ref{ass:iv-common}\ref{ass:4} states that there is no direct effect of the IV $Z_t$ on the future states and rewards except through the action $A_t$, which is a typical assumption in the literature of causal inference with IVs \citep{angrist1995identification, angrist1996identification}. Note that by Assumption \ref{ass:MDP}~\ref{ass:iv-reward}, we have implicitly restricted the effect of $Z_t$ on the reward $R_t$ only through $A_t$ in this assumption. Assumption \ref{ass:iv-common}~\ref{ass:5} requires that the IV $Z_t$ will influence the action $A_t$, which is called IV relevance in the causal inference.  Assumption \ref{ass:iv-common}\ref{ass:iv-zu-ind}, corresponding to IV independence, ensures that the effect of $Z_t$ on futures states and rewards is unconfounded by adjusting the current state $S_t$. The homogeneous assumption on the conditional distribution of $Z_t$ given $S_t$ is imposed here as our target parameter $J(\pi)$ is defined over the infinite horizon. Define a function 
$$
\Theta^*(s,z) = \PP(Z_t = z\given S_t=s)
$$
for every $(s, z) \in \cS \times \cZ$, which is independent of the decision point due to such time-homogeneity. 
In addition, Assumption \ref{ass:iv-common}\ref{ass:iv-compliance} essentially indicates that there is no interaction between $U_t$ and $Z_t$ in affecting whether the action $A_t$ will comply with $Z_t$ or not. This so-called independent compliance assumption has been widely adopted in identifying the average treatment effect of binary treatments in causal inference \citep{wang2018bounded,cui2021semiparametric}. Here we generalize it to the setting of multiple treatments and instrumental variables, which may be of independent interest. A graphical illustration of Assumptions \ref{ass:MDP} and \ref{ass:iv-common} is presented in Figure \ref{fig:mdp}, which also illustrates how the offline data in the confounded MDP are generated. Moreover, we provide a numerical example in \S\ref{sec:exppppp} of the Supplementary Material to illustrate the IV structure of our data generating process.  In the following section, we introduce value function (VF)-based identification and marginalized importance sampling (MIS)-based identification, respectively. 

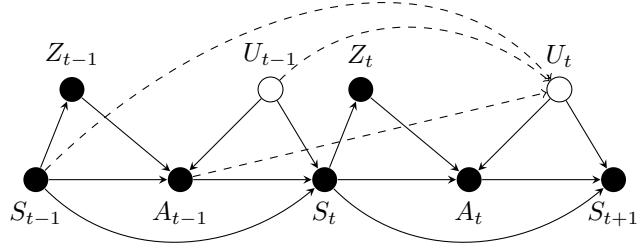
\begin{figure}[htbp]
    \centering
\begin{tikzpicture}[scale=1.2]
        \node (so) at (-4.8,0) [label=below:{$S_{t-1}$}, circle, fill=black]{};
        \node (zo) at (-4.4,1) [label=above:{$Z_{t-1}$}, circle, fill=black]{};
        \node (uo) at (-2.2,1) [label=above:{$U_{t-1}$}, circle, draw]{};
        \node (ao) at (-3.2,0) [label=below:{$A_{t-1}$}, circle, fill=black]{};
        \node (s) at (-1.6,0) [label=below:{$S_t$}, circle, fill=black]{};
        \node (z) at (-1.2,1) [label=above:{$Z_t$}, circle, fill=black]{};
        \node (sp) at (1.6,0) [label=below:{$S_{t+1}$}, circle, fill=black]{};
        \node (u) at (1,1) [label=above:{$U_t$}, circle, draw]{};
        \node (a) at (0,0) [label=below:{$A_t$}, circle, fill=black]{};
    \draw[-stealth] (ao) edge (s);
    \draw[-stealth] (so) edge (ao);
    \draw[-stealth] (zo) edge (ao);
    \draw[-stealth] (uo) edge (ao);
    \draw[dashed,->] (so) [out=-315,in=130] edge (u);
    \draw[dashed,->] (ao) edge (u);
    \draw[-stealth] (uo) edge (s);
    \draw[-stealth] (so) [out=-45,in=220] edge (s);
    \draw[-stealth] (so) edge (zo);
    \draw[-stealth] (s) edge (z);
    \draw[dashed,->] (uo) [out=45,in=140] edge (u);

    \draw[-stealth] (a) edge (sp);
    \draw[-stealth] (s) edge (a);
    \draw[-stealth] (z) edge (a);
    \draw[-stealth] (s) [out=-45,in=220] edge (sp);
    \draw[-stealth] (u) edge (a)
                        edge (sp);
\end{tikzpicture}
\caption{A graphical illustration of the confounded MDP satisfying Assumptions \ref{ass:MDP} and \ref{ass:iv-common}. Here $\{Z_t\}_{t\geq 0}$ are IVs, which satisfy IV independence, i.e., $Z_t \indp U_t \given S_t$. Also the trajectory $\{S_t, U_t, A_t\}_{t\geq 0}$ satisfies the Markov property, i.e., $(S_{t+1}, U_{t+1})\indp \{S_j, U_j, A_j\}_{0\leq j<t} \given (S_t, U_t, A_t)$. In the meanwhile, the dependence of $U_t$ on $U_{t-1}$ and other variables given $S_t$ (dashed arrows) is prohibited by Assumption \ref{ass:10}. }
\label{fig:mdp}
\end{figure}

\label{sec:iv-id}

\subsection{Value Function-based Identification}
\label{sec:iv-vf-id}
In the unconfounded MDP, value function defined in \eqref{eq:iv-val-func} can be used to identify $J(\pi)$ and itself can be identified via the Bellman equation. However, due to the existence of unobserved confounders, the regular Bellman equation, which relies on the Markovian assumption, does not hold in general and the effect of actions on the reward cannot be identified either. Fortunately, by leveraging the IV, we are able to provide a way to identify $J(\pi)$ via the state-value function $V^\pi$, which can be identified by an IV-aided Bellman equation. Before stating our result, we make one additional assumption.

\begin{assumption}\label{ass:10}
We have $(Z_t, U_t) \indp (\{S_j, U_j, A_j\}_{j < t})\given S_t$ for $t \geq 1$.
\end{assumption}
Assumption \ref{ass:10} ensures that $(Z_t, U_t)$ is ``memoryless", which does not dependent on past observations. This essentially ensures that the stochastic process $\{S_t, Z_t, A_t\}_{t\geq0}$ satisfies Markov property. The memoryless assumption on the unobserved confounders has been commonly used in the confounded MDP. See \cite{kallus2020confounding,shi2022off} for more details.

\begin{lemma}
\label{lemma:iv-vf-id}
Under Assumptions \ref{ass:MDP} and \ref{ass:iv-common}, for any $s \in \cS$ and $\pi \in \Pi$, we have
\$
V^\pi(s) = \EE\left[ \sum_{t = 0}^\infty \gamma^t R_t \left( \prod_{j = 0}^t \frac{Z_j^\top A_j \pi(A_j \given S_j)}{\Delta^*(S_j, A_j) \Theta^*(S_j, Z_j)  } \right ) \bigggiven S_0 = s \right]. 
\$
If additionally Assumption \ref{ass:10} is satisfied, it holds for any $t\geq 0$ that, 
\$
V^\pi(s) = \EE \left[ \frac{Z_t^\top A_t \pi(A_t \given S_t)}{\Delta^*(S_t, A_t) \Theta^*(S_t, Z_t) } \cdot \left(R_t + \gamma V^\pi(S_{t+1})\right ) \Biggiven S_t = s \right].
\$
Then the policy value $J(\pi)$ for $\pi \in \Pi$ can be identified via
\$
J(\pi) = (1-\gamma) \cdot \EE_{S_0 \sim \nu}\left[V^\pi(S_0)\right].
\$
\end{lemma}
\begin{proof}
See \S\ref{prf:lemma:iv-vf-id} 
of the Supplementary Material for a detailed proof. 
\end{proof}

We remark that the Bellman equation in the unconfounded MDP takes the following form,
\$
V^\pi_{\textsf{unconf}}(s) = \EE \left[ \frac{\pi(A_t \given S_t)}{\PP(A_t\given S_t)} \cdot \left(R_t + \gamma V_{\textsf{unconf}}^\pi(S_{t+1})\right ) \Biggiven S_t = s \right],
\$
where $V^\pi_{\textsf{unconf}}$ is the corresponding state-value function in the unconfounded MDP. In comparison, to deal with the unobserved confounders, our identification result in Lemma \ref{lemma:iv-vf-id} incorporates the IVs into the action density ratio. It is also interesting to see that if one can observe the trajectory $\{S_t, Z_t, A_t, R_t\}_{t \geq 0}$ to the infinity, then Assumptions \ref{ass:MDP} and \ref{ass:iv-common} are sufficient to identify $V^\pi(s)$ and $J(\pi)$ based on the first statement of Lemma \ref{lemma:iv-vf-id}. However, due to the limitation of only observing trajectories up to a finite horizon, we impose Assumption \ref{ass:10} so that Bellman equation is satisfied and used to break the curse of infinite-horizon. Based on Lemma \ref{lemma:iv-vf-id}, we introduce the following VF-based estimating equation, which will be used later in \S\ref{sec:iv-vf} to construct an estimator of the value function $V^\pi$. 

\begin{corollary}[VF-based Estimating Equation]
\label{cor:iv-vf-phi-0}
Under Assumptions \ref{ass:MDP}, \ref{ass:iv-common}, and \ref{ass:10}, it holds for any function $g\colon \cS\to \RR$ that
\$
\EE\left[ \frac{1}{T} \sum_{t = 0}^{T-1} g(S_t) \frac{Z_t^\top A_t \pi(A_t\given S_t)}{\Delta^*(S_t,A_t) \Theta^*(S_t,Z_t)} \cdot \left(R_t + \gamma V^\pi(S_{t+1})\right) \right] = \EE\left[\frac{1}{T}\sum_{t=0}^{T-1} g(S_t) V^\pi(S_t)\right]. 
\$
\end{corollary}
\begin{proof}
See \S\ref{prf:cor:iv-vf-phi-0} 
of the Supplementary Material for a detailed proof. 
\end{proof}

\subsection{Marginalized Importance Sampling-based Identification}
\label{sec:iv-mis-id}

In this subsection, we propose another way to identify $J(\pi)$ via the marginal importance sampling. We first introduce the following notations. 
For any $t\geq 0$, we denote by $p_t^\pi(\cdot)$ the marginal distribution of $S_t$ under the known initial observed state distribution $\nu$ following the policy $\pi$. 
In the meanwhile, with a slight abuse of notations, we denote by $p_t^b(\cdot)$ the marginal distribution of $S_t$ under the unknown offline data generating distribution $\zeta$ following the behavior policy $b$. In addition, we denote by for every $s \in \cS$, 
\#\label{eq:iv-ratio-def}
d^\pi(s) = (1-\gamma) \sum_{t = 0}^\infty \gamma^t p_t^\pi(s), \qquad d^b(s) = \frac{1}{T} \sum_{t = 0}^{T-1} p_t^b(s), \qquad w^\pi(s) = \frac{d^\pi(s)}{d^b(s)}
\# 
the discounted state visitation measure under the policy $\pi$, the average state visitation measure under the behavior policy $b$, and their density ratio, respectively. As $T$ is the same for each trajectory, they all share the same $d^b$. In general, if $T$ is different from subjects to subjects, then one can treat $T$ as a random variable and define ratio function as a mixture of different ratio functions.
Motivated by the idea of marginalized importance sampling for off-policy evaluation in the standard MDP \citep{liu2018breaking}, we establish the following novel identification result for the expected total reward $J(\pi)$ in the confounded MDP.

\begin{lemma}
\label{lemma:iv-mis-j}
Under Assumptions \ref{ass:MDP}--\ref{ass:10}, for any $\pi \in \Pi$, we have
\$
J(\pi) = \EE\left[ \frac{1}{T} \sum_{t = 0}^{T-1} \frac{Z_t^\top A_t \pi(A_t\given S_t)}{\Delta^*(S_t,A_t) \Theta^*(S_t,Z_t)} \cdot w^\pi(S_t) R_t \right], 
\$
where $w^\pi$ is defined in \eqref{eq:iv-ratio-def}.
\end{lemma}
\begin{proof}
See \S\ref{prf:lemma:iv-mis-j} 
of the Supplementary Material for a detailed proof. 
\end{proof}

We remark that the expected total reward in the unconfounded MDP takes the following form,
\$
J_{\textsf{unconf}}(\pi) = \EE \left[ \frac{1}{T} \sum_{t = 0}^{T-1} \frac{\pi(A_t \given S_t)}{\PP(A_t\given S_t)}\cdot w^\pi(S_t) R_t \right],
\$
where $J_{\textsf{unconf}}(\pi)$ is the corresponding expected total reward in the unconfounded MDP, and the expectation $\EE[\cdot]$ is taken with respect to the trajectory generated by the behavior policy. In comparison, to deal with the unobserved confounders, our identification result in Lemma \ref{lemma:iv-mis-j} incorporates the IVs into the action density ratio. 
Based on Lemma \ref{lemma:iv-mis-j}, we introduce the following MIS-based estimating equation, which will be used later in \S\ref{sec:iv-mis} to construct an estimator of the density ratio $w^\pi$.

\begin{lemma}[MIS-based Estimating Equation]
\label{lemma:iv-est-eq}
Under Assumptions \ref{ass:MDP}-\ref{ass:10}, for any $\pi \in \Pi$, it holds for any function $f\colon \cS \to \RR$ that 
\$
(1-\gamma) \EE_{S_0\sim \nu} \left [ f(S_0) \right] = \EE\left[ \frac{1}{T} \sum_{t = 0}^{T-1} \frac{Z_t^\top A_t \pi(A_t\given S_t)}{\Delta^*(S_t,A_t) \Theta^*(S_t,Z_t)} \cdot w^\pi(S_t) \left ( f(S_t) -  \gamma  f(S_{t+1})  \right ) \right]. 
\$
\end{lemma}
\begin{proof}
See \S\ref{prf:lemma:iv-est-eq} 
for a detailed proof. 
\end{proof}
Lemma \ref{lemma:iv-est-eq} shows that with the help of IVs and Assumption \ref{ass:10}, the estimating equation for the density ratio $w^\pi$ holds under the confounded MDP. This is different from the existing approaches such as \cite{liu2018breaking,zhang2020gendice} for estimating ratio functions in the standard MDP setting, which relies crucially on the no unmeasured confounding assumption. 


\section{Instrumental-Variable-Assisted RL with Pessimism}\label{sec:iv-pess}
In this section, we introduce three pessimistic RL methods to estimate $\pi^*$ in our confounded MDP. Generally, pessimistic RL first employs the offline data to construct a conservative estimate of the values for any policy, then select the policy with the highest conservative estimate of its value. Though the recently proposed pessimistic RL shows promising performance in practice \citep{kumarConservativeQlearningOffline2020,yuMopoModelbasedOffline2020,kidambiMorelModelbasedOffline2020,deng2021score}, its theoretical understanding is far from complete and is only limited to fully observable MDP \citep{levine2020offline,shi2022pessimistic}. In this section, we adapt the idea of pessimism in our case and present related theoretical results in Section \ref{sec:iv-theory}. 


Since both identification results in \S\ref{sec:iv-vf-id} and \S\ref{sec:iv-mis-id} require estimating the quantities $\Delta^*(s,a)$ and $\Theta^*(s,z)$, we first introduce the estimating procedure for such quantities. We assume that there exists an oracle that gives estimators of $\Delta^*(s,a)$ and $\Theta^*(s,z)$ via two loss functions $\hat L_0(\Delta)$ and $\hat L_1(\Theta)$ as follows, 
\$
\hat \Delta\in \argmin_{\Delta\in \cF_0} \hat L_0(\Delta),\qquad \hat \Theta \in \argmin_{\Theta \in \cF_1} \hat L_1(\Theta), 
\$
where $\cF_0$ and $\cF_1$ are two function classes. 
We remark that we can use the negative likelihood functions for $\hat L_0$ and $\hat L_1$ (see \S\ref{sec:iv-theory} for details). 
In the meanwhile, we construct two confidence sets for $\Delta$ and $\Theta$, respectively, as follows, 
\#\label{eq:mle-conf-set}
& \ci^0_{\alpha_0} = \left \{\Delta \in \cF_0 \colon \hat L_0(\Delta) - \hat L_0(\hat \Delta) \leq \alpha_0 \right\},\\
& \ci^1_{\alpha_1} = \left\{\Theta \in \cF_1  \colon \hat L_1(\Theta) - \hat L_1(\hat \Theta) \leq \alpha_1 \right\},
\#
where $(\alpha_0,\alpha_1)$ are some constants that will be specified later. These two confidence sets are used to construct conservative estimators for $J(\pi)$ via either VF-, MIS-, or doubly robust (DR)-based estimation.

\subsection{VF-based Pessimistic Method}\label{sec:iv-vf}
We introduce VF-based pessimistic RL in this section. 
We first define the following quantity, 
\$ 
\hat \Phi^\pi_\vf(v,g; \Delta, \Theta) = \hat \EE\left[ \frac{1}{T} \sum_{t = 0}^{T-1} g(S_t) \left ( \frac{Z_t^\top A_t \pi(A_t\given S_t)}{\Delta(S_t,A_t) \Theta(S_t,Z_t)} \left(R_t + \gamma v(S_{t+1})\right) - v(S_t) \right ) \right], 
\$ 
where $\hat \EE[\cdot]$ is the empirical measure defined by the offline data $\cD$. In the meanwhile, we define its population counterpart as 
\$
\Phi^\pi_\vf(v,g;\Delta, \Theta) = \EE\left[ \frac{1}{T} \sum_{t = 0}^{T-1} g(S_t) \left ( \frac{Z_t^\top A_t \pi(A_t\given S_t)}{\Delta(S_t,A_t) \Theta(S_t,Z_t)} \left(R_t + \gamma v(S_{t+1})\right) - v(S_t) \right ) \right]
\$ 
for any $(v, g, \Delta, \Theta)$, where the expectation $\EE[\cdot]$ is taken with respect to the trajectory generated by the behavior policy. Then by the VF-based estimating equation specified in Corollary \ref{cor:iv-vf-phi-0}, it is easy to see that $\Phi^\pi_\vf(V^\pi,g;\Delta^*, \Theta^*) = 0$ for any function $g\colon \cS \to \RR$.

With the aforementioned notions, for any $(\Delta, \Theta)$, we construct an estimator of $V^\pi$ via solving the following min-max optimization problem, 
\#\label{eq:v-pi-def}
\hat v_{\Delta, \Theta}^\pi \in \argmin_{v\in \cV} \max_{g\in \cW} \hat \Phi_\vf^\pi (v, g; \Delta, \Theta),
\#
where $\cV$ and $\cW$ are two sets to be specified later. To obtain an estimation $\hat \pi_\vf$ for an optimal in-class policy $\pi^\ast$ that maximizes the expected total reward $J(\pi)$ defined in \eqref{eq:iv-val-func}, we formulate the following optimization problem,
\#\label{eq:iv-vf-hat-pi}
& \hat \pi_\vf = \argmax_{\pi\in \Pi} \min_{ (\Delta, \Theta) \in \ci^0_{\alpha_0}\times \ci^1_{\alpha_1} } \min_{v \in \ci_{\alpha_\vf}^\vf(\Delta, \Theta, \pi)} (1-\gamma)\EE_{S \sim \nu}[v(S)], \\
& \text{with } \ci_{\alpha_\vf}^\vf(\Delta, \Theta, \pi) = \left \{v\in \cV \colon \max_{g\in \cW} \hat \Phi_\vf^\pi(v, g; \Delta, \Theta) - \max_{g\in \cW} \hat \Phi_\vf^\pi(\hat v_{\Delta, \Theta}^\pi, g; \Delta, \Theta) \leq \alpha_\vf \right \}, 
\#
where $(\alpha_0, \alpha_1, \alpha_\vf)$ are constants to be specified, and $\ci^0_{\alpha_0}$ and $\ci^1_{\alpha_1}$ are confidence sets defined in \eqref{eq:mle-conf-set}. Intuitively, the policy $\hat \pi_\vf$ defined in \eqref{eq:iv-vf-hat-pi} aims to maximize the most pessimistic estimator of the expected total reward. As we will see in Theorem \ref{thm:iv-vf}, such a pessimistic method provably converges to an optimal policy with data coverage assumption only for the optimal policy and some other mild conditions.

\subsection{MIS-based Pessimistic Method}\label{sec:iv-mis}

We introduce MIS-based pessimistic RL in this section. We first define the following quantity,
\$
\hat \Phi^\pi_\mis(w, f; \Delta, \Theta) =  & \EE_{S_0 \sim \nu}\left[ (1-\gamma) f(S_0) \right] \\
& - \hat \EE \left[ \frac{1}{T} \sum_{t=0}^{T-1} \frac{Z_t^\top A_t \pi(A_t\given S_t)}{\Delta(S_t,A_t) \Theta(S_t,Z_t)}w(S_t) \left (f(S_t) - \gamma f(S_{t+1}) \right ) \right].
\$
We define its population counterpart as
\$
\Phi^\pi_\mis(w,f;\Delta, \Theta) = & \EE_{S_0 \sim \nu}\left[ (1-\gamma) f(S_0) \right] \\
& - \EE \left[ \frac{1}{T} \sum_{t=0}^{T-1} \frac{Z_t^\top A_t \pi(A_t\given S_t)}{\Delta(S_t,A_t) \Theta(S_t,Z_t)}w(S_t) \left (f(S_t) - \gamma f(S_{t+1}) \right ) \right]
\$ 
for any $(w, f, \Delta, \Theta)$. Then by the MIS-based estimating equation specified in Lemma \ref{lemma:iv-est-eq}, it can be seen that $\Phi^\pi_\mis(w^\pi,f;\Delta^*, \Theta^*) = 0$ for any function $f\colon \cS \to \RR$. With the aforementioned notions, for any $(\Delta, \Theta)$, we construct an estimator of $w^\pi$ via solving the following minimax optimization problem, 
\#\label{eq:w-pi-def}
& \hat w^\pi_{\Delta, \Theta} \in \argmin_{w\in \cW} \max_{f\in \cV} \hat \Phi^\pi_\mis(w, f; \Delta, \Theta). 
\#
With a slight abuse of notations, here $\cW$ and $\cV$ are again two sets to be specified later. 
We aim to obtain an optimal policy that maximizes the expected total reward $J(\pi)$ by utilizing the estimators constructed in \eqref{eq:w-pi-def}. 
For this, we further define the following estimator of $J(\pi)$ via Lemma \ref{lemma:iv-mis-j}, 
\$ 
\hat L_\mis (w, \pi; \Delta, \Theta) = \hat \EE \left[ \frac{1}{T} \sum_{t = 0}^{T-1} \frac{Z_t^\top A_t \pi(A_t\given S_t)}{\Delta(S_t,A_t) \Theta(S_t,Z_t)} w(S_t) R_t \right].
\$ 
Then we aim to solve the following optimization problem, 
\#\label{eq:iv-mis-hat-pi}
& \hat \pi_\mis \in \argmax_{\pi\in \Pi} \min_{ (\Delta, \Theta) \in \ci^0_{\alpha_0}\times \ci^1_{\alpha_1} } \min_{w \in \ci^\mis_{\alpha_\mis}(\Delta, \Theta, \pi) } \hat L_\mis(w, \pi; \Delta, \Theta), \\
& \text{with } \ci^\mis_{\alpha_\mis}(\Delta, \Theta, \pi) = \Biggl\{  w\in \cW  \colon \max_{f\in \cV} \hat \Phi_\mis^\pi(w, f; \Delta, \Theta)  - \max_{f\in \cV} \hat \Phi_\mis^\pi(\hat w^\pi_{\Delta, \Theta}, f; \Delta, \Theta) < \alpha_\mis \Biggr\}, 
\#
where $(\alpha_0, \alpha_1, \alpha_\mis)$ are constants to be specified, and $\ci^0_{\alpha_0}$ and $\ci^1_{\alpha_1}$ are confidence sets defined in \eqref{eq:mle-conf-set}. Similarly as in \eqref{eq:iv-vf-hat-pi}, the policy $\hat \pi_\mis$ defined in \eqref{eq:iv-mis-hat-pi} aims to maximize the most pessimistic estimator of the expected total reward. As we will see in Theorem \ref{thm:iv-mis}, such a pessimistic method provably converges to an optimal policy with realizability assumption only for the optimal policy.

\subsection{DR-based Pessimistic Method}\label{sec:iv-dr}

As a combination of VF-based and MIS-based policy optimization methods, we introduce a doubly robust (DR)-based pessimistic RL algorithm in this section. 
We define the following DR estimator with its population counterpart, 
\$ 
& \hat L_\dr(w, v, \pi; \Delta, \Theta) = \hat \EE \left[\frac{1}{T} \sum_{t = 0}^{T-1} \frac{Z_t^\top A_t \pi(A_t\given S_t)}{\Delta(S_t,A_t) \Theta(S_t,Z_t)} w(S_t) \left( R_t + \gamma v(S_{t+1}) - v(S_t) \right) \right] \\
& \qquad \qquad \qquad \qquad \quad + (1-\gamma) \EE_{S_0\sim \nu} \left[v(S_0)\right], \\
& L_\dr(w, v, \pi; \Delta, \Theta) = \EE \left[\frac{1}{T} \sum_{t = 0}^{T-1} \frac{Z_t^\top A_t \pi(A_t\given S_t)}{\Delta(S_t,A_t) \Theta(S_t,Z_t)} w(S_t) \left( R_t + \gamma v(S_{t+1}) - v(S_t) \right) \right] \\
& \qquad \qquad \qquad \qquad \quad + (1-\gamma) \EE_{S_0\sim \nu} \left[v(S_0)\right]. 
\$ 
Note that $L_\dr(w^\pi, v, \pi; \Delta^*, \Theta^*) = L_\dr(w, V^\pi, \pi; \Delta^*, \Theta^*) = J(\pi)$ for any $(\pi, w, v)\in \Pi \times \cW\times \cV$. Thus, the quantity $\hat L_\dr$ serves as a valid DR estimator of $J(\pi)$. 
In the follows, based on such a DR estimator of the expected total reward, we formulate the following optimization problem for estimating the optimal in-class policy,
\#
& \hat \pi_\dr \in \argmax_{\pi\in \Pi} \min_{(\Delta, \Theta) \in \ci^0_{\alpha_0}\times \ci^1_{\alpha_1}} \min_{(w,v) \in \ci_{\alpha_\mis, \alpha_\vf}(\Delta, \Theta, \pi)} \hat L_\dr(w,v,\pi; \Delta, \Theta),  \\
& \text{with } \ci_{\alpha_\mis, \alpha_\vf}(\Delta, \Theta, \pi) = \ci^\mis_{\alpha_\mis}(\Delta, \Theta, \pi)\times \ci_{\alpha_\vf}^\vf(\Delta, \Theta, \pi), \label{eq:iv-dr}
\#
where $(\alpha_0, \alpha_1, \alpha_\vf, \alpha_\mis)$ are constants to be specified, and $\ci_{\alpha_\vf}^\vf(\Delta, \Theta, \pi)$ and $\ci^\mis_{\alpha_\mis}(\Delta, \Theta, \pi)$ are defined in \eqref{eq:iv-vf-hat-pi} and \eqref{eq:iv-mis-hat-pi}, respectively. As we will see in Theorem \ref{thm:iv-dr}, such a DR-based pessimistic method provably converges to an optimal policy with realizability assumption only for the optimal policy.

\section{Theoretical Results}
\label{sec:iv-theory}

In this section, we investigate theoretical properties of the aforementioned three methods. We aim to derive the finite-sample upper bounds for the sub-optimality of our estimated policies, i.e., $\subopt{(\widehat \pi)}$, where $\hat \pi$ is either $\hat \pi_\vf, \hat \pi_\mis$, or $\hat \pi_\dr$. To begin with, we first introduce the following definition of covering number, and then impose metric entropy conditions on related classes of functions used in the proposed algorithms. 

\begin{definition}[Covering Number]
Let $(\cC, \|\cdot \|_\infty)$ be a normed space, and $\cH \subseteq \cC$. The set $\{x_1, x_2, \ldots, x_{n}\}$ is a $\vareps$-covering over $\cH$ if $\cH\subseteq \cup_{i = 1}^n B(x_i, \vareps)$, where $B(x_i, \vareps)$ is the sup-norm-ball centered at $x_i$ with radius $\vareps$. Then the covering number of $\cH$ is defined as $N(\vareps, \cH, \|\cdot\|_\infty) = \min\{n\colon \exists \text{ $\vareps$-covering over $\cH$ of size $n$}\}$. 
\end{definition}

\begin{assumption}\label{ass:spaces}
The following statements hold. 
\begin{enumerate}[label=(\alph*)]
\item \label{ass:bounded-covering} For any set $\cH \in \{\cF_0, \cF_1, \cV, \cW, \Pi\}$, there exists a constant $\pdim_{\cH}$ such that 
\$
N(\vareps, \cH, \|\cdot\|_\infty) \leq c\cdot (1/\vareps)^{\pdim_{\cH}}, 
\$
where $c > 0$ is a constant. Further, we denote by $\pdim_{\cH_1, \cH_2, \ldots, \cH_k} = \sum_{j\in [k]}\pdim_{\cH_j}$ for any class of functions $\{\cH_1, \cH_2, \ldots, \cH_k\}$. 
\item \label{ass:upper-bound-delta} There exist positive constants $C_{\Delta^*}$ and $C_{\Theta^*}$ such that $|\Delta^*(s, a)| \geq C_{\Delta^*}^{-1}$ and $\Theta^*(s, z) \geq C_{\Theta^*}^{-1}$ for any $(s,z,a)\in \cS\times \cZ\times \cA$, where $\Theta^*(s,z)$ and $\Delta^*(s,a)$ are defined in Assumption \ref{ass:iv-common}.
\item \label{ass:bounded-mle} We have $|\Delta(s,a)| \geq C_{\Delta^*}^{-1}$ and $\Theta(s,z) \geq C_{\Theta^*}^{-1}$ for any $(\Delta, \Theta, s, a, z)\in \cF_0\times \cF_1 \times \cS \times \cA \times \cZ$. 
\item \label{ass:lip}
We have $\sup_{s\in \cS}|V^{\pi_1}(s) - V^{\pi_2}(s)| \leq L_\Pi \cdot \sup_{(s,a)\in \cS\times \cA} |\pi_1(a\given s) - \pi_2(a\given s)|$ for any $\pi_1,\pi_2 \in \Pi$, where $L_\Pi$ is a positive constant. 
\item \label{ass:function-bound}
We have $\|v\|_\infty \leq 1/(1-\gamma)$ and $\|w\|_\infty \leq C_*$ for any $(v,w)\in \cV\times \cW$, where $C_* > 0$ is a constant. 
\end{enumerate}
\end{assumption}

Assumption \ref{ass:spaces}\ref{ass:bounded-covering} states that the function spaces have finite-log covering numbers, which has been widely used in the existing literature \citep[e.g.,][]{antos2008learning}.
Assumption \ref{ass:spaces}\ref{ass:upper-bound-delta} states that the conditional probability $\Theta^*$ and the compliance $\Delta^*$ are uniformly lower bounded. Basically we require to have a data coverage on all the IVs and a non-negligible gap in terms of compliance. But we do not require a coverage assumption on all the actions. In practice, it seems that the coverage of all IVs is more plausible than that of all actions due to the noncompliance.
With Assumption \ref{ass:spaces}\ref{ass:upper-bound-delta}, we only need to consider a lower bounded function class to recover $\Theta^*$ and $\Delta^*$, which is imposed in Assumption \ref{ass:spaces}\ref{ass:bounded-mle}. In the meanwhile, the Lipschitz condition imposed in Assumption \ref{ass:spaces}\ref{ass:lip} aims to control the complexity of the value function class induced by $\Pi$, i.e., the class $\{V^\pi(\cdot)\colon \pi\in \Pi\}$. Such an assumption is commonly imposed in related literature \citep{zhou2017residual,liao2020batch}. 
Finally, Assumption \ref{ass:spaces}\ref{ass:function-bound} states that the sets $\cV$ and $\cW$ are uniformly bounded for deriving the exponential inequalities.


\begin{assumption}\label{ass:ergodic}
The sequence $\{S_t, Z_t, U_t, A_t\}_{t\geq 0}$ admits a unique stationary distribution $G_\stat$ over $\cS\times \cZ \times \cU \times \cA$ and is geometrically ergodic, i.e., there exists a function $\varphi\colon \cS\times \cZ \times \cU \times \cA \to \RR^+$ and a constant $\kappa > 0$ such that 
\$
\left\| G_\stat(\cdot) - G_t(\cdot \given s_0, z_0, u_0, a_0 ) \right\|_\tv  \leq \varphi(s_0, z_0, u_0, a_0) \cdot \exp\left(-2\kappa t\right), 
\$
where $G_t(\cdot \given s_0, z_0, u_0, a_0)$ is the marginal distribution of $(S_t, Z_t, U_t, A_t)$ given $(S_0, Z_0, U_0, A_0) = (s_0, z_0, u_0, a_0)$ under the behavior policy $b$. Further, we have $\int \varphi(s,z,u,a) \ud \nu(s, z, u, a) \leq c$ and $\int \varphi(s,z,u,a) \ud G_\stat(s, z, u, a) \leq c$ for some positive absolute constant $c$. 
\end{assumption}

Assumption \ref{ass:ergodic} states that the the Markov chain $\{S_t, Z_t, U_t, A_t\}_{t\geq 0}$ mixes geometrically. Such an assumption is widely adopted in the related literature \citep{van1998learning,wang2021projected} to deal with dependent data.


To establish the upper bounds for the sub-optimality of the resulting policies, we need to first show that in our proposed algorithms, there exists at least one feasible solution that satisfies the constraints with properly chosen constants. In the following, we focus on $\ci^0_{\alpha_0}$ and $\ci^1_{\alpha_1}$ for $\Delta^*$ and  $\Theta^*$ respectively. Since $\ci^0_{\alpha_0}$ and $\ci^1_{\alpha_1}$  can be constructed by many methods, to keep our theoretical results general, we assume that there exists a proper choice of $(\alpha_0,\alpha_1)$ that ensures $\Delta^* \in \ci^0_{\alpha_0}$ and $\Theta^* \in \ci^1_{\alpha_1}$ and then give a valid example that justifies this assumption. 

\begin{assumption}\label{ass:iv-sl-res}
There exists $(\alpha_0, \alpha_1)$ such that with probability at least $1 - \delta$, we have 
\$
\Delta^* \in \ci^0_{\alpha_0}, \qquad \Theta^* \in \ci^1_{\alpha_1}.
\$
Further, with probability at least $1 - \delta$, for any $(\Delta, \Theta) \in \ci^0_{\alpha_0} \times \ci^1_{\alpha_1}$, we have
\$
& \EE\left[ \frac{1}{T} \sum_{t = 0}^{T-1} \left \| \Delta^*(S_t,\cdot) - \Delta(S_t,\cdot) \right \|_1^2  \right ] \leq \xi_0^2 \frac{C_{\Delta^*}}{NT \kappa} \cdot \pdim_{\cF_0} \log\frac{2}{\delta},  \\
& \EE\left[ \frac{1}{T} \sum_{t = 0}^{T-1} \left \| \Theta^*(S_t,\cdot) - \Theta(S_t,\cdot) \right \|_1^2  \right ] \leq \xi_1^2 \frac{C_{\Theta^*}}{NT \kappa} \cdot \pdim_{\cF_1} \log\frac{2}{\delta}. 
\$
\end{assumption}

We now illustrate that Assumption \ref{ass:iv-sl-res} can be realized via maximum likelihood estimation (MLE) by replacing $\xi_0$ and $\xi_1$ with proper quantities. 
Note that the estimation of $\Delta^*$ can be decomposed into the estimation of $\PP(A = a\given S=s,Z=z)$ for all $z\in \cZ$, which can also be obtained via MLE. This implies that estimating $\Delta^*$ is similar to estimating $\Theta^*$. Therefore, we only show how to estimate $\Theta^*$ so that Assumption \ref{ass:iv-sl-res} holds for the simplicity of presentation. 
By maximum likelihood, we construct the loss function $\hat L_1$ and the estimator $\hat \Theta$ as follows, 
\$
\hat L_1(\Theta) = - \hat \EE\left[ \frac{1}{T} \sum_{t = 0}^{T-1} \log \Theta(S_t, Z_t) \right] = - \frac{1}{NT} \sum_{i\in[N]} \sum_{t=0}^{T-1} \log \Theta(S_t^i, Z_t^i), ~~~ \hat \Theta \in \argmin_{\Theta\in \cF_1} \hat L_1(\Theta), 
\$
where for the ease of notations, we denote by $\hat \EE[\cdot]$ the empirical measure generated by the offline data $\cD$ hereafter. 
In addition, we assume that $\cF_1$ is a parametric class such that $\cF_1 = \{\Theta_\theta\colon \theta\in \RR^d, \|\theta\|_2 \leq \theta_{\max}\}$. We introduce the following results. 

\begin{theorem}\label{thm:iv-param-theta}
Suppose $\cF_1 = \{\Theta_\theta\colon \theta\in \RR^d \text{ and } \|\theta\|_2 \leq \theta_{\max}\}$, and 
\$
\alpha_1 = c \cdot \frac{C_{\Theta^*}}{NT \kappa} \cdot d \log \frac{\theta_{\max}}{\delta} \log(NT),
\$
where $c / (N^2 T^2) \cdot \log (NT) \leq \delta \leq 1$. 
Then under Assumptions \ref{ass:iv-common}, \ref{ass:spaces}\ref{ass:upper-bound-delta}, \ref{ass:spaces}\ref{ass:bounded-mle}, and \ref{ass:ergodic}, it holds that with probability at least $1-\delta$ that $\Theta^* \in \ci^1_{\alpha_1}$. Further, with probability at least $1 - \delta$, it holds for any $\Theta \in \ci^1_{\alpha_1}$ that 
\$
\sqrt{\EE\left[ \|\Theta(S,\cdot) - \Theta^*(S,\cdot)\|_1^2 \right]} \leq c\sqrt{ \frac{C_{\Theta^*}}{NT \kappa} \cdot d\log \frac{\theta_{\max}}{\delta}}. 
\$
\end{theorem}
\begin{proof}
See \S\ref{prf:thm:iv-param-theta} 
of the Supplementary Material for a detailed proof. 
\end{proof}

Supported by Theorem \ref{thm:iv-param-theta}, we assume Assumption \ref{ass:iv-sl-res} holds throughout this section.

\subsection{Theoretical Results for VF-based Pessimistic Method}\label{sec:vf-theory}

We first impose the following assumption, which assumes that $V^\pi$ is realizable in $\cV$ for any policy $\pi$, and $w^{\pi^*}$ is realizable in $\cW$ only for the optimal policy $\pi^*$. 

\begin{assumption}
\label{ass:iv-vf-realizable}
We have $V^\pi \in \cV$ for any $\pi\in \Pi$ and $w^{\pi^*} \in \cW$. Further, we have $-w\in \cW$ for any $w \in \cW$. 
\end{assumption}


In the following lemma, we show that with a proper choice of $\alpha_\vf$, we have $V^\pi \in \ci^\vf_{\alpha_\vf}(\Delta^*, \Theta^*, \pi)$ with a high probability. 

\begin{lemma}
\label{lemma:iv-v-pi-in-conf}
Suppose 
\$
\alpha_\vf = c\cdot \frac{C_{\Delta^*} C_{\Theta^*} C_*}{1-\gamma} \sqrt{\frac{\pdim_{\cW,\cV,\Pi}}{NT\kappa} \cdot \log\frac{1}{\delta} \log(NT) }
\$
and $c / (NT)^2 \leq \delta \leq 1$. 
Then under Assumptions \ref{ass:spaces} and \ref{ass:iv-vf-realizable}, with probability at least $1 - \delta$, it holds for any $\pi\in \Pi$ that $V^\pi \in \ci^\vf_{\alpha_\vf}(\Delta^*, \Theta^*, \pi)$. 
\end{lemma}
\begin{proof}
See \S\ref{prf:lemma:iv-v-pi-in-conf} 
of the Supplementary Material for a detailed proof. 
\end{proof}

In the following lemma, we show that for any $v\in \cup_{(\Delta, \Theta) \in \ci^0_{\alpha_0} \times \ci^1_{\alpha_1}} \ci^\vf_{\alpha_\vf}(\Delta, \Theta, \pi)$, we can upper bound the risk of $\max_{g\in \cW} \Phi^\pi_\vf(v,g;\Delta^*, \Theta^*)$, which in turn bounds the suboptimality of the estimated policy.

\begin{lemma}
\label{lemma:iv-v-in-conf-good}
Let $(\alpha_0, \alpha_1, \alpha_\vf)$ be those defined in Assumption \ref{ass:iv-sl-res} and Lemma \ref{lemma:iv-v-pi-in-conf} and $c / (NT)^2 \leq \delta \leq 1$.
Then under Assumptions \ref{ass:iv-common}--\ref{ass:iv-vf-realizable}, with probability at least $1 - \delta$, it holds for any policy $\pi\in \Pi$ and $v\in \cup_{(\Delta, \Theta) \in \ci^0_{\alpha_0} \times \ci^1_{\alpha_1}} \ci^\vf_{\alpha_\vf}(\Delta, \Theta, \pi)$ that
\$
\max_{g\in \cW} \Phi^\pi_\vf(v,g;\Delta^*, \Theta^*) \leq c\cdot \frac{C_{\Delta^*}^2 C_{\Theta^*}^2 C_*}{1-\gamma} (\xi_0 + \xi_1)  L_\Pi \sqrt{\frac{1}{NT\kappa}\cdot \pdim_{\cF_0,\cF_1,\cW,\cV,\Pi} \cdot \log\frac{1}{\delta} \log(NT)}. 
\$
\end{lemma}
\begin{proof}
See \S\ref{prf:lemma:iv-v-in-conf-good} 
of the Supplementary Material for a detailed proof. 
\end{proof}

Equipped with the above results, we introduce the following theorem, which characterizes the suboptimality of the learned policy $\hat \pi_\vf$ constructed in \eqref{eq:iv-vf-hat-pi}. 

\begin{theorem}
\label{thm:iv-vf}
Suppose $c / (NT)^2 \leq \delta \leq 1$. 
Under Assumptions \ref{ass:iv-common}--\ref{ass:iv-vf-realizable}, it holds with probability at least $1 - \delta$ that
\$
\subopt(\hat \pi_\vf) \leq c\cdot \frac{ C_{\Delta^*}^2 C_{\Theta^*}^2 C_*}{1-\gamma} (\xi_0 + \xi_1) L_\Pi \sqrt{\frac{1}{NT \kappa}\cdot \pdim_{\cF_0,\cF_1,\cW,\cV,\Pi} \cdot \log\frac{1}{\delta} \log(NT)}. 
\$
\end{theorem}
\begin{proof}[Proof Sketch]
In the proof sketch, we assume that we have full knowledge on $\Delta^*$ and $\Theta^*$.
By the definition of $J(\pi)$ in \eqref{eq:iv-val-func}, we have
\$
J(\pi^*) - J(\hat \pi_\vf) &  = (1-\gamma) \EE_{S_0\sim \nu}\left [V^{\pi^*}(S_0) - V^{\hat \pi_\vf}(S_0)\right] \\
& \leq (1-\gamma) \EE_{S_0\sim \nu} \left [V^{\pi^*}(S_0) \right] - \min_{v \in \ci_{\alpha_\vf}^\vf(\Delta^*, \Theta^*, \hat \pi_\vf)} (1-\gamma) \EE_{S_0\sim \nu} \left [v(S_0) \right] \\
& \leq (1-\gamma) \EE_{S_0\sim \nu} \left [V^{\pi^*}(S_0) \right] -  \min_{v \in \ci_{\alpha_\vf}^\vf(\Delta^*, \Theta^*, \pi^*)} (1-\gamma) \EE_{S_0\sim \nu} \left [v(S_0) \right]  \\
& \leq (1-\gamma) \cdot \max_{v \in \ci_{\alpha_\vf}^\vf(\Delta^*, \Theta^*, \pi^*)} \left |  \EE_{S_0\sim \nu} \left [V^{\pi^*}(S_0) - v(S_0) \right] \right |, 
\$
where in the first inequality, we use Lemma \ref{lemma:iv-v-pi-in-conf} that $V^{\hat \pi_\vf} \in \ci_{\alpha_\vf}^\vf(\Delta^*, \Theta^*, \hat \pi_\vf)$ with a high probability; while in the second inequality, we use the optimality of $\hat \pi_\vf$. 
In the meanwhile, by Lemmas \ref{lemma:iv-mis-j} and \ref{lemma:iv-est-eq}, we have the following decomposition, 
\#\label{eq:iv-vf-pp2-scketch}
& (1-\gamma) \EE_{S_0\sim \nu} \left [ V^{\pi^*}(S_0) \right] = J(\pi^*) = \EE\left[ \frac{1}{T} \sum_{t = 0}^{T-1}  w^{\pi^*}(S_t) \frac{Z_t^\top A_t \pi^*(A_t\given S_t)}{\Delta^*(S_t,A_t) \Theta^*(S_t,Z_t)} R_t \right], \\
& (1-\gamma) \EE_{S_0\sim \nu} \left [ v(S_0) \right] = \EE\left[ \frac{1}{T} \sum_{t = 0}^{T-1} w^{\pi^*}(S_t) \frac{Z_t^\top A_t \pi^*(A_t\given S_t)}{\Delta^*(S_t,A_t) \Theta^*(S_t,Z_t)} \left ( v(S_t) - \gamma  v(S_{t+1})  \right ) \right]. 
\#
Now, by plugging \eqref{eq:iv-vf-pp2-scketch}, we have
\#\label{eq:OPE error}
& J(\pi^*) - J(\hat \pi_\vf) \leq  \max_{v \in \ci_{\alpha_\vf}^\vf(\Delta^*, \Theta^*, \pi^*)} \left | \Phi_\vf^{\pi^*}(v, w^{\pi^*}; \Delta^*, \Theta^*) \right |.
\#
We then can upper bound the above suboptimality by Lemma \ref{lemma:iv-v-in-conf-good}, which concludes the proof of the theorem. 
See \S\ref{prf:thm:iv-vf} 
of the Supplementary Material for a detailed proof. 
\end{proof}



In Theorem \ref{thm:iv-vf}, we impose data coverage and realizability assumptions as in Assumption \ref{ass:iv-vf-realizable}, which only requires that the offline data covers the trajectory generated by the optimal policy $\pi^*$ and $V^\pi$ is realizable in $\cV$ for any $\pi$. Our upper bound on the suboptimality of the estimated policy indicates that the regret of finding an optimal policy converges to $0$ as long as the number of trajectories or that of decision points on each trajectory goes to infinite.


\subsection{Theoretical Results for MIS-based Pessimistic Method}\label{sec:mis-theory}

We first impose the following assumption, which assumes that $w^\pi$ is realizable in $\cW$ for any policy $\pi$, and $V^{\pi^*}$ is realizable in $\cV$ only for the optimal policy $\pi^*$. 

\begin{assumption}
\label{ass:iv-mis-realizable}
We have $w^\pi \in \cW$ for any $\pi\in \Pi$ and $V^{\pi^*} \in \cV$. Further, we have $-v\in \cV$ for any $v \in \cV$. 
\end{assumption}


In the following lemma, we show that with a proper choice of $\alpha_\mis$, we have $w^\pi \in \ci^\mis_{\alpha_\mis}(\Delta^*, \Theta^*, \pi)$ with high probability.

\begin{lemma}
\label{lemma:iv-w-pi-in-conf}
Suppose
\$
\alpha_\mis = c\cdot \frac{C_{\Delta^*} C_{\Theta^*} C_*}{1-\gamma} \sqrt{\frac{1}{NT \kappa} \pdim_{\cV,\cW,\Pi}\log\frac{1}{\delta} \log(NT)}
\$
and $c / (NT)^2 \leq \delta \leq 1$. 
Then under Assumptions \ref{ass:spaces} and \ref{ass:iv-mis-realizable}, with probability at least $1 - \delta$, it holds for any $\pi \in \Pi$ that $w^\pi \in \ci^\mis_{\alpha_\mis}(\Delta^*, \Theta^*, \pi)$. 
\end{lemma}
\begin{proof}
See \S\ref{prf:lemma:iv-w-pi-in-conf} 
of the Supplementary Material for a detailed proof.
\end{proof}

In the following lemma, we show that for any $w \in \cup_{(\Delta, \Theta) \in \ci^0_{\alpha_0} \times \ci^1_{\alpha_1}} \ci^\mis_{\alpha_\mis}(\Delta, \Theta, \pi)$, we can upper bound the risk $\max_{f\in \cV} \Phi^\pi_\mis(w,f;\Delta^*, \Theta^*)$.

\begin{lemma}
\label{lemma:iv-w-in-conf-good}
Let $(\alpha_0, \alpha_1, \alpha_\mis)$ be those defined in Assumption \ref{ass:iv-sl-res} and Lemma \ref{lemma:iv-w-pi-in-conf}, and $c / (NT)^2 \leq \delta \leq 1$. 
Then under Assumptions \ref{ass:iv-common}--\ref{ass:iv-sl-res}, and \ref{ass:iv-mis-realizable}, with probability at least $1 - \delta$, it holds for any $\pi\in \Pi$ and $w \in \cup_{(\Delta, \Theta) \in \ci^0_{\alpha_0} \times \ci^1_{\alpha_1}} \ci^\mis_{\alpha_\mis}(\Delta, \Theta, \pi)$ that
\$
\max_{f\in \cV} \Phi^\pi_\mis(w,f;\Delta^*, \Theta^*) \leq c\cdot \frac{ C_{\Delta^*}^2 C_{\Theta^*}^2 C_*}{1-\gamma} (\xi_0 + \xi_1) \sqrt{\frac{1}{NT \kappa}\cdot \pdim_{\cF_0,\cF_1,\cW,\cV,\Pi} \cdot \log\frac{1}{\delta}\log(NT) }. 
\$
\end{lemma}
\begin{proof}
See \S\ref{prf:lemma:iv-w-in-conf-good} 
of the Supplementary Material for a detailed proof. 
\end{proof}

Equipped with the above results, we introduce the following theorem, which characterizes the suboptimality of the learned policy $\hat \pi_\mis$ constructed in \eqref{eq:iv-mis-hat-pi}. 

\begin{theorem}
\label{thm:iv-mis}
Suppose $c / (NT)^2 \leq \delta \leq 1$. Under Assumptions \ref{ass:iv-common}--\ref{ass:iv-sl-res}, and \ref{ass:iv-mis-realizable}, it holds with probability at least $1 - \delta$ that
\$
\subopt(\hat \pi_\mis) \leq c\cdot \frac{ C_{\Delta^*}^2 C_{\Theta^*}^2 C_*}{1-\gamma} (\xi_0 + \xi_1) \sqrt{\frac{1}{NT \kappa}\cdot \pdim_{\cF_0,\cF_1,\cW,\cV,\Pi} \cdot \log\frac{1}{\delta}\log(NT)}.
\$
\end{theorem}
\begin{proof}
See \S\ref{prf:thm:iv-mis} 
of the Supplementary Material for a detailed proof. 
\end{proof}

In Theorem \ref{thm:iv-mis}, we impose data coverage and realizability assumptions as in Assumption \ref{ass:iv-mis-realizable}, which only requires that $V^{\pi^*}$ is realizable in $\cV$ and the offline data covers the trajectory generated by the policy $\pi$ for any $\pi\in \Pi$. 

\subsection{Theoretical Results for DR-based Pessimistic Method}\label{sec:dr-theory}
In this section, we study the theoretical properties of our DR-based pessimistic method for confounded MDP.

\begin{theorem}\label{thm:iv-dr}
Let $(\alpha_0, \alpha_1, \alpha_\mis, \alpha_\vf)$ be those defined in Assumption \ref{ass:iv-sl-res}, Lemmas \ref{lemma:iv-v-pi-in-conf}, \ref{lemma:iv-w-pi-in-conf}, and one of Assumptions \ref{ass:iv-vf-realizable} and \ref{ass:iv-mis-realizable} hold. Then under Assumptions \ref{ass:iv-common}--\ref{ass:iv-sl-res}, it holds with probability at least $1 - \delta$ for any $c / (NT)^2 \leq \delta \leq 1$ that
\$
\subopt(\hat \pi_\dr) \leq c\cdot \frac{C_{\Delta^*}^2 C_{\Theta^*}^2 C_*}{1-\gamma} (\xi_0 + \xi_1) \sqrt{\frac{1}{NT \kappa} \pdim_{\cF_0,\cF_1,\cW,\cV,\Pi} \log\frac{1}{\delta}\log(NT) }, 
\$
where $\hat \pi_\dr$ is defined in \eqref{eq:iv-dr}. 
\end{theorem}
\begin{proof}
See \S\ref{prf:thm:iv-dr} 
of the Supplementary Material for a detailed proof. 
\end{proof}

Theorem \ref{thm:iv-dr} shows that $\hat \pi_\dr$ is a doubly robust estimator of the optimal policy in the sense that either Assumption \ref{ass:iv-vf-realizable} or Assumption \ref{ass:iv-mis-realizable} ensures the convergence of $\hat \pi_\dr$.  Our results before hinge on the data coverage and realizability assumptions. 
Further, we consider the case when such assumptions are violated. 
We denote by
\#\label{eq:v-tilde-def}
\tilde v^\pi \in \argmin_{v\in \cV} \max_{w\in \cW} \Phi_\vf^\pi(v,w; \Delta^*, \Theta^*), \qquad \tilde w^\pi\in \argmin_{w\in \cW} \max_{v\in \cV} \Phi_\mis^\pi(w, v; \Delta^*, \Theta^*)
\#
and introduce the following assumption. 
\begin{assumption}[Model Misspecification]\label{ass:model-spec}
The following statements hold.
\begin{enumerate}[label=(\alph*)]
    \item\label{ass:vf-spec} We have $\|V^\pi - \tilde v^\pi \|_\infty \leq \vareps^\cV_\vf$ for any $\pi\in \Pi$ and $\|w^{\pi^*} - \tilde w^{\pi^*}\|_\infty \leq \vareps^\cW_\vf$. 
    \item\label{ass:mis-spec} We have $\|w^\pi - \tilde w^\pi \|_\infty \leq \vareps^\cW_\mis$ for any $\pi\in \Pi$ and $\|V^{\pi^*} - \tilde v^{\pi^*}\|_\infty \leq \vareps^\cV_\mis$. 
\end{enumerate}
\end{assumption}

Though Assumption \ref{ass:model-spec} requires that \ref{ass:vf-spec} and \ref{ass:mis-spec} hold simultaneously, we remark that previous assumptions imposed in VF-, MIS-, and DR-based pessimism can be recovered by such an assumption. 
Specifically, Assumptions \ref{ass:iv-vf-realizable} and \ref{ass:iv-mis-realizable} can be recovered by taking $(\vareps^\cV_\vf, \vareps^\cW_\vf, \vareps^\cV_\mis, \vareps^\cW_\mis) = (0, 0, \infty, \infty)$ and $(\vareps^\cV_\vf, \vareps^\cW_\vf, \vareps^\cV_\mis, \vareps^\cW_\mis) = (\infty, \infty, 0, 0)$, respectively, in Assumption \ref{ass:model-spec}. 
Similarly, the data coverage and realizability assumptions in Theorem \ref{thm:iv-dr} can also be recovered by either taking $(\vareps^\cV_\vf, \vareps^\cW_\vf, \vareps^\cV_\mis, \vareps^\cW_\mis) = (0, 0, \infty, \infty)$ or taking $(\vareps^\cV_\vf, \vareps^\cW_\vf, \vareps^\cV_\mis, \vareps^\cW_\mis) = (\infty, \infty, 0, 0)$. To simplify the notation, we remark that $ \leq \infty$ means $< \infty$.

\begin{theorem}\label{thm:iv-dr-spec}
Let $(\alpha_0, \alpha_1, \alpha_\mis, \alpha_\vf)$ be those defined in Assumption \ref{ass:iv-sl-res}, Lemma \ref{lemma:iv-v-pi-in-conf}, and Lemma \ref{lemma:iv-w-pi-in-conf}. Then under Assumptions \ref{ass:iv-common}--\ref{ass:iv-sl-res}, and \ref{ass:model-spec}, it holds with probability at least $1 - \delta$ for any $c / (NT)^2 \leq \delta \leq 1$ that
\$
\subopt(\hat \pi_\dr) & \leq c\cdot \frac{C_{\Delta^*}^2 C_{\Theta^*}^2 C_*}{1-\gamma} (\xi_0 + \xi_1) \sqrt{\frac{1}{NT \kappa} \pdim_{\cF_0,\cF_1,\cW,\cV,\Pi} \log\frac{NT }{\delta}}  \\
& \qquad + 3 C_{\Delta^*} C_{\Theta^*} \min\left\{ C_* \vareps^\cV_\vf + \vareps^\cW_\vf/(1-\gamma), ~C_* \vareps^\cV_\mis + \vareps^\cW_\mis/(1-\gamma) \right \}, 
\$
where $\hat \pi_\dr$ is defined in \eqref{eq:iv-dr}. 
\end{theorem}
\begin{proof}
See \S\ref{prf:thm:iv-dr-spec} 
of the Supplementary Material for a detailed proof. 
\end{proof}

In Theorem \ref{thm:iv-dr-spec}, the first term on the right-hand side of the suboptimality upper bound corresponds to the suboptimality of the DR-based estimator in Theorem \ref{thm:iv-dr}, and the second term characterizes the additional bias induced by model misspecification. 
We remark that either $(\vareps^\cV_\vf, \vareps^\cW_\vf) = (0,0)$ or $(\vareps^\cV_\mis, \vareps^\cW_\mis) = (0,0)$ in Theorem \ref{thm:iv-dr-spec} ensures zero bias, which corresponds to Theorem \ref{thm:iv-dr}.

\section{Dual Formulation}\label{sec: dual form}

To improve the computational efficiency of estimating the optimal in-class policy due to the confidence sets, we propose a dual formulation of the aforementioned pessimistic methods. 
For the purpose of clear illustration, we only consider the dual formulation of the VF-based pessimistic method proposed in \S\ref{sec:iv-vf}. Similar formulations for MIS-based and DR-based methods can also be derived accordingly. 

For the ease of presentation, we assume that there exists an oracle that gives us $\Delta^*$ and $\Theta^*$. Without the existence of such an oracle, we only need to employ two additional dual variables to consider the uncertainty induced by estimating $\Delta^*$ and $\Theta^*$. We consider the following dual form of \eqref{eq:iv-vf-hat-pi}, 
\#\label{eq:vf-dual}
& \hat \pi_\vf^\dagger = \argmax_{\pi\in \Pi} \max_{\lambda \geq 0} \min_{v \in \cV} ~(1-\gamma)\EE_{S \sim \nu}[v(S)] + \lambda \cdot \left(  \hat M_\vf^\pi(v) - \alpha_\vf \right ), \\
& \text{s.t. } \hat M_\vf^\pi(v) = \max_{g\in \cW} \hat \Phi_\vf^\pi(v, g; \Delta^*, \Theta^*) - \max_{g\in \cW} \hat \Phi_\vf^\pi(\hat v_{\Delta^*, \Theta^*}^\pi, g; \Delta^*, \Theta^*), 
\#
where $\hat v_{\Delta^*, \Theta^*}^\pi = \argmin_{v\in \cV}\max_{g\in \cW} \hat \Phi_\vf^\pi(v, g; \Delta^*, \Theta^*)$ and $\lambda$ is the dual variable that corresponds to the constraint $v\in \ci^\vf_{\alpha_\vf}(\Delta^*, \Theta^*, \pi)$. 
In comparison to the constrained optimization problem in \eqref{eq:iv-vf-hat-pi}, the problem in \eqref{eq:vf-dual} can be solved efficiently using gradient-based methods. 

In the following theorem, we characterize the suboptimality of $\hat \pi_\vf^\dagger$.  

\begin{theorem}\label{thm:dual}
Suppose that $\cV$ is convex, $\alpha_\vf$ is defined in Lemma \ref{lemma:iv-v-in-conf-good}, and $c / (NT)^2 \leq \delta \leq 1$. 
Under Assumptions \ref{ass:iv-common}--\ref{ass:iv-vf-realizable}, it holds with probability at least $1 -\delta$ that
\$
\subopt(\hat \pi_\vf^\dagger) \leq c\cdot \frac{C_{\Delta^*}^2 C_{\Theta^*}^2 C_*}{1-\gamma} (\xi_0 + \xi_1)  L_\Pi \sqrt{\frac{1}{NT\kappa}\cdot \pdim_{\cF_0,\cF_1,\cW,\cV,\Pi} \cdot \log\frac{1}{\delta} \log(NT)}. 
\$
\end{theorem}

\begin{proof}
For notational convenience, we denote by $(\hat \pi, \hat \lambda, \hat v)$ the solution of \eqref{eq:vf-dual}. 
Note that $(1-\gamma) \EE_{S_0\sim \nu}[v(S_0)]$ is a lower bounded real-valued convex functional w.r.t. $v$, and $\hat M^{\pi^*}_\vf$ is also a convex functional. In the meanwhile, we have $\hat M^{\pi^*}_\vf(\hat v^{\pi^*}_{\Delta^*, \Theta^*}) = 0$. Thus, by Theorem 1 of \S{8.6} in \cite{luenberger1997optimization}, strong duality holds, i.e., 
\#\label{eq:wefuir8-pp}
& \max_{\lambda\geq 0} \min_{v\in \cV} \left\{ (1-\gamma) \EE_{S_0\sim \nu}\left [v(S_0)\right] + \lambda \cdot \left(\hat M_\vf^{\pi^*}(v) - \alpha_\vf \right ) \right\}  \\
& \qquad = \min_{v\in \cV} \max_{\lambda\geq 0} \left\{ (1-\gamma) \EE_{S_0\sim \nu}\left [v(S_0)\right] + \lambda \cdot \left(\hat M_\vf^{\pi^*}(v) - \alpha_\vf \right ) \right\}. 
\#
By Lemma \ref{lemma:iv-v-pi-in-conf}, it holds with probability at least $1 - \delta$ that $\hat M_\vf^{\hat \pi} (V^{\hat \pi}) \leq \alpha_\vf$. 
Thus, with probability at least $1 - \delta$, we have
\#\label{eq:fuiewhr}
& J(\pi^*) - J(\hat \pi)  \\
& \quad = (1-\gamma) \EE_{S_0\sim \nu}\left [V^{\pi^*}(S_0) - V^{\hat \pi}(S_0)\right] \\
& \quad \leq (1-\gamma) \EE_{S_0\sim \nu}\left [V^{\pi^*}(S_0)\right] - \left ( (1-\gamma) \EE_{S_0\sim \nu}\left [V^{\hat \pi}(S_0)\right] + \hat \lambda \cdot \left(  \hat M_\vf^{\hat \pi}(V^{\hat \pi}) - \alpha_\vf \right ) \right )  \\
& \quad \leq (1-\gamma) \EE_{S_0\sim \nu}\left [V^{\pi^*}(S_0)\right] - \min_{v\in \cV} \left\{ (1-\gamma) \EE_{S_0\sim \nu}\left [v(S_0)\right] + \hat \lambda \cdot \left(  \hat M_\vf^{\hat \pi}(v) - \alpha_\vf \right ) \right\}  \\
& \quad = (1-\gamma) \EE_{S_0\sim \nu}\left [V^{\pi^*}(S_0)\right] - \max_{\pi\in \Pi} \max_{\lambda\geq 0} \min_{v\in \cV} \left\{ (1-\gamma) \EE_{S_0\sim \nu}\left [v(S_0)\right] + \lambda \cdot \left(  \hat M_\vf^{\pi}(v) - \alpha_\vf \right ) \right\}  \\
& \quad = (1-\gamma) \EE_{S_0\sim \nu}\left [V^{\pi^*}(S_0)\right] - \max_{\lambda\geq 0} \min_{v\in \cV} \left\{ (1-\gamma) \EE_{S_0\sim \nu}\left [v(S_0)\right] + \lambda \cdot \left(  \hat M_\vf^{\pi^*}(v) - \alpha_\vf \right ) \right\}, 
\#
where in the second inequality, we use the fact that $V^{\hat \pi}\in \cV$; in the third inequality, we use the definition of $\hat \pi$ and $\hat \lambda$; in the last inequality, we use the fact that $\pi^* \in \Pi$. 
By combining \eqref{eq:wefuir8-pp} and \eqref{eq:fuiewhr}, we have
\#\label{eq:erfuie}
& J(\pi^*) - J(\hat \pi)  \\
& \quad \leq (1-\gamma) \EE_{S_0\sim \nu}\left [V^{\pi^*}(S_0)\right] - \min_{v\in \cV} \max_{\lambda\geq 0} \left\{ (1-\gamma) \EE_{S_0\sim \nu}\left [v(S_0)\right] + \lambda \cdot \left( \hat M_\vf^{\pi^*}(v) - \alpha_\vf \right ) \right\}  \\
& \quad \leq (1-\gamma) \EE_{S_0\sim \nu}\left [V^{\pi^*}(S_0)\right] - \min_{v\in \cV\colon \hat M_\vf^{\pi^*}(v) \leq \alpha_\vf} (1-\gamma) \EE_{S_0\sim \nu}\left [v(S_0)\right]  \\
& \quad \leq \max_{v\in \cV\colon \hat M_\vf^{\pi^*}(v) \leq \alpha_\vf} \left | (1-\gamma) \EE_{S_0\sim \nu}\left [V^{\pi^*}(S_0) - v(S_0)\right] \right |. 
\#
Note that by Lemmas \ref{lemma:iv-mis-j} and \ref{lemma:iv-est-eq}, we have
\#\label{eq:eriufhwe}
(1-\gamma) \EE_{S_0\sim \nu}\left [V^{\pi^*}(S_0) - v(S_0)\right] = \Phi_\vf^{\pi^*}(v, w^{\pi^*}; \Delta^*, \Theta^*). 
\#
By plugging \eqref{eq:eriufhwe} into \eqref{eq:erfuie}, we have
\$
J(\pi^*) - J(\hat \pi) & \leq \max_{v\in \cV\colon \hat M_\vf^{\pi^*}(v) \leq \alpha_\vf} \left | \Phi_\vf^{\pi^*}(v, w^{\pi^*}; \Delta^*, \Theta^*) \right |  \\
& \leq c\cdot \frac{C_{\Delta^*}^2 C_{\Theta^*}^2 C_*}{1-\gamma} (\xi_0 + \xi_1)  L_\Pi \sqrt{\frac{1}{NT\kappa}\cdot \pdim_{\cF_0,\cF_1,\cW,\cV,\Pi} \cdot \log\frac{1}{\delta} \log(NT)},
\$
where in the last inequality, we use Lemma \ref{lemma:iv-v-in-conf-good} with the fact that $w^{\pi^*}\in \cW$.  This concludes the proof. 
\end{proof}

In Theorem \ref{thm:dual}, with an additional assumption that $\cV$ is convex, we show that a similar suboptimality holds as in Theorem \ref{thm:iv-vf} for the VF-based method. 
Thus, to avoid computational challenges induced by the confidence sets in \eqref{eq:iv-vf-hat-pi}, we only need to solve \eqref{eq:vf-dual} to obtain an optimal policy. 
We remark that similar dual formulations for MIS-based and DR-based methods can be derived, as well as their theoretical properties.

\section{Identifiability}\label{sec: identification}

We discuss interesting identifiability results implied by the proposed algorithm in this section. Following the convention, we say a parameter $\theta$ is identifiable if $\theta \rightarrow \mathbb{P}_{\theta}$ is injective. On the contrary, non-identifiability implies that there exists two different parameters that their corresponding data distributions coincide.  In the following, we use a tabular MDP as an example to illustrate non-identifiability issue in RL. Then we discuss the identifiability required by our methods. 

\vskip5pt
\noindent\textbf{Non-Identifiability in Tabular MDP.}
We consider a tabular MDP with states $\cS = \{s_1, s_2, \ldots, s_{|\cS|}\}$, where the behavior policy $b$ used to generate offline data only covers states $\{s_2, \ldots, s_{|\cS|}\}$. We assume that the expected total reward under such a tabular MDP is $J(\pi)$ for any policy $\pi$. Since the offline data generated following $b$ never covers the state $s_1$, we cannot infer any information of the reward received at the state $s_1$. Thus, for any policy $\pi$ that arrives the state $s_1$ with a nonzero probability, we cannot identify the value $J(\pi)$ uniquely. In the meanwhile, the state-value function $V^\pi\colon \cS\to \RR$ is not uniquely identifiable for any policy $\pi$ (even for $\pi^*$), since the value $V^\pi(s_1)$ is not identifiable.

\vskip5pt
\noindent\textbf{Identifiability Required by Our Methods.}
In \S\ref{sec:iv-id} and \S\ref{sec:iv-pess}, we do not explicitly impose any identifiability assumptions. But certain identifiability assumptions are implied by our data coverage assumptions as follows. For the ease of presentation, we assume that there exists an oracle that gives us $\Delta^*$ and $\Theta^*$.
\begin{itemize}
    \item VF-based pessimistic algorithm. As imposed in Assumption \ref{ass:iv-vf-realizable}, we require that $w^{\pi^*}$ is upper bounded and modelled correctly. Thus, we know that the trajectory generated by the optimal policy $\pi^*$ is covered by the offline data, which implies that $J(\pi^*)$ is identifiable but not necessary for $J(\pi)$ for $\pi \neq \pi^*$. This can also be seen by the min-max estimation procedure such as \eqref{eq:v-pi-def}, which will correctly upper bound the policy evaluation error so that $J(\pi^*)$ is uniquely identified. Also see the proof of Lemma \ref{lemma:iv-v-in-conf-good}. We remark that we do not require our data distribution to uniquely identify $V^\pi$ for any $\pi \in \Pi$ and the IV-aided Bellman equation in Lemma \ref{lemma:iv-vf-id} could have multiple fixed point solutions. See \cite{chen2022well} for when the uniqueness can be implied and that $w^{\pi^*}$ is upper bounded does not imply $V^{\pi^*}$ is uniquely identified.
    \item MIS-based pessimistic algorithm. As imposed in Assumption \ref{ass:iv-mis-realizable}, we require that $w^{\pi}$ is upper bounded for any policy $\pi$ and modelled correctly. Thus, we know that the trajectory generated by any policy $\pi$ is covered by the offline data, which implies that $J(\pi)$ is identifiable for any $\pi$. Similarly, we do not require our data distribution to uniquely identify/estimate $w^\pi$ as the MIS-based estimating equation defined in Lemma \ref{lemma:iv-est-eq} may have multiple fixed point solutions. 
    \item DR-based pessimistic algorithm. Since either Assumption \ref{ass:iv-vf-realizable} or Assumption \ref{ass:iv-mis-realizable} holds, we require that $J(\pi^*)$ is identifiable or $J(\pi)$ is identifiable for any $\pi$. In either case, we do not require that $w^\pi$ and $V^\pi$ are uniquely identified by our data distribution.
\end{itemize}

\section{Conclusion}
In this paper, we study the offline RL in the face of unmeasured confounders. We focus on resolving the following two challenges: (i) the agent may be confounded by the unmeasured confounders; (ii) the offline data may not provide sufficient coverage. To resolve the first challenge, by employing IVs, we establish VF- and MIS-based identification results for the expected total reward in the confounded MDPs. To resolve the second challenge, we employ pessimism to achieve policy learning. Specifically, we propose VF- and MIS-based pessimistic policy estimators, which are constructed by maximizing the most conservative expected total reward associated with the estimated value function and density ratio, respectively. As a combination, we also propose a DR-based estimator. As for theoretical contributions, under mild coverage and realizability assumptions, we show that the suboptimalities of the proposed estimators are upper bounded by $O(\log(NT) (NT)^{-1/2})$. Further, we consider the case when the models are misspecified, i.e., previous coverage and realizability assumptions no longer hold. We remark that such a misspecified case is a unified framework of the aforementioned estimators. 
\bibliographystyle{ims}
\bibliography{rl_ref}

\begin{thebibliography}{79}
\expandafter\ifx\csname natexlab\endcsname\relax\def\natexlab#1{#1}\fi
\expandafter\ifx\csname url\endcsname\relax
  \def\url#1{\texttt{#1}}\fi
\expandafter\ifx\csname urlprefix\endcsname\relax\def\urlprefix{}\fi

\bibitem[{Angrist and Imbens(1995)}]{angrist1995identification}
\text{Angrist, J.} and \text{Imbens, G.} (1995).
\newblock Identification and estimation of local average treatment effects.

\bibitem[{Angrist et~al.(1996)Angrist, Imbens and
  Rubin}]{angrist1996identification}
\text{Angrist, J.~D.}, \text{Imbens, G.~W.} and \text{Rubin, D.~B.} (1996).
\newblock Identification of causal effects using instrumental variables.
\newblock \textit{Journal of the American statistical Association}, \textbf{91}
  444--455.

\bibitem[{Antos et~al.(2008)Antos, Szepesv{\'a}ri and
  Munos}]{antos2008learning}
\text{Antos, A.}, \text{Szepesv{\'a}ri, C.} and \text{Munos, R.} (2008).
\newblock Learning near-optimal policies with {B}ellman-residual minimization
  based fitted policy iteration and a single sample path.
\newblock \textit{Machine Learning} 89--129.

\bibitem[{Arshad et~al.(2019)Arshad, Anderson and Sharif}]{Arshad2019}
\text{Arshad, A.}, \text{Anderson, B.} and \text{Sharif, A.} (2019).
\newblock {Comparison of organ donation and transplantation rates between
  opt-out and opt-in systems}.
\newblock \textit{Kidney International}, \textbf{95} 1453--1460.

\bibitem[{Barrera and Gobet(2021)}]{barrera2021generalization}
\text{Barrera, D.} and \text{Gobet, E.} (2021).
\newblock Generalization bounds for nonparametric regression with
  $\beta$-mixing samples.
\newblock \textit{arXiv preprint arXiv:2108.00997}.

\bibitem[{Bennett et~al.(2021)Bennett, Kallus, Li and Mousavi}]{bennett2021off}
\text{Bennett, A.}, \text{Kallus, N.}, \text{Li, L.} and \text{Mousavi, A.}
  (2021).
\newblock Off-policy evaluation in infinite-horizon reinforcement learning with
  latent confounders.
\newblock In \textit{International Conference on Artificial Intelligence and
  Statistics}. PMLR.

\bibitem[{Berbee(1979)}]{Berbee1979RandomWW}
\text{Berbee, H. C.~P.} (1979).
\newblock Random walks with stationary increments and renewal theory.

\bibitem[{Brookhart et~al.(2010)Brookhart, St{\"u}rmer, Glynn, Rassen and
  Schneeweiss}]{brookhart2010confounding}
\text{Brookhart, M.~A.}, \text{St{\"u}rmer, T.}, \text{Glynn, R.~J.},
  \text{Rassen, J.} and \text{Schneeweiss, S.} (2010).
\newblock Confounding control in healthcare database research: challenges and
  potential approaches.
\newblock \textit{Medical care}, \textbf{48} S114.

\bibitem[{Chen and Jiang(2019)}]{chen2019information}
\text{Chen, J.} and \text{Jiang, N.} (2019).
\newblock Information-theoretic considerations in batch reinforcement learning.
\newblock \textit{arXiv preprint arXiv:1905.00360}.

\bibitem[{Chen and Zhang(2021)}]{chen2021estimating}
\text{Chen, S.} and \text{Zhang, B.} (2021).
\newblock Estimating and improving dynamic treatment regimes with a
  time-varying instrumental variable.
\newblock \textit{arXiv preprint arXiv:2104.07822}.

\bibitem[{Chen and Qi(2022)}]{chen2022well}
\text{Chen, X.} and \text{Qi, Z.} (2022).
\newblock On well-posedness and minimax optimal rates of nonparametric
  q-function estimation in off-policy evaluation.
\newblock \textit{arXiv preprint arXiv:2201.06169}.

\bibitem[{Cho et~al.(2022)Cho, Holloway, Couper and Kosorok}]{cho2022}
\text{Cho, H.}, \text{Holloway, S.~T.}, \text{Couper, D.~J.} and \text{Kosorok,
  M.~R.} (2022).
\newblock {Multi-stage optimal dynamic treatment regimes for survival outcomes
  with dependent censoring}.
\newblock \textit{Biometrika}.
\newblock Asac047.
\newline\urlprefix\url{https://doi.org/10.1093/biomet/asac047}

\bibitem[{Cui and Tchetgen~Tchetgen(2021)}]{cui2021semiparametric}
\text{Cui, Y.} and \text{Tchetgen~Tchetgen, E.} (2021).
\newblock A semiparametric instrumental variable approach to optimal treatment
  regimes under endogeneity.
\newblock \textit{Journal of the American Statistical Association},
  \textbf{116} 162--173.

\bibitem[{Deng et~al.(2021)Deng, Fu, Wang, Yang, Bai, Wang and
  Jiang}]{deng2021score}
\text{Deng, Z.}, \text{Fu, Z.}, \text{Wang, L.}, \text{Yang, Z.}, \text{Bai,
  C.}, \text{Wang, Z.} and \text{Jiang, J.} (2021).
\newblock Score: Spurious correlation reduction for offline reinforcement
  learning.
\newblock \textit{arXiv preprint arXiv:2110.12468}.

\bibitem[{Ernst et~al.(2005)Ernst, Geurts and Wehenkel}]{ernst2005tree}
\text{Ernst, D.}, \text{Geurts, P.} and \text{Wehenkel, L.} (2005).
\newblock Tree-based batch mode reinforcement learning.
\newblock \textit{Journal of Machine Learning Research} 503--556.

\bibitem[{Ertefaie and Strawderman(2018)}]{ertefaie2018constructing}
\text{Ertefaie, A.} and \text{Strawderman, R.~L.} (2018).
\newblock Constructing dynamic treatment regimes over indefinite time horizons.
\newblock \textit{Biometrika}, \textbf{105} 963--977.

\bibitem[{Farahmand et~al.(2016)Farahmand, Ghavamzadeh, Szepesv{\'a}ri and
  Mannor}]{farahmand2016regularized}
\text{Farahmand, A.-m.}, \text{Ghavamzadeh, M.}, \text{Szepesv{\'a}ri, C.} and
  \text{Mannor, S.} (2016).
\newblock Regularized policy iteration with nonparametric function spaces.
\newblock \textit{The Journal of Machine Learning Research}, \textbf{17}
  4809--4874.

\bibitem[{Geer et~al.(2000)Geer, van~de Geer and Williams}]{geer2000empirical}
\text{Geer, S.~A.}, \text{van~de Geer, S.} and \text{Williams, D.} (2000).
\newblock \textit{Empirical Processes in M-estimation}, vol.~6.
\newblock Cambridge university press.

\bibitem[{Gottesman et~al.(2019)Gottesman, Johansson, Komorowski, Faisal,
  Sontag, Doshi-Velez and Celi}]{gottesman2019guidelines}
\text{Gottesman, O.}, \text{Johansson, F.}, \text{Komorowski, M.},
  \text{Faisal, A.}, \text{Sontag, D.}, \text{Doshi-Velez, F.} and \text{Celi,
  L.~A.} (2019).
\newblock Guidelines for reinforcement learning in healthcare.
\newblock \textit{Nature medicine}, \textbf{25} 16--18.

\bibitem[{Hern{\'a}n and Robins(2010)}]{hernan2010causal}
\text{Hern{\'a}n, M.~A.} and \text{Robins, J.~M.} (2010).
\newblock Causal inference: What if.

\bibitem[{Hua et~al.(2021)Hua, Mei, Zohar, Giral and Xu}]{hua2020personalized}
\text{Hua, W.}, \text{Mei, H.}, \text{Zohar, S.}, \text{Giral, M.} and
  \text{Xu, Y.} (2021).
\newblock Personalized dynamic treatment regimes in continuous time: A
  {B}ayesian joint model for optimizing clinical decisions with timing.
\newblock \textit{Bayesian Analysis}.

\bibitem[{Jiang and Huang(2020)}]{jiang2020minimax}
\text{Jiang, N.} and \text{Huang, J.} (2020).
\newblock Minimax value interval for off-policy evaluation and policy
  optimization.
\newblock \textit{Advances in Neural Information Processing Systems},
  \textbf{33} 2747--2758.

\bibitem[{Jiang and Li(2016)}]{jiang2016doubly}
\text{Jiang, N.} and \text{Li, L.} (2016).
\newblock Doubly robust off-policy value evaluation for reinforcement learning.
\newblock In \textit{International Conference on Machine Learning}. PMLR.

\bibitem[{Jin et~al.(2021)Jin, Yang and Wang}]{jin2021pessimism}
\text{Jin, Y.}, \text{Yang, Z.} and \text{Wang, Z.} (2021).
\newblock Is pessimism provably efficient for offline rl?
\newblock In \textit{International Conference on Machine Learning}. PMLR.

\bibitem[{Kalashnikov et~al.(2018)Kalashnikov, Irpan, Pastor, Ibarz, Herzog,
  Jang, Quillen, Holly, Kalakrishnan, Vanhoucke
  et~al.}]{kalashnikov2018scalable}
\text{Kalashnikov, D.}, \text{Irpan, A.}, \text{Pastor, P.}, \text{Ibarz, J.},
  \text{Herzog, A.}, \text{Jang, E.}, \text{Quillen, D.}, \text{Holly, E.},
  \text{Kalakrishnan, M.}, \text{Vanhoucke, V.} \text{et~al.} (2018).
\newblock Scalable deep reinforcement learning for vision-based robotic
  manipulation.
\newblock In \textit{Conference on Robot Learning}. PMLR.

\bibitem[{Kallus and Uehara(2020)}]{kallus2020double}
\text{Kallus, N.} and \text{Uehara, M.} (2020).
\newblock Double reinforcement learning for efficient off-policy evaluation in
  markov decision processes.
\newblock \textit{J. Mach. Learn. Res.}, \textbf{21} 1--63.

\bibitem[{Kallus and Uehara(2022)}]{kallus2022efficiently}
\text{Kallus, N.} and \text{Uehara, M.} (2022).
\newblock Efficiently breaking the curse of horizon in off-policy evaluation
  with double reinforcement learning.
\newblock \textit{Operations Research}.

\bibitem[{Kallus and Zhou(2020)}]{kallus2020confounding}
\text{Kallus, N.} and \text{Zhou, A.} (2020).
\newblock Confounding-robust policy evaluation in infinite-horizon
  reinforcement learning.
\newblock \textit{Advances in Neural Information Processing Systems},
  \textbf{33} 22293--22304.

\bibitem[{Kasiske et~al.(2010)Kasiske, Zeier, Chapman, Craig, Ekberg, Garvey,
  Green, Jha, Josephson, Kiberd, Kreis, McDonald, Newmann, Obrador, Vincenti,
  Cheung, Earley, Raman, Abariga, Wagner and Balk}]{Kasiske2010}
\text{Kasiske, B.~L.}, \text{Zeier, M.~G.}, \text{Chapman, J.~R.}, \text{Craig,
  J.~C.}, \text{Ekberg, H.}, \text{Garvey, C.~A.}, \text{Green, M.~D.},
  \text{Jha, V.}, \text{Josephson, M.~A.}, \text{Kiberd, B.~A.}, \text{Kreis,
  H.~A.}, \text{McDonald, R.~A.}, \text{Newmann, J.~M.}, \text{Obrador, G.~T.},
  \text{Vincenti, F.~G.}, \text{Cheung, M.}, \text{Earley, A.}, \text{Raman,
  G.}, \text{Abariga, S.}, \text{Wagner, M.} and \text{Balk, E.~M.} (2010).
\newblock {KDIGO clinical practice guideline for the care of kidney transplant
  recipients: A summary}.
\newblock \textit{Kidney International}, \textbf{77} 299--311.

\bibitem[{Kidambi et~al.(2020)Kidambi, Rajeswaran, Netrapalli and
  Joachims}]{kidambiMorelModelbasedOffline2020}
\text{Kidambi, R.}, \text{Rajeswaran, A.}, \text{Netrapalli, P.} and
  \text{Joachims, T.} (2020).
\newblock {MOReL}: {{Model}}-based offline reinforcement learning.
\newblock \textit{arXiv preprint arXiv:2005.05951}.

\bibitem[{Komorowski et~al.(2018)Komorowski, Celi, Badawi, Gordon and
  Faisal}]{komorowski2018artificial}
\text{Komorowski, M.}, \text{Celi, L.~A.}, \text{Badawi, O.}, \text{Gordon,
  A.~C.} and \text{Faisal, A.~A.} (2018).
\newblock The artificial intelligence clinician learns optimal treatment
  strategies for sepsis in intensive care.
\newblock \textit{Nature medicine}, \textbf{24} 1716--1720.

\bibitem[{Kosorok and Laber(2019)}]{kosorok2019precision}
\text{Kosorok, M.~R.} and \text{Laber, E.~B.} (2019).
\newblock Precision medicine.
\newblock \textit{Annual review of statistics and its application}, \textbf{6}
  263.

\bibitem[{Kumar et~al.(2020)Kumar, Zhou, Tucker and
  Levine}]{kumarConservativeQlearningOffline2020}
\text{Kumar, A.}, \text{Zhou, A.}, \text{Tucker, G.} and \text{Levine, S.}
  (2020).
\newblock Conservative {Q}-learning for offline reinforcement learning.
\newblock \textit{arXiv preprint arXiv:2006.04779}.

\bibitem[{Levine et~al.(2020)Levine, Kumar, Tucker and Fu}]{levine2020offline}
\text{Levine, S.}, \text{Kumar, A.}, \text{Tucker, G.} and \text{Fu, J.}
  (2020).
\newblock Offline reinforcement learning: Tutorial, review, and perspectives on
  open problems.
\newblock \textit{arXiv preprint arXiv:2005.01643}.

\bibitem[{Liao et~al.(2021{\natexlab{a}})Liao, Fu, Yang, Wang, Kolar and
  Wang}]{liao2021instrumental}
\text{Liao, L.}, \text{Fu, Z.}, \text{Yang, Z.}, \text{Wang, Y.}, \text{Kolar,
  M.} and \text{Wang, Z.} (2021{\natexlab{a}}).
\newblock Instrumental variable value iteration for causal offline
  reinforcement learning.
\newblock \textit{arXiv preprint arXiv:2102.09907}.

\bibitem[{Liao et~al.(2021{\natexlab{b}})Liao, Klasnja and
  Murphy}]{liao2021off}
\text{Liao, P.}, \text{Klasnja, P.} and \text{Murphy, S.} (2021{\natexlab{b}}).
\newblock Off-policy estimation of long-term average outcomes with applications
  to mobile health.
\newblock \textit{Journal of the American Statistical Association},
  \textbf{116} 382--391.

\bibitem[{Liao et~al.(2020)Liao, Qi, Klasnja and Murphy}]{liao2020batch}
\text{Liao, P.}, \text{Qi, Z.}, \text{Klasnja, P.} and \text{Murphy, S.}
  (2020).
\newblock Batch policy learning in average reward markov decision processes.
\newblock \textit{arXiv preprint arXiv:2007.11771}.

\bibitem[{Liu et~al.(2018)Liu, Li, Tang and Zhou}]{liu2018breaking}
\text{Liu, Q.}, \text{Li, L.}, \text{Tang, Z.} and \text{Zhou, D.} (2018).
\newblock Breaking the curse of horizon: Infinite-horizon off-policy
  estimation.
\newblock \textit{Advances in Neural Information Processing Systems},
  \textbf{31}.

\bibitem[{Lu et~al.(2022)Lu, Min, Wang and Yang}]{lu2022pessimism}
\text{Lu, M.}, \text{Min, Y.}, \text{Wang, Z.} and \text{Yang, Z.} (2022).
\newblock Pessimism in the face of confounders: Provably efficient offline
  reinforcement learning in partially observable markov decision processes.
\newblock \textit{arXiv preprint arXiv:2205.13589}.

\bibitem[{Luckett et~al.(2019)Luckett, Laber, Kahkoska, Maahs, Mayer-Davis and
  Kosorok}]{luckett2019estimating}
\text{Luckett, D.~J.}, \text{Laber, E.~B.}, \text{Kahkoska, A.~R.},
  \text{Maahs, D.~M.}, \text{Mayer-Davis, E.} and \text{Kosorok, M.~R.} (2019).
\newblock Estimating dynamic treatment regimes in mobile health using
  v-learning.
\newblock \textit{Journal of the American Statistical Association} 1--39.

\bibitem[{Luenberger(1997)}]{luenberger1997optimization}
\text{Luenberger, D.~G.} (1997).
\newblock \textit{Optimization by vector space methods}.
\newblock John Wiley \& Sons.

\bibitem[{Meitz and Saikkonen(2021)}]{meitz2021subgeometric}
\text{Meitz, M.} and \text{Saikkonen, P.} (2021).
\newblock Subgeometric ergodicity and $\beta$-mixing.
\newblock \textit{Journal of Applied Probability}, \textbf{58} 594--608.

\bibitem[{Munos and Szepesv{\'a}ri(2008)}]{munos2008finite}
\text{Munos, R.} and \text{Szepesv{\'a}ri, C.} (2008).
\newblock Finite-time bounds for fitted value iteration.
\newblock \textit{Journal of Machine Learning Research}, \textbf{9} 815--857.

\bibitem[{Nachum et~al.(2019)Nachum, Chow, Dai and Li}]{nachum2019dualdice}
\text{Nachum, O.}, \text{Chow, Y.}, \text{Dai, B.} and \text{Li, L.} (2019).
\newblock Dualdice: Behavior-agnostic estimation of discounted stationary
  distribution corrections.
\newblock \textit{Advances in Neural Information Processing Systems},
  \textbf{32}.

\bibitem[{Namkoong et~al.(2020)Namkoong, Keramati, Yadlowsky and
  Brunskill}]{namkoong2020off}
\text{Namkoong, H.}, \text{Keramati, R.}, \text{Yadlowsky, S.} and
  \text{Brunskill, E.} (2020).
\newblock Off-policy policy evaluation for sequential decisions under
  unobserved confounding.
\newblock \textit{Advances in Neural Information Processing Systems},
  \textbf{33} 18819--18831.

\bibitem[{OpenAI(2018)}]{OpenAI_dota}
\text{OpenAI} (2018).
\newblock Openai five.
\newblock \url{https://blog.openai.com/openai-five/}.

\bibitem[{Pearl(2009)}]{pearl2009causality}
\text{Pearl, J.} (2009).
\newblock \textit{Causality}.
\newblock Cambridge university press.

\bibitem[{Peters et~al.(2017)Peters, Janzing and
  Sch{\"o}lkopf}]{peters2017elements}
\text{Peters, J.}, \text{Janzing, D.} and \text{Sch{\"o}lkopf, B.} (2017).
\newblock \textit{Elements of causal inference: foundations and learning
  algorithms}.
\newblock The MIT Press.

\bibitem[{Precup(2000)}]{precup2000eligibility}
\text{Precup, D.} (2000).
\newblock Eligibility traces for off-policy policy evaluation.
\newblock \textit{Computer Science Department Faculty Publication Series} 80.

\bibitem[{Puterman(2014)}]{puterman2014markov}
\text{Puterman, M.~L.} (2014).
\newblock \textit{Markov \relax{D}ecision \relax{P}rocesses: \relax{D}iscrete
  Stochastic Dynamic Programming}.
\newblock John Wiley \& Sons.

\bibitem[{Raghu et~al.(2017)Raghu, Komorowski, Celi, Szolovits and
  Ghassemi}]{raghu2017continuous}
\text{Raghu, A.}, \text{Komorowski, M.}, \text{Celi, L.~A.}, \text{Szolovits,
  P.} and \text{Ghassemi, M.} (2017).
\newblock Continuous state-space models for optimal sepsis treatment: a deep
  reinforcement learning approach.
\newblock In \textit{Machine Learning for Healthcare Conference}. PMLR.

\bibitem[{Rashidinejad et~al.(2021)Rashidinejad, Zhu, Ma, Jiao and
  Russell}]{rashidinejad2021bridging}
\text{Rashidinejad, P.}, \text{Zhu, B.}, \text{Ma, C.}, \text{Jiao, J.} and
  \text{Russell, S.} (2021).
\newblock Bridging offline reinforcement learning and imitation learning: A
  tale of pessimism.
\newblock \textit{Advances in Neural Information Processing Systems},
  \textbf{34} 11702--11716.

\bibitem[{Shalev-Shwartz et~al.(2016)Shalev-Shwartz, Shammah and
  Shashua}]{shalev2016safe}
\text{Shalev-Shwartz, S.}, \text{Shammah, S.} and \text{Shashua, A.} (2016).
\newblock Safe, multi-agent, reinforcement learning for autonomous driving.
\newblock \textit{arXiv preprint arXiv:1610.03295}.

\bibitem[{Shi et~al.(2021)Shi, Uehara and Jiang}]{shi2021minimax}
\text{Shi, C.}, \text{Uehara, M.} and \text{Jiang, N.} (2021).
\newblock A minimax learning approach to off-policy evaluation in partially
  observable markov decision processes.
\newblock \textit{arXiv preprint arXiv:2111.06784}.

\bibitem[{Shi et~al.(2020)Shi, Zhang, Lu and Song}]{shi2020statistical}
\text{Shi, C.}, \text{Zhang, S.}, \text{Lu, W.} and \text{Song, R.} (2020).
\newblock Statistical inference of the value function for reinforcement
  learning in infinite horizon settings.
\newblock \textit{arXiv preprint arXiv:2001.04515}.

\bibitem[{Shi et~al.(2022{\natexlab{a}})Shi, Zhu, Shen, Luo, Zhu and
  Song}]{shi2022off}
\text{Shi, C.}, \text{Zhu, J.}, \text{Shen, Y.}, \text{Luo, S.}, \text{Zhu, H.}
  and \text{Song, R.} (2022{\natexlab{a}}).
\newblock Off-policy confidence interval estimation with confounded markov
  decision process.
\newblock \textit{arXiv preprint arXiv:2202.10589}.

\bibitem[{Shi et~al.(2022{\natexlab{b}})Shi, Li, Wei, Chen and
  Chi}]{shi2022pessimistic}
\text{Shi, L.}, \text{Li, G.}, \text{Wei, Y.}, \text{Chen, Y.} and \text{Chi,
  Y.} (2022{\natexlab{b}}).
\newblock Pessimistic q-learning for offline reinforcement learning: Towards
  optimal sample complexity.
\newblock \textit{arXiv preprint arXiv:2202.13890}.

\bibitem[{Silver et~al.(2016)Silver, Huang, Maddison, Guez, Sifre, Van
  Den~Driessche, Schrittwieser, Antonoglou, Panneershelvam, Lanctot
  et~al.}]{silver2016mastering}
\text{Silver, D.}, \text{Huang, A.}, \text{Maddison, C.~J.}, \text{Guez, A.},
  \text{Sifre, L.}, \text{Van Den~Driessche, G.}, \text{Schrittwieser, J.},
  \text{Antonoglou, I.}, \text{Panneershelvam, V.}, \text{Lanctot, M.}
  \text{et~al.} (2016).
\newblock Mastering the game of {G}o with deep neural networks and tree search.
\newblock \textit{Nature} 484--489.

\bibitem[{Sutton and Barto(2018)}]{sutton2018reinforcement}
\text{Sutton, R.~S.} and \text{Barto, A.~G.} (2018).
\newblock \textit{Reinforcement Learning: An Introduction}.
\newblock MIT press.

\bibitem[{Tang et~al.(2019)Tang, Feng, Li, Zhou and Liu}]{tang2019doubly}
\text{Tang, Z.}, \text{Feng, Y.}, \text{Li, L.}, \text{Zhou, D.} and \text{Liu,
  Q.} (2019).
\newblock Doubly robust bias reduction in infinite horizon off-policy
  estimation.
\newblock \textit{arXiv preprint arXiv:1910.07186}.

\bibitem[{Tchetgen et~al.(2020)Tchetgen, Ying, Cui, Shi and
  Miao}]{tchetgen2020introduction}
\text{Tchetgen, E. J.~T.}, \text{Ying, A.}, \text{Cui, Y.}, \text{Shi, X.} and
  \text{Miao, W.} (2020).
\newblock An introduction to proximal causal learning.
\newblock \textit{arXiv preprint arXiv:2009.10982}.

\bibitem[{Thomas and Brunskill(2016)}]{thomas2016data}
\text{Thomas, P.} and \text{Brunskill, E.} (2016).
\newblock Data-efficient off-policy policy evaluation for reinforcement
  learning.
\newblock In \textit{International Conference on Machine Learning}. PMLR.

\bibitem[{Uehara et~al.(2020)Uehara, Huang and Jiang}]{uehara2020minimax}
\text{Uehara, M.}, \text{Huang, J.} and \text{Jiang, N.} (2020).
\newblock Minimax weight and q-function learning for off-policy evaluation.
\newblock In \textit{International Conference on Machine Learning}. PMLR.

\bibitem[{Uehara et~al.(2021)Uehara, Imaizumi, Jiang, Kallus, Sun and
  Xie}]{uehara2021finite}
\text{Uehara, M.}, \text{Imaizumi, M.}, \text{Jiang, N.}, \text{Kallus, N.},
  \text{Sun, W.} and \text{Xie, T.} (2021).
\newblock Finite sample analysis of minimax offline reinforcement learning:
  Completeness, fast rates and first-order efficiency.
\newblock \textit{arXiv preprint arXiv:2102.02981}.

\bibitem[{Van~Roy(1998)}]{van1998learning}
\text{Van~Roy, B.} (1998).
\newblock \textit{Learning and value function approximation in complex decision
  processes}.
\newblock Ph.D. thesis, Massachusetts Institute of Technology.

\bibitem[{Wang et~al.(2021{\natexlab{a}})Wang, Qi and Wong}]{wang2021projected}
\text{Wang, J.}, \text{Qi, Z.} and \text{Wong, R.~K.} (2021{\natexlab{a}}).
\newblock Projected state-action balancing weights for offline reinforcement
  learning.
\newblock \textit{arXiv preprint arXiv:2109.04640}.

\bibitem[{Wang and Tchetgen~Tchetgen(2018)}]{wang2018bounded}
\text{Wang, L.} and \text{Tchetgen~Tchetgen, E.} (2018).
\newblock Bounded, efficient and multiply robust estimation of average
  treatment effects using instrumental variables.
\newblock \textit{Journal of the Royal Statistical Society: Series B
  (Statistical Methodology)}, \textbf{80} 531--550.

\bibitem[{Wang et~al.(2021{\natexlab{b}})Wang, Yang and
  Wang}]{wang2021provably}
\text{Wang, L.}, \text{Yang, Z.} and \text{Wang, Z.} (2021{\natexlab{b}}).
\newblock Provably efficient causal reinforcement learning with confounded
  observational data.
\newblock \textit{Advances in Neural Information Processing Systems},
  \textbf{34} 21164--21175.

\bibitem[{Wang et~al.(2020)Wang, Foster and Kakade}]{wang2020statistical}
\text{Wang, R.}, \text{Foster, D.~P.} and \text{Kakade, S.~M.} (2020).
\newblock What are the statistical limits of offline rl with linear function
  approximation?
\newblock \textit{arXiv preprint arXiv:2010.11895}.

\bibitem[{Wang et~al.(2021{\natexlab{c}})Wang, Wu, Salakhutdinov and
  Kakade}]{wang2021instabilities}
\text{Wang, R.}, \text{Wu, Y.}, \text{Salakhutdinov, R.} and \text{Kakade, S.}
  (2021{\natexlab{c}}).
\newblock Instabilities of offline rl with pre-trained neural representation.
\newblock In \textit{International Conference on Machine Learning}. PMLR.

\bibitem[{Xie et~al.(2021)Xie, Cheng, Jiang, Mineiro and
  Agarwal}]{xie2021bellman}
\text{Xie, T.}, \text{Cheng, C.-A.}, \text{Jiang, N.}, \text{Mineiro, P.} and
  \text{Agarwal, A.} (2021).
\newblock Bellman-consistent pessimism for offline reinforcement learning.
\newblock \textit{Advances in neural information processing systems},
  \textbf{34}.

\bibitem[{Yan et~al.(2022)Yan, Li, Chen and Fan}]{yan2022efficacy}
\text{Yan, Y.}, \text{Li, G.}, \text{Chen, Y.} and \text{Fan, J.} (2022).
\newblock The efficacy of pessimism in asynchronous q-learning.
\newblock \textit{arXiv preprint arXiv:2203.07368}.

\bibitem[{Yu et~al.(2020)Yu, Thomas, Yu, Ermon, Zou, Levine, Finn and
  Ma}]{yuMopoModelbasedOffline2020}
\text{Yu, T.}, \text{Thomas, G.}, \text{Yu, L.}, \text{Ermon, S.}, \text{Zou,
  J.}, \text{Levine, S.}, \text{Finn, C.} and \text{Ma, T.} (2020).
\newblock {MOPO}: Model-based offline policy optimization.
\newblock \textit{arXiv preprint arXiv:2005.13239}.

\bibitem[{Zhan et~al.(2022)Zhan, Huang, Huang, Jiang and Lee}]{zhan2022offline}
\text{Zhan, W.}, \text{Huang, B.}, \text{Huang, A.}, \text{Jiang, N.} and
  \text{Lee, J.~D.} (2022).
\newblock Offline reinforcement learning with realizability and single-policy
  concentrability.
\newblock \textit{arXiv preprint arXiv:2202.04634}.

\bibitem[{Zhang and Liu(2014)}]{zhang2014multicategory}
\text{Zhang, C.} and \text{Liu, Y.} (2014).
\newblock Multicategory angle-based large-margin classification.
\newblock \textit{Biometrika}, \textbf{101} 625--640.

\bibitem[{Zhang and Bareinboim(2019)}]{zhang2019near}
\text{Zhang, J.} and \text{Bareinboim, E.} (2019).
\newblock Near-optimal reinforcement learning in dynamic treatment regimes.
\newblock \textit{Advances in Neural Information Processing Systems},
  \textbf{32}.

\bibitem[{Zhang et~al.(2020)Zhang, Dai, Li and Schuurmans}]{zhang2020gendice}
\text{Zhang, R.}, \text{Dai, B.}, \text{Li, L.} and \text{Schuurmans, D.}
  (2020).
\newblock Gendice: Generalized offline estimation of stationary values.
\newblock \textit{arXiv preprint arXiv:2002.09072}.

\bibitem[{Zhou et~al.(2021)Zhou, Zhu and Qu}]{zhou2021estimating}
\text{Zhou, W.}, \text{Zhu, R.} and \text{Qu, A.} (2021).
\newblock Estimating optimal infinite horizon dynamic treatment regimes via
  pt-learning.
\newblock \textit{arXiv preprint arXiv:2110.10719}.

\bibitem[{Zhou et~al.(2017)Zhou, Mayer-Hamblett, Khan and
  Kosorok}]{zhou2017residual}
\text{Zhou, X.}, \text{Mayer-Hamblett, N.}, \text{Khan, U.} and \text{Kosorok,
  M.~R.} (2017).
\newblock Residual weighted learning for estimating individualized treatment
  rules.
\newblock \textit{Journal of the American Statistical Association},
  \textbf{112} 169--187.

\end{thebibliography}

\appendix
\section{Analysis of Maximum Likelihood Estimation}
In this section, we provide theoretical results on the estimation of $\Theta^*$ and $\Delta^*$.  We denote by 
\$
L_1(\Theta) = -\EE\left[ \frac{1}{T} \sum_{t = 0}^{T-1} \log \Theta(S_t, Z_t) \right]
\$ 
the population counterpart of $\hat L_1$. 
We define 
\#\label{eq:theta-hellinger}
H^2(\Theta_1, \Theta_2) = \frac{1}{2} \cdot \EE \left[\frac{1}{T} \sum_{t = 0}^{T-1} \int \left( \sqrt{\Theta_1(S_t, z)} - \sqrt{\Theta_2(S_t, z)} \right)^2 \ud z \right]. 
\#
We have the following supporting result. 
\begin{lemma}\label{lemma:iv-mle-1}
Under Assumption \ref{ass:spaces}~\ref{ass:upper-bound-delta}, for any $\Theta_1, \Theta_2 \in \cF_1$, it holds with probability at least $1 - \delta$ for any $c / (NT)^2 \leq \delta \leq 1$ that
\$
& \left| \left(L_1(\Theta_1) - L_1(\Theta_2)\right) - \left(\hat L_1(\Theta_1) - \hat L_1(\Theta_2)\right) \right|  \\
& \qquad \leq c\cdot \frac{\log C_{\Theta^*}}{NT \kappa} \log \frac{1}{\delta} \log(NT) + c\cdot \sqrt{\frac{C_{\Theta^*}}{NT \kappa} H^2(\Theta_1, \Theta_2) \log\frac{1}{\delta} \log(NT) }, 
\$
where $H^2(\Theta_1, \Theta_2)$ is defined in \eqref{eq:theta-hellinger}. 
\end{lemma}
\begin{proof}
By Theorem \ref{thm:bernstein-mixing}, it holds with probability at least $1 - \delta$ that 
\#\label{eq:04939335}
& \left| \left(L_1(\Theta_1) - L_1(\Theta_2)\right) - \left(\hat L_1(\Theta_1) - \hat L_1(\Theta_2)\right) \right|  \\
& \quad \leq c\cdot \frac{\log C_{\Theta^*}}{NT \kappa} \log \frac{1}{\delta} \log(NT) + c\cdot \sqrt{\frac{1}{NT \kappa} \EE\left[\frac{1}{T} \sum_{t = 0}^{T-1} \left( \log\frac{\Theta_1(S_t,Z_t)}{\Theta_2(S_t,Z_t)} \right)^2 \right] \log\frac{1}{\delta} \log(NT) }. 
\#
Now, it suffices to upper bound the variance term on the RHS of the above inequality. Note that $\log x \leq 2(\sqrt{x}-1)$ for any $x>0$. Thus, for any $s\in \cS$, we have
\#\label{eq:iv-kl-hell}
& \int \Theta^*(s,z)\left(\log\frac{\Theta_1(s,z)}{\Theta_2(s,z)}\right)^2 \ud z\\
& \qquad \leq 4 \int \Theta^* \max\left\{\left(\sqrt{\frac{\Theta_2}{\Theta_1}} - 1\right)^2, \left(\sqrt{\frac{\Theta_1}{\Theta_2}} - 1\right)^2\right\} \ud z  \\
& \qquad = 4 \int \max\left\{\frac{\Theta^*}{\Theta_1}\left(\sqrt{\Theta_2} - \sqrt{\Theta_1}\right)^2, \frac{\Theta^*}{\Theta_2}\left(\sqrt{\Theta_1} - \sqrt{\Theta_2}\right)^2\right\} \ud z  \\
& \qquad \leq 4C_{\Theta^*} \int \left(\sqrt{\Theta_1(s,z)} - \sqrt{\Theta_2(s,z)}\right)^2 \ud z, 
\#
which implies that
\#\label{eq:049395}
\EE\left[\frac{1}{T} \sum_{t = 0}^{T-1} \left( \log\frac{\Theta_1(S_t,Z_t)}{\Theta_2(S_t,Z_t)} \right)^2 \right] \leq 8C_{\Theta^*} H^2(\Theta_1, \Theta_2). 
\#
By plugging \eqref{eq:049395} into \eqref{eq:04939335}, we conclude the proof of the lemma. 
\end{proof}

\subsection{Proof of Theorem \ref{thm:iv-param-theta}}\label{prf:thm:iv-param-theta}
\begin{proof}
\vskip5pt
\noindent\textbf{Proof of the first statement.} 
It suffices to show that with probability at least $1 - \delta$, we have  
\$
\hat L_1(\Theta^*) - \hat L_1(\hat \Theta) \leq \alpha_1. 
\$
By Corollary \ref{cor:vdg-param}, it holds with probability at least $1 - \delta$ that 
\#\label{eq:3824783}
H^2(\Theta^*, \hat \Theta) \leq c \cdot \frac{d}{NT \kappa} \log\frac{\theta_{\max}}{\delta}, 
\#
where $c>0$ is an absolute constant, which may vary from lines to lines. 
Thus, by Lemma \ref{lemma:iv-mle-1}, it holds with probability at least $1 - \delta$ that
\#\label{eq:qwe1f}
& \left| \left(L_1(\Theta^*) - L_1(\hat \Theta)\right) - \left(\hat L_1(\Theta^*) - \hat L_1(\hat \Theta)\right) \right|  \\
& \qquad \leq c\cdot \frac{\log C_{\Theta^*}}{NT \kappa} d \log \frac{\theta_{\max}}{\delta} \log(NT) + c\cdot \sqrt{\frac{C_{\Theta^*}}{NT \kappa} H^2(\Theta^*, \hat \Theta) d \log\frac{\theta_{\max}}{\delta} \log(NT) }  \\
& \qquad \leq c\cdot \frac{C_{\Theta^*}d}{NT \kappa} \log\frac{\theta_{\max}}{\delta} \log(NT), 
\#
where we use a covering argument and \eqref{eq:3824783} in the first and last inequalities, respectively. 
Further, by a similar idea as in \eqref{eq:iv-kl-hell}, we upper bound $|L_1(\Theta^*) - L_1(\hat \Theta)|$ as follows, 
\#\label{eq:qwe2f}
|L_1(\Theta^*) - L_1(\hat \Theta)| & = \left|\EE\left[ \frac{1}{T} \sum_{t = 0}^{T-1} \int \Theta^*(S_t,z) \log \frac{\Theta^*(S_t,z)}{\hat \Theta(S_t,z)}\ud z\right] \right| \\
& \leq 2C_{\Theta^*} H^2(\Theta^*, \hat \Theta) \leq c \cdot \frac{C_{\Theta^*}d}{NT \kappa} \log \frac{\theta_{\max}}{\delta}, 
\#
where we use \eqref{eq:3824783} and Corollary \ref{cor:vdg-param} in the first and last inequalities, respectively. Now, by combining \eqref{eq:qwe1f} and \eqref{eq:qwe2f}, it holds with probability at least $1 - \delta$ that 
\$
\hat L_1(\Theta^*) - \hat L_1(\hat \Theta) \leq c\cdot \frac{C_{\Theta^*}d}{NT \kappa} \log\frac{\theta_{\max}}{\delta} \log(NT) = \alpha_1, 
\$
which concludes the proof of the first statement. 

\vskip5pt
\noindent\textbf{Proof of the second statement.} 
By Lemma \ref{lemma:iv-mle-1}, with probability at least $1 - \delta$, for any $\Theta \in \ci^1_{\alpha_1}$, we have
\#\label{eq:qwww2f}
& \left| \left(L_1(\Theta^*) - L_1(\Theta)\right) - \left(\hat L_1(\Theta^*) - \hat L_1(\Theta)\right) \right|  \\
& \quad \leq c\cdot \frac{\log C_{\Theta^*}}{NT \kappa} d \log \frac{\theta_{\max}}{\delta} \log(NT) + c\cdot \sqrt{\frac{C_{\Theta^*}}{NT \kappa} H^2(\Theta^*, \Theta) d \log\frac{\theta_{\max}}{\delta} \log(NT) }, 
\#
where we use a covering argument. 
In the meanwhile, by the first statement, we have $\Theta^*\in \ci^1_{\alpha_1}$ with probability at least $1 - \delta$. Thus, we have
\#\label{eq:qwww3f}
\left| \hat L_1(\Theta^*) - \hat L_1(\Theta) \right| \leq \left| \hat L_1(\Theta^*) - \hat L_1(\hat \Theta) \right| + \left| \hat L_1(\hat \Theta) - \hat L_1(\Theta) \right| \leq 2\alpha_1, 
\#
where we use the fact that $\Theta \in \ci^1_{\alpha_1}$. By combining \eqref{eq:qwww2f} and \eqref{eq:qwww3f}, with probability at least $1- \delta$, it holds for any $\Theta \in \ci^1_{\alpha_1}$ that 
\#\label{eq:qwww4f}
L_1(\Theta) - L_1(\Theta^*) \leq & c\cdot \frac{\log C_{\Theta^*}}{NT \kappa} d \log \frac{\theta_{\max}}{\delta} \log(NT) \\
& + c\cdot \sqrt{\frac{C_{\Theta^*}}{NT \kappa} H^2(\Theta^*, \Theta) d \log\frac{\theta_{\max}}{\delta} \log(NT) }. 
\#
On the other hand, it holds for any $s\in \cS$ that
\$
-\int \Theta^*(s,z) \log\frac{\Theta(s,z)}{\Theta^*(s,z)}\ud z & \geq - 2 \int \Theta^*(s,z) \left(\sqrt{\frac{\Theta(s,z)}{\Theta^*(s,z)}}-1\right)\ud z \\
& = \int \left(\Theta^*(s,z) +  \Theta(s,z) - 2\sqrt{\Theta(s,z) \Theta^*(s,z)}\right)\ud z  \\
& = \int \left(\sqrt{\Theta^*(s,z)} + \sqrt{\Theta(s,z)} \right)^2\ud z,
\$
which implies that 
\#\label{eq:qwww5f}
L_1(\Theta) - L_1(\Theta^*) \geq 2H^2(\Theta^*, \Theta). 
\#
By combining \eqref{eq:qwww4f} and \eqref{eq:qwww5f}, we have 
\$
H^2(\Theta^*, \Theta) \leq c\cdot \frac{\log C_{\Theta^*}}{NT \kappa} d \log \frac{\theta_{\max}}{\delta} \log(NT) + c\cdot \sqrt{\frac{C_{\Theta^*}}{NT \kappa} H^2(\Theta^*, \Theta) d \log\frac{\theta_{\max}}{\delta} \log(NT) }, 
\$
which implies that 
\$
H^2(\Theta^*, \Theta) \leq c\cdot \frac{C_{\Theta^*}d}{NT \kappa} \log \frac{\theta_{\max}}{\delta} \log(NT). 
\$
Now, by Lemma \ref{lemma:vi-vdg-lemma}, with probability at least $1-\delta$, it holds for any $\Theta \in \ci^1_{\alpha_1}$ that 
\$
\sqrt{\EE\left[ \frac{1}{T} \sum_{t = 0}^{T-1} \|\Theta(S_t,\cdot) - \Theta^*(S_t,\cdot)\|_1^2 \right]} \leq c\cdot \sqrt{\frac{C_{\Theta^*}d}{NT \kappa} \log \frac{\theta_{\max}}{\delta} \log(NT) }, 
\$
which concludes the proof of the second statement. 
\end{proof}

\section{Proofs of Results in \S\ref{sec:iv-id}}
We provide proofs of results in \S\ref{sec:iv-id}. We first present proofs for \S\ref{sec:iv-mis-id}, then we present proofs for \S\ref{sec:iv-vf-id}. 

\subsection{Proof of Lemma \ref{lemma:iv-mis-j}}\label{prf:lemma:iv-mis-j}
\begin{proof}
We observe for any $t\in \{0, 1, \ldots, T-1\}$ that
\$
& \EE\left[ \frac{Z_t^\top A_t \pi(A_t\given S_t)}{\Delta^*(S_t,A_t) \Theta^*(S_t,Z_t)} w^\pi(S_t) R_t \Biggiven S_t \right] \\
& \quad = \sum_{a\in \cA} \EE\left[ \frac{Z_t^\top a \pi(a\given S_t) d^\pi(S_t) \ind\{A_t = a\} }{\Delta^*(S_t,a) \Theta^*(S_t,Z_t) d^b(S_t)} R(S_t, U_t, a, S_{t+1}, U_{t+1}) \Biggiven S_t \right] \\
& \quad = \sum_{a\in \cA} \EE\left[ \frac{Z_t^\top a \pi(a\given S_t) d^\pi(S_t) \ind\{A_t = a\} }{\Delta^*(S_t,a) \Theta^*(S_t,Z_t) d^b(S_t)} r(S_t, U_t, a) \Biggiven S_t \right] \\
& \quad =  \sum_{a\in \cA} \EE\left[ \frac{Z_t^\top a \pi(a\given S_t) d^\pi(S_t) \PP(A_t = a\given S_t, U_t, Z_t) }{\Delta^*(S_t, a) d^b(S_t) \Theta^*(S_t,Z_t)} r(S_t, U_t, a) \Biggiven S_t \right ] \\
& \quad = \sum_{a\in \cA}\EE\left[ \frac{\pi(a\given S_t) d^\pi(S_t) \PP(A_t = a\given S_t, U_t, Z_t=a) }{\Delta^*(S_t, a) d^b(S_t) } r(S_t, U_t, a) \Biggiven S_t \right ] \\
& \quad \quad - \sum_{a\in \cA}\sum_{z\in \cZ, z\neq a} \frac{1}{K-1}\EE\left[ \frac{\pi(a\given S_t) d^\pi(S_t) \PP(A_t = a\given S_t, U_t, Z_t=z) }{\Delta^*(S_t, a) d^b(S_t) } r(S_t, U_t, a) \Biggiven S_t \right ]  \\
& \quad = \sum_{a\in \cA} \frac{\pi(a\given S_t) d^\pi(S_t)}{d^b(S_t)} \EE_{U_t}[r(S_t, U_t, a)\given S_t], 
\$
where in the first equality, we use the definition of $w^\pi(s)$ and Assumption \ref{ass:10}; in the second equality, we use Assumption \ref{ass:iv-common}~\ref{ass:4}; in the third equality, we use Assumption \ref{ass:iv-common}~\ref{ass:5}; in the forth equality, we use Assumption \ref{ass:iv-common}~\ref{ass:iv-zu-ind}; while in the fifth equality, we use Assumption \ref{ass:iv-common}~\ref{ass:iv-compliance}. Now, by Assumption \ref{ass:10}, we know that $\EE_{U_t}[r(S_t, U_t, a)\given S_t] = \tilde r(S_t, a)$, where the function $\tilde r$ is independent of $t$. Therefore, we have 
\$
& \EE\left[ \frac{1}{T} \sum_{t = 0}^{T-1} \frac{Z_t^\top A_t \pi(A_t\given S_t)}{\Delta^*(S_t,A_t) \Theta^*(S_t,Z_t)} w^\pi(S_t) R_t \right]  \\
& \qquad = \frac{1}{T} \sum_{t = 0}^{T-1} \EE\left[ \sum_{a\in \cA} \frac{\pi(a\given S_t) d^\pi(S_t)}{d^b(S_t)} \tilde r(S_t, a) \right ]  \\
& \qquad = \frac{1}{T} \sum_{t = 0}^{T-1} \sum_{a\in \cA} \int \frac{\pi(a\given s) d^\pi(s)}{d^b(s)} \tilde r(s, a) p_t^b(s) \ud s   \\
& \qquad = \sum_{a\in \cA} \int \pi(a\given s) d^\pi(s) \tilde r(s, a) \ud s = J(\pi), 
\$
which concludes the proof of the lemma. 
\end{proof}

\subsection{Proof of Lemma \ref{lemma:iv-est-eq}} \label{prf:lemma:iv-est-eq}
\begin{proof}
Similar to the proof of Lemma \ref{lemma:iv-mis-j} in \S\ref{prf:lemma:iv-mis-j}, we observe that
\#\label{eq:iv-f11-vf}
& \EE\left[ \frac{1}{T} \sum_{t = 0}^{T-1} \frac{Z_t^\top A_t d^\pi(S_t) \pi(A_t\given S_t)}{\Delta(S_t,A_t) d^b(S_t) P(Z_t\given S_t)} f(S_t) \right ] = \int f(s') d^\pi(s') \ud s'. 
\#
Similarly, we have
\#\label{eq:iv-f12-vf}
& \EE\left[ \frac{1}{T} \sum_{t=0}^{T-1} \frac{Z_t^\top A_t d^\pi(S_t) \pi(A_t\given S_t)}{\Delta(S_t,A_t) d^b(S_t) P(Z_t\given S_t)} f(S_{t+1}) \right] \\
& \qquad = \int f(s') d^\pi(s) \pi(a\given s) \PP(S'=s\given S=s, A=a) \ud s'\ud a \ud s.
\#
Meanwhile, by the definition of $d^\pi(s,a)$, we have
\#\label{eq:iv-f13-vf}
d^\pi(s') & = (1 - \gamma) \sum_{t = 0}^\infty \gamma^t p_t^\pi(s') \\
& = (1-\gamma) \nu(s') + (1 - \gamma) \sum_{t = 0}^\infty \gamma^{t+1} p_{t+1}^\pi(s')  \\
& = (1-\gamma) \left ( \nu(s') + \gamma  \sum_{t = 0}^\infty  \gamma^t \int \PP(S_{t+1} = s' \given S_t = s, A_t = a) \pi(a\given s) p_t^\pi(s) \ud s \ud a \right )  \\
& = (1 - \gamma ) \nu(s') + \gamma \int \PP(S' = s' \given S = s, A = a) \pi(a\given s) d^\pi(s) \ud s \ud a, 
\#
where we use the assumption that $S_{t+1}\given (S_t, A_t)$ is time-homogeneous. 
Combining \eqref{eq:iv-f11-vf} and \eqref{eq:iv-f12-vf}, we have
\$
& \EE\left[ \frac{1}{T} \sum_{t=0}^{T-1} \frac{Z_t^\top A_t d^\pi(S_t) \pi(A_t\given S_t)}{\Delta(S_t,A_t) d^b(S_t) P(Z_t\given S_t)} \left ( f(S_t) -  \gamma  f(S_{t+1})  \right ) \right] \\
& \qquad = \int f(s') \left( d^\pi(s')  - \gamma \int d^\pi(s) \pi(a\given s) \PP(S'=s'\given S=s,A=a) \ud s \ud a \right)  \ud s'  \\
& \qquad = (1-\gamma) \int f(s') \nu(s') \ud s' = (1-\gamma) \EE_{S\sim \nu}\left[f(S)\right],
\$
where we use \eqref{eq:iv-f13-vf} in the forth equality. This concludes the proof of the lemma.
\end{proof}

\subsection{Proof of Lemma \ref{lemma:iv-vf-id}}
\label{prf:lemma:iv-vf-id}
\begin{proof}
Similar to the proof of Lemma \ref{lemma:iv-mis-j} in \S\ref{prf:lemma:iv-mis-j}, we observe that 
\#\label{eq:iv-feb71}
& \EE \left[ \frac{Z_0^\top A_0 \pi(A_0 \given S_0)}{\Delta^*(S_0, A_0) P(Z_0 \given S_0)  }  R_0 \Biggiven S_0 = s \right] \\
& \qquad = \EE_{U_0}\left [ \sum_{a\in \cA} \pi(a\given S_0) r(S_0, U_0, a)\Biggiven S_0 = s \right] = \EE_\pi[R_0 \given S_0 = s], 
\#
where $\EE_\pi[\cdot]$ denotes that the expectation is taken w.r.t. $A_0 \sim \pi(\cdot \given S_0)$. 
For notational convenience, we denote by 
\#\label{eq:def-rhot}
\rho_t = \frac{Z_t^\top A_t \pi(A_t \given S_t)}{\Delta^*(S_t, A_t) \Theta^*(Z_t \given S_t)}. 
\#
Similarly, by Assumption \ref{ass:10}, we observe that
\$
\EE \left[ \rho_0 \rho_1 R_1 \Biggiven S_0 = s \right] &  = \EE\left [ \sum_{a\in \cA} \pi(a\given S_0) \EE\left[ \rho_1  R_1 \Biggiven S_0, U_0, A_0 = a \right] \Biggiven S_0 = s \right]  \\
& = \EE\left [ \sum_{a\in \cA} \pi(a\given S_0) \EE\left[ \EE\left[ \rho_1  R_1 \Biggiven S_1, U_1 \right] \Biggiven S_0, U_0, A_0 = a \right] \Biggiven S_0 = s \right]  \\
& = \EE\left [ \sum_{a\in \cA} \pi(a\given S_0) \EE\left[ \EE_\pi[R_1 \given S_1, U_1] \given S_0, U_0, A_0 = a \right] \Biggiven S_0 = s \right] \\
& = \EE_\pi[R_1 \given S_0 = s]. 
\$
Now, by induction, it holds for any $t \geq 0$ that 
\$
\EE\left[ R_t \cdot \prod_{j = 0}^t \rho_j \Biggiven S_0 = s \right] = \EE_\pi[R_t \given S_0 = s], 
\$
which implies that 
\$
V^\pi(s) = \EE\left[ \sum_{t = 0}^\infty \gamma^t R_t \prod_{j = 0}^t \rho_j  \Biggiven S_0 = s \right]. 
\$

To show that the IV-aided Bellman equation holds, by a similar argument as in \eqref{eq:iv-feb71}, we observe that
\$
\EE \left[ \rho_0  \gamma V^\pi(S_1) \Biggiven S_0 = s \right] & = \EE \left[\rho_0 \cdot \EE\left[ \sum_{t = 0}^\infty \gamma^{t+1} R_{t+1} \prod_{j = 1}^{t+1} \rho_j \Biggiven S_1 \right] \Biggiven S_0 = s\right] \\
& = \EE \left[\rho_0 \EE\left[ \sum_{t = 0}^\infty \gamma^{t+1} R_{t+1} \left( \prod_{j = 1}^{t+1} \rho_j \right ) \Biggiven S_1, U_1 \right] \Biggiven S_0 = s\right] \\
& = \EE\left[ \sum_{t = 1}^\infty \gamma^{t} R_{t} \left( \prod_{j = 0}^{t} \rho_j \right ) \Biggiven S_0 = s\right],
\$
where the second equality relies on Assumption \ref{ass:10}.  By the definition of $\rho_t$ in \eqref{eq:def-rhot}, we have
\$
& \EE \left[ \frac{Z_0^\top A_0 \pi(A_0 \given S_0)}{\Delta^*(S_0, A_0) P(Z_0 \given S_0)  }  (R_0 + \gamma V^\pi(S_1) ) \Biggiven S_0 = s \right]  \\
& \qquad = \EE\left[ \sum_{t = 0}^\infty \gamma^{t+1} R_{t+1} \left( \prod_{j = 0}^{t+1} \frac{Z_j^\top A_j \pi(A_j \given S_j)}{\Delta^*(S_j, A_j) P(Z_j \given S_j)  } \right ) \Biggiven S_0 = s\right] = V^\pi(s).
\$
By Assumption \ref{ass:10} again, we can show that for any $k \geq 0$
\$
V^\pi(s) = \EE\left[ \sum_{t = k}^\infty \gamma^{t-k} R_t \left( \prod_{j = k}^t \frac{Z_j^\top A_j \pi(A_j \given S_j)}{\Delta^*(S_j, A_j) P(Z_j \given S_j)  } \right ) \bigggiven S_k = s \right],
\$
which concludes the proof of the lemma. 
\end{proof}

\subsection{Proof of Corollary \ref{cor:iv-vf-phi-0}}
\label{prf:cor:iv-vf-phi-0}
\begin{proof}
We have
\$
& \EE\left[ f(S_t) \frac{Z_t^\top A_t \pi(A_t\given S_t)}{\Delta^*(S_t,A_t) \Theta^*(S_t,Z_t)} \left(R_t + \gamma V^\pi(S_{t+1})\right) \right] \\
& \qquad = \EE\left[f(S_t) \EE\left[ \frac{Z_t^\top A_t \pi(A_t\given S_t)}{\Delta^*(S_t,A_t) \Theta^*(S_t,Z_t)} \left(R_t + \gamma V^\pi(S_{t+1})\right) \Biggiven S_t \right] \right] \\
& \qquad = \EE\left[f(S_t) V^\pi(S_t)\right],
\$
where the last equality comes from Lemma \ref{lemma:iv-vf-id}. By summing all $t\in \{0, 1, \ldots, T-1\}$, we conclude the proof of the corollary. 
\end{proof}

\section{Proofs of Results in \S\ref{sec:vf-theory}}

\subsection{Proof of Theorem \ref{thm:iv-vf}}
\label{prf:thm:iv-vf}
\begin{proof}
By the definition of $J(\pi)$ in \eqref{eq:iv-val-func}, we proceed as follows,
\#\label{eq:iv-vf-pp1}
& J(\pi^*) - J(\hat \pi_\vf)  \\
& \qquad = (1-\gamma) \EE_{S_0\sim \nu}\left [V^{\pi^*}(S_0) - V^{\hat \pi_\vf}(S_0)\right] \\
& \qquad \leq (1-\gamma) \EE_{S_0\sim \nu} \left [V^{\pi^*}(S_0) \right] - \min_{(\Delta, \Theta) \in \ci^0_{\alpha_0}\times \ci^1_{\alpha_1} } \min_{v \in \ci_{\alpha_\vf}^\vf(\Delta, \Theta, \hat \pi)} (1-\gamma) \EE_{S_0\sim \nu} \left [v(S_0) \right] \\
& \qquad \leq (1-\gamma) \EE_{S_0\sim \nu} \left [V^{\pi^*}(S_0) \right] -  \min_{ (\Delta, \Theta) \in \ci^0_{\alpha_0}\times \ci^1_{\alpha_1} } \min_{v \in \ci_{\alpha_\vf}^\vf(\Delta, \Theta, \pi^*)} (1-\gamma) \EE_{S_0\sim \nu} \left [v(S_0) \right], 
\#
where in the first inequality, we use Lemma \ref{lemma:iv-v-pi-in-conf}; while in the last inequality, we use the optimality of $\hat \pi_\vf$. It suffices to characterize the RHS of the above. 
We proceed \eqref{eq:iv-vf-pp1} as follows, 
\#\label{eq:iv-vf-pp3}
& J(\pi^*) - J(\hat \pi_\vf) \\
& \quad \leq \max_{ (\Delta, \Theta) \in \ci^0_{\alpha_0}\times \ci^1_{\alpha_1} } \max_{v \in \ci_{\alpha_\vf}^\vf(\Delta, \Theta, \pi^*)} \left | (1-\gamma) \EE_{S_0\sim \nu}[v(S_0)] - (1-\gamma) \EE_{S_0\sim \nu}\left[ V^{\pi^*}(S_0) \right]  \right |. 
\#
Meanwhile, by Lemmas \ref{lemma:iv-mis-j} and \ref{lemma:iv-est-eq}, we have
\#\label{eq:iv-vf-pp2}
& (1-\gamma) \EE_{S_0\sim \nu} \left [ V^{\pi^*}(S_0) \right] = J(\pi^*) = \EE\left[ \frac{1}{T} \sum_{t = 0}^{T-1}  w^{\pi^*}(S_t) \frac{Z_t^\top A_t \pi^*(A_t\given S_t)}{\Delta^*(S_t,A_t) \Theta^*(S_t,Z_t)} R_t \right], \\
& (1-\gamma) \EE_{S_0\sim \nu} \left [ v(S_0) \right] = \EE\left[ \frac{1}{T} \sum_{t = 0}^{T-1} w^{\pi^*}(S_t) \frac{Z_t^\top A_t \pi^*(A_t\given S_t)}{\Delta^*(S_t,A_t) \Theta^*(S_t,Z_t)} \left ( v(S_t) - \gamma  v(S_{t+1})  \right ) \right]. 
\#
Now, by plugging \eqref{eq:iv-vf-pp2} into the RHS of \eqref{eq:iv-vf-pp3}, we obtain
\$
& J(\pi^*) - J(\hat \pi_\vf) \leq \max_{ (\Delta, \Theta) \in \ci^0_{\alpha_0}\times \ci^1_{\alpha_1} } \max_{v \in \ci_{\alpha_\vf}^\vf(\Delta, \Theta, \pi^*)} \left | \Phi_\vf^{\pi^*}(v, w^{\pi^*}; \Delta^*, \Theta^*) \right |.
\$
By continuing the above computation, we have
\$
& J(\pi^*) - J(\hat \pi_\vf) \\
& \quad \leq \max_{ (\Delta, \Theta) \in \ci^0_{\alpha_0}\times \ci^1_{\alpha_1} } \max_{v \in \ci_{\alpha_\vf}^\vf(\Delta, \Theta, \pi^*)} \max_{g\in \cW} \left | \Phi_\vf^{\pi^*}(v, g; \Delta^*, \Theta^*) \right |  \\
& \quad =  \max_{ (\Delta, \Theta) \in \ci^0_{\alpha_0}\times \ci^1_{\alpha_1} } \max_{v \in \ci_{\alpha_\vf}^\vf(\Delta, \Theta, \pi^*)} \max_{g\in \cW} \max \left \{ \Phi_\vf^{\pi^*}(v, g; \Delta^*, \Theta^*), -\Phi_\vf^{\pi^*}(v, g; \Delta^*, \Theta^*) \right \}  \\
& \quad =  \max_{ (\Delta, \Theta) \in \ci^0_{\alpha_0}\times \ci^1_{\alpha_1} } \max_{v \in \ci_{\alpha_\vf}^\vf(\Delta, \Theta, \pi^*)} \max_{g\in \cW} \max \left \{ \Phi_\vf^{\pi^*}(v, g; \Delta^*, \Theta^*), \Phi_\vf^{\pi^*}(v, -g; \Delta^*, \Theta^*) \right \}  \\
& \quad =  \max_{ (\Delta, \Theta) \in \ci^0_{\alpha_0}\times \ci^1_{\alpha_1} } \max_{v \in \ci_{\alpha_\vf}^\vf(\Delta, \Theta, \pi^*)} \max_{g\in \cW} \Phi_\vf^{\pi^*}(v, g; \Delta^*, \Theta^*)  \\
& \quad \leq c\cdot \frac{ C_{\Delta^*}^2 C_{\Theta^*}^2 C_*}{1-\gamma} (\xi_0 + \xi_1) L_\Pi \sqrt{\frac{1}{NT \kappa}\cdot \pdim_{\cF_0,\cF_1,\cW,\cV,\Pi} \cdot \log\frac{1}{\delta} \log(NT)}, 
\$
where in the first inequality, we use Assumption \ref{ass:iv-vf-realizable}; in the third equality, we use the fact that $\cW$ is symmetric; while in the last inequality, we use Lemma \ref{lemma:iv-v-in-conf-good}. This concludes the proof of the theorem. 
\end{proof}

\subsection{Proof of Lemma \ref{lemma:iv-v-pi-in-conf}}
\label{prf:lemma:iv-v-pi-in-conf}
\begin{proof}
By Assumption \ref{ass:iv-vf-realizable}, we know that $V^\pi\in \cV$. Thus, to show that $V^\pi \in \ci^\vf_{\alpha_\vf}(\Delta^*, \Theta^*, \pi)$ with a high probability, it suffices to show that 
\#\label{eq:iv-vpi-in-ci-wtp}
\max_{g\in \cW} \hat \Phi_\vf^\pi(V^\pi, g; \Delta^*, \Theta^*) - \max_{g\in \cW} \hat \Phi_\vf^\pi(\hat v_{\Delta^*, \Theta^*}^\pi, g; \Delta^*, \Theta^*) \leq \alpha_\vf. 
\#
In the follows, we show that \eqref{eq:iv-vpi-in-ci-wtp} holds with a high probability. 
For the simplicity of notations, we denote by $\Phi_\vf^\pi(v, g; *) = \Phi_\vf^\pi(v, g; \Delta^*, \Theta^*)$ and $\hat v^\pi_* = \hat v^\pi_{\Delta^*, \Theta^*}$ for any $(\pi,v,g)$.
Note that 
\#\label{eq:iv-ee1-vf}
& \max_{g \in \cW} \hat \Phi_\vf^\pi(V^\pi, g; *) - \max_{g \in \cW} \hat \Phi_\vf^\pi(\hat v^\pi_*, g; *)  \\
& \qquad = \max_{g \in \cW} \hat \Phi_\vf^\pi(V^\pi, g; *) - \max_{g \in \cW}  \Phi_\vf^\pi(V^\pi, g; *)  + \max_{g \in \cW}  \Phi_\vf^\pi(V^\pi, g; *) - \max_{g \in \cW}  \Phi_\vf^\pi(\hat v^\pi_*, g; *)  \\
& \qquad \qquad + \max_{g \in \cW} \Phi_\vf^\pi(\hat v^\pi_*, g; *) - \max_{g \in \cW} \hat \Phi_\vf^\pi(\hat v^\pi_*, g; *)  \\
& \qquad \leq \max_{g \in \cW} \hat \Phi_\vf^\pi(V^\pi, g; *) - \max_{g \in \cW}  \Phi_\vf^\pi(V^\pi, g; *) + \max_{g \in \cW} \Phi_\vf^\pi(\hat v^\pi_*, g; *) - \max_{g \in \cW} \hat \Phi_\vf^\pi(\hat v^\pi_*, g; *)  \\
& \qquad \leq 2 \max_{v\in \cV} \left | \max_{g \in \cW} \hat \Phi_\vf^\pi(v, g; *) - \max_{g \in \cW}  \Phi_\vf^\pi(v, g; *) \right |  \\
& \qquad \leq 2 \max_{v\in \cV} \max_{g \in \cW} \left |  \hat \Phi_\vf^\pi(v, g; *) - \Phi_\vf^\pi(v, g; *) \right |,
\#
where in the first inequality, we use the fact that $\max_{g \in \cW}  \Phi_\vf^\pi(V^\pi, g; *) = 0$ while $\max_{g \in \cW}  \Phi_\vf^\pi(v, g; *) \geq 0$ for any $v$. In the meanwhile, by Theorem \ref{thm:hoeffding-mixing}, with probability at least $1 - \delta$, it holds for any $(\pi,v,g)\in \Pi \times \cV \times \cW$ that
\#\label{eq:iv-ee2-vf}
\left |  \hat \Phi_\vf^\pi(v, g; *) - \Phi_\vf^\pi(v, g; *) \right | \leq c\cdot \frac{C_{\Delta^*} C_{\Theta^*} C_*}{1-\gamma} \sqrt{\frac{\pdim_{\cW,\cV,\Pi}}{NT\kappa} \cdot \log\frac{1}{\delta} \log(NT) }, 
\#
where we use Assumption \ref{ass:upper-bound-delta} and $\|g\|_\infty \leq C_*$ for any $g \in \cW$. 
Now, combining \eqref{eq:iv-ee1-vf} and \eqref{eq:iv-ee2-vf}, with probability at least $1 - \delta$, we have
\$
& \max_{g \in \cW} \hat \Phi_\vf^\pi(V^\pi, g; *) - \max_{g \in \cW} \hat \Phi_\vf^\pi(\hat v^\pi_*, g; *) \leq c\cdot \frac{C_{\Delta^*} C_{\Theta^*} C_*}{1-\gamma} \sqrt{\frac{\pdim_{\cW,\cV,\Pi}}{NT\kappa} \cdot \log\frac{1}{\delta} \log(NT) } = \alpha_\vf, 
\$
which implies that $V^\pi \in \ci^\vf_{\alpha_\vf}(\Delta^*, \Theta^*, \pi)$ for any $\pi\in \Pi$. This concludes the proof of the lemma. 
\end{proof}

\subsection{Proof of Lemma \ref{lemma:iv-v-in-conf-good}}
\label{prf:lemma:iv-v-in-conf-good}
\begin{proof}
Since $v \in \cup_{(\Delta, \Theta) \in \ci^0_{\alpha_0} \times \ci^1_{\alpha_1}} \ci^\vf_{\alpha_\vf}(\Delta, \Theta, \pi)$, there exists a pair $(\tilde \Delta, \tilde\Theta) \in \ci^0_{\alpha_0} \times \ci^1_{\alpha_1}$ such that $v \in \ci^\vf_{\alpha_\vf}(\tilde \Delta, \tilde \Theta, \pi)$. For the simplicity of notations, we denote by 
\#\label{eq:iv-dd0-vf}
\tilde v \in \argmin_{v\in \cV} \max_{g\in \cW} \hat \Phi^\pi_\vf(v, g; \tilde \Delta, \tilde \Theta),
\#
i.e., $\tilde v = \hat v^\pi_{\tilde \Delta, \tilde \Theta}$, which is defined in \eqref{eq:v-pi-def}. By the definition of $\tilde v$ and $v \in \ci^\vf_{\alpha_\vf}(\tilde \Delta, \tilde \Theta, \pi)$, we know that 
\#\label{eq:iv-dd1-vf}
\max_{g\in \cW} \hat \Phi^\pi_\vf(v,g;\tilde \Delta, \tilde \Theta) - \max_{g\in \cW} \hat \Phi^\pi_\vf(\tilde v,g;\tilde \Delta, \tilde \Theta) \leq \alpha_\vf. 
\#
Note that 
\#\label{eq:iv-dd2-vf}
& \max_{g\in \cW} \Phi^\pi_\vf(v,g;\Delta^*, \Theta^*)  \\
& \qquad = \max_{g\in \cW} \Phi^\pi_\vf(v,g;\Delta^*, \Theta^*) - \max_{g\in \cW} \hat \Phi^\pi_\vf(v,g;\Delta^*, \Theta^*) + \max_{g\in \cW} \hat \Phi^\pi_\vf(v,g;\Delta^*, \Theta^*)  \\
& \qquad \qquad - \max_{g\in \cW} \hat \Phi^\pi_\vf(v,g;\tilde \Delta, \tilde \Theta) + \max_{g\in \cW} \hat \Phi^\pi_\vf(v,g;\tilde \Delta, \tilde \Theta) - \max_{g\in \cW} \hat \Phi^\pi_\vf(\tilde v,g;\tilde \Delta, \tilde \Theta)  \\
& \qquad \qquad + \max_{g\in \cW} \hat \Phi^\pi_\vf(\tilde v,g;\tilde \Delta, \tilde \Theta) - \max_{g\in \cW} \Phi^\pi_\vf(\tilde v,g;\tilde \Delta, \tilde \Theta) + \max_{g\in \cW} \Phi^\pi_\vf(\tilde v,g;\tilde \Delta, \tilde \Theta)  \\
& \qquad \leq 2 \underbrace{\max_{(v,g,\Delta,\Theta)\in (\cV, \cW, \cF_0, \cF_1)} \left | \Phi^\pi_\vf(v,g;\Delta, \Theta) - \hat \Phi^\pi_\vf(v,g;\Delta, \Theta) \right|}_{\text{Term (I)}} + \underbrace{\max_{g\in \cW} \Phi^\pi_\vf(\tilde v,g;\tilde \Delta, \tilde \Theta)}_{\text{Term (II)}}  \\
& \qquad \qquad + \underbrace{\max_{g\in \cW} \left | \hat \Phi^\pi_\vf(v,g;\Delta^*, \Theta^*) - \hat \Phi^\pi_\vf(v,g;\tilde \Delta, \tilde \Theta) \right|}_{\text{Term (III)}} + \alpha_\vf,
\#
where we use \eqref{eq:iv-dd1-vf} in the last inequality. Now we upper bound terms (I), (II), and (III) on the RHS of \eqref{eq:iv-dd2-vf}. 

\vskip5pt
\noindent\textbf{Upper Bounding Term (I).}
By Theorem \ref{thm:hoeffding-mixing}, with probability at least $1 - \delta$, it holds for any $(v,g,\Delta,\Theta, \pi) \in (\cV, \cW, \cF_0, \cF_1, \Pi)$ that
\$
\left |  \hat \Phi_\vf^\pi(v, g; \Delta, \Theta) - \Phi_\vf^\pi(v, g; \Delta, \Theta) \right | \leq c\cdot \frac{C_{\Delta^*} C_{\Theta^*} C_*}{1-\gamma} \sqrt{\frac{\pdim_{\cF_0,\cF_1,\cW,\cV, \Pi}}{NT \kappa}\log\frac{1}{\delta} \log(NT)},
\$
which implies that with probability at least $1 - \delta$, we have
\#\label{eq:iv-dd3-vf}
\text{Term (I)} \leq c\cdot \frac{C_{\Delta^*} C_{\Theta^*} C_*}{1-\gamma} \sqrt{\frac{\pdim_{\cF_0,\cF_1,\cW,\cV,\Pi}}{NT \kappa}\log\frac{1}{\delta}\log(NT) }.
\#

\vskip5pt
\noindent\textbf{Upper Bounding Term (II).} 
We introduce the following lemma to help upper bound term (II).

\begin{lemma}
\label{lemma:iv-any-v-hat-small-loss}
Suppose $\alpha_0$ and $\alpha_1$ are defined in Assumption \ref{ass:iv-sl-res}. 
With probability at least $1 - \delta$, for any $(\Delta, \Theta, \pi) \in \ci^0_{\alpha_0} \times \ci^1_{\alpha_1} \times \Pi$, we have
\$
\max_{g\in \cW} \Phi_\vf^\pi(\hat v^\pi_{\Delta, \Theta}, g; \Delta, \Theta) \leq c\cdot \frac{C_{\Delta^*}^2 C_{\Theta^*}^2 C_*}{1-\gamma} (\xi_0 + \xi_1) L_\Pi \sqrt{\frac{\pdim_{\cF_0,\cF_1,\cW,\cV,\Pi}}{NT \kappa}\log\frac{1}{\delta}\log(NT) },
\$
where $\hat v^\pi_{\Delta, \Theta}$ is defined in \eqref{eq:v-pi-def}, $\xi_0$ and $\xi_1$ are constants defined in Assumption \ref{ass:iv-sl-res}. 
\end{lemma}
\begin{proof}
See \S\ref{prf:lemma:iv-any-v-hat-small-loss} for a detailed proof. 
\end{proof}

By the definition of $\tilde v$ in \eqref{eq:iv-dd0-vf} and Lemma \ref{lemma:iv-any-v-hat-small-loss}, with probability at least $1 - \delta$, we have
\#\label{eq:iv-dd4-vf}
\text{Term (II)} \leq  c\cdot \frac{C_{\Delta^*}^2 C_{\Theta^*}^2 C_*}{1-\gamma} (\xi_0 + \xi_1) L_\Pi \sqrt{\frac{\pdim_{\cF_0,\cF_1,\cW,\cV,\Pi}}{NT \kappa}\log\frac{1}{\delta} \log(NT)}.
\#

\vskip5pt
\noindent\textbf{Upper Bounding Term (III).} 
Note that 
\#\label{eq:iv-dd5-vf}
& \left | \hat \Phi^\pi_\vf(v,g;\Delta^*, \Theta^*) - \hat \Phi^\pi_\vf(v,g;\tilde \Delta, \tilde \Theta) \right|  \\
& \leq \left | \left(\hat \EE - \EE\right) \left[ \frac{1}{T} \sum_{t = 0}^{T-1} g(S_t) \left ( \frac{Z_t^\top A_t\pi(A_t\given S_t) }{\Delta^*(S_t,A_t)\Theta^*(S_t,Z_t)} - \frac{Z_t^\top A_t\pi(A_t\given S_t)}{\tilde \Delta(S_t,A_t) \tilde \Theta(S_t,Z_t)} \right) \left( R_t + \gamma v(S_{t+1}) \right) \right ] \right |  \\
& \qquad + \left | \EE \left[ \frac{1}{T} \sum_{t = 0}^{T-1} g(S_t) \left ( \frac{Z_t^\top A_t\pi(A_t\given S_t) }{\Delta^*(S_t,A_t)\Theta^*(S_t,Z_t)} - \frac{Z_t^\top A_t\pi(A_t\given S_t)}{\tilde \Delta(S_t,A_t) \tilde \Theta(S_t,Z_t)} \right) \left( R_t + \gamma v(S_{t+1}) \right) \right ] \right |. 
\#
For the first term on the RHS of \eqref{eq:iv-dd5-vf}, by Theorem \ref{thm:hoeffding-mixing}, with probability at least $1- \delta$, it holds for any $(v,g,\pi)\in \cV\times \cW \times \Pi$ that
\#\label{eq:iv-dd5-vf-1}
& \left | \left(\hat \EE - \EE\right) \left[ \frac{1}{T} \sum_{t = 0}^{T-1} g(S_t) \left ( \frac{Z_t^\top A_t\pi(A_t\given S_t) }{\Delta^*(S_t,A_t)\Theta^*(S_t,Z_t)} - \frac{Z_t^\top A_t\pi(A_t\given S_t)}{\tilde \Delta(S_t,A_t) \tilde \Theta(S_t,Z_t)} \right) \left( R_t + \gamma v(S_{t+1}) \right) \right ] \right |  \\
& \qquad \leq c\cdot \frac{C_{\Delta^*} C_{\Theta^*} C_*}{1-\gamma} \sqrt{\frac{\pdim_{\cF_0,\cF_1,\cW,\cV,\Pi}}{NT \kappa}\log\frac{1}{\delta} \log(NT)}. 
\#
For the second term on the RHS of \eqref{eq:iv-dd5-vf}, with probability at least $1- \delta$, it holds that 
\#\label{eq:iv-dd5-vf-2}
& \left | \EE \left[ \frac{1}{T} \sum_{t = 0}^{T-1} g(S_t) \left ( \frac{Z_t^\top A_t\pi(A_t\given S_t) }{\Delta^*(S_t,A_t)\Theta^*(S_t,Z_t)} - \frac{Z_t^\top A_t\pi(A_t\given S_t)}{\tilde \Delta(S_t,A_t) \tilde \Theta(S_t,Z_t)} \right) \left( R_t + \gamma v(S_{t+1}) \right) \right ] \right |  \\
& \qquad \leq \frac{C_*}{1-\gamma} \EE\left[ \frac{1}{T} \sum_{t = 0}^{T-1} \left | \frac{1}{\tilde \Delta(S_t,A_t) \tilde \Theta(S_t,Z_t)} - \frac{1}{\Delta^*(S_t,A_t) \Theta^*(S_t,Z_t)} \right| \right ]  \\
& \qquad = \frac{C_*}{1-\gamma} \EE\left[ \frac{1}{T} \sum_{t = 0}^{T-1} \left | \frac{\Theta^*(S_t,Z_t) - \tilde \Theta(S_t,Z_t)}{\tilde \Delta(S_t,A_t)\tilde \Theta(S_t,Z_t)\Theta^*(S_t,Z_t)} - \frac{\Delta^*(S_t,A_t) - \tilde \Delta(S_t,A_t)}{\Delta^*(S_t,A_t)\tilde \Delta(S_t,A_t) \Theta^*(S_t,Z_t)} \right| \right ]  \\
& \qquad \leq \frac{C_{\Delta^*} C_{\Theta^*} C_*}{1-\gamma} \Bigg( C_{\Theta^*} \EE\left[ \frac{1}{T} \sum_{t = 0}^{T-1} \left \| \Delta^*(S_t,\cdot) - \tilde \Delta(S_t,\cdot) \right \|_1  \right ]  \\
& \qquad \qquad \qquad \qquad \qquad + C_{\Delta^*} \EE\left[ \frac{1}{T} \sum_{t = 0}^{T-1} \left \| \Theta^*(S_t,\cdot ) - \tilde \Theta(S_t,\cdot )\right\|_1  \right ] \Bigg)  \\
& \qquad \leq \frac{C_{\Delta^*} C_{\Theta^*} C_*}{1-\gamma} \Bigg( C_{\Theta^*} \sqrt{\EE\left[ \frac{1}{T} \sum_{t = 0}^{T-1} \left \| \Delta^*(S_t,\cdot) - \tilde \Delta(S_t,\cdot) \right \|_1^2  \right ]} \\
& \qquad \qquad \qquad \qquad \qquad + C_{\Delta^*} \sqrt{\EE\left[ \frac{1}{T} \sum_{t = 0}^{T-1} \left \| \Theta^*(S_t,\cdot ) - \tilde \Theta(S_t,\cdot )\right\|_1^2  \right ]} \Bigg)  \\
& \qquad \leq \frac{C_{\Delta^*} C_{\Theta^*} C_*}{1-\gamma} \Bigg(\xi_0 C_{\Theta^*} \sqrt{\frac{C_{\Delta^*}}{NT \kappa} \pdim_{\cF_0}\log\frac{1}{\delta} \log(NT)}  + \xi_1 C_{\Delta^*} \sqrt{\frac{C_{\Theta^*}}{NT \kappa} \pdim_{\cF_1} \log\frac{1}{\delta} \log(NT)} \Bigg),
\#
where in the first inequality, we use the fact that $\|v\|_\infty \leq 1/ (1 - \gamma)$ and $\|g\|_\infty \leq C_*$; in the third inequality, we use Cauchy-Schwarz inequality; while in the last inequality, we use Assumption \ref{ass:iv-sl-res} with the fact that $(\tilde \Delta, \tilde\Theta) \in \ci^0_{\alpha_0} \times \ci^1_{\alpha_1}$. Now, by plugging \eqref{eq:iv-dd5-vf-1} and \eqref{eq:iv-dd5-vf-2} into \eqref{eq:iv-dd5-vf}, with probability at least $1 - \delta$,  it holds for any $v \in \cup_{(\Delta, \Theta) \in \ci^0_{\alpha_0} \times \ci^1_{\alpha_1}} \ci^\vf_{\alpha_\vf}(\Delta, \Theta, \pi)$, $g\in \cW$, and $\pi \in \Pi$ that 
\#\label{eq:879434}
& \left | \hat \Phi^\pi_\vf(v,g;\Delta^*, \Theta^*) - \hat \Phi^\pi_\vf(v,g;\tilde \Delta, \tilde \Theta) \right|  \\
& \qquad \leq c\cdot \frac{C_{\Delta^*}^2 C_{\Theta^*}^2 C_*}{1-\gamma} (\xi_0 + \xi_1) \sqrt{\frac{1}{NT\kappa}\cdot \pdim_{\cF_0,\cF_1,\cW,\cV,\Pi} \cdot \log\frac{1}{\delta} \log(NT)}. 
\#

\vskip5pt
Now, by plugging \eqref{eq:iv-dd3-vf}, \eqref{eq:iv-dd4-vf}, and \eqref{eq:879434} into 
\eqref{eq:iv-dd2-vf}, with probability at least $1 - \delta$, it holds for any $v \in \cup_{(\Delta, \Theta) \in \ci^0_{\alpha_0} \times \ci^1_{\alpha_1}} \ci^\vf_{\alpha_\vf}(\Delta, \Theta, \pi)$ and $\pi \in \Pi$ that
\$
& \max_{g\in \cW} \Phi^\pi_\vf(v,g;\Delta^*, \Theta^*) \leq c\cdot \frac{C_{\Delta^*}^2 C_{\Theta^*}^2 C_*}{1-\gamma} (\xi_0 + \xi_1) L_\Pi \sqrt{\frac{1}{NT\kappa}\cdot \pdim_{\cF_0,\cF_1,\cW,\cV,\Pi} \cdot \log\frac{1}{\delta} \log(NT)},
\$
which concludes the proof of the lemma.
\end{proof}

\subsection{Proof of Lemma \ref{lemma:iv-any-v-hat-small-loss}}
\label{prf:lemma:iv-any-v-hat-small-loss}
\begin{proof}
Note that 
\#\label{eq:iv-ee3-vf}
& \max_{g\in \cW} \Phi_\vf^\pi(\hat v^\pi_{\Delta, \Theta}, g; \Delta, \Theta)  \\
& \qquad = \max_{g\in \cW} \Phi_\vf^\pi(\hat v^\pi_{\Delta, \Theta}, g; \Delta, \Theta) - \max_{g\in \cW} \hat \Phi_\vf^\pi(\hat v^\pi_{\Delta, \Theta}, g; \Delta, \Theta) + \max_{g\in \cW} \hat \Phi_\vf^\pi(\hat v^\pi_{\Delta, \Theta}, g; \Delta, \Theta)  \\
& \qquad \qquad - \max_{g\in \cW} \hat \Phi_\vf^\pi( V^\pi, g; \Delta, \Theta)  + \max_{g\in \cW} \hat \Phi_\vf^\pi( V^\pi, g; \Delta, \Theta) - \max_{g\in \cW}  \Phi_\vf^\pi(V^\pi, g; \Delta, \Theta)  \\
& \qquad \qquad + \max_{g\in \cW}  \Phi_\vf^\pi(V^\pi, g; \Delta, \Theta) - \max_{g\in \cW}  \Phi_\vf^\pi(V^\pi, g; \Delta^*, \Theta^*)  \\
& \qquad \leq 2 \max_{v\in \cV} \max_{g\in \cW} \left |  \Phi_\vf^\pi(v, g; \Delta, \Theta) -  \hat \Phi_\vf^\pi(v, g; \Delta, \Theta) \right |  \\
& \qquad \qquad + \max_{g\in \cW} \left |  \Phi_\vf^\pi(V^\pi, g; \Delta, \Theta) -   \Phi_\vf^\pi(V^\pi, g; \Delta^*, \Theta^*) \right |,
\#
where we use the fact that $\hat v^\pi_{\Delta, \Theta} \in \argmin_{v\in \cV} \max_{g\in \cW} \hat \Phi^\pi_\vf(v, g; \Delta, \Theta)$ in the last inequality. In the meanwhile, by Theorem \ref{thm:hoeffding-mixing}, with probability at least $1 - \delta$, it holds for any $(v,g,\pi)\in \cV\times \cW \times \Pi$ that
\#\label{eq:iv-ee4-vf}
\left |  \hat \Phi_\vf^\pi(v, g; \Delta, \Theta) - \Phi_\vf^\pi(v, g; \Delta, \Theta) \right | \leq c\cdot \frac{C_{\Delta^*} C_{\Theta^*} C_*}{1-\gamma} \sqrt{\frac{\pdim_{\cV,\cW,\Pi}}{NT \kappa}\log\frac{1}{\delta} \log(NT)}.
\#
Also, we upper bound the second term on the RHS of \eqref{eq:iv-ee3-vf} with probability at least $1- \delta$ by a similar argument as in \eqref{eq:iv-dd5-vf-2}, 
\#\label{eq:iv-ee5-vf}
& \left |  \Phi_\vf^\pi(V^\pi, g; \Delta, \Theta) -   \Phi_\vf^\pi(V^\pi, g; \Delta^*, \Theta^*) \right |  \\
& \leq \frac{C_{\Delta^*} C_{\Theta^*} C_*}{1-\gamma} \left(\xi_0 C_{\Theta^*} \sqrt{\frac{C_{\Delta^*}}{NT \kappa} \pdim_{\cF_0}\log\frac{1}{\delta} \log(NT)} + \xi_1 C_{\Delta^*} \sqrt{\frac{C_{\Theta^*}}{NT \kappa} \pdim_{\cF_1} \log\frac{1}{\delta} \log(NT)} \right),
\#
where in the first inequality, we use the fact that $\|V^\pi\|_\infty \leq 1/ (1 - \gamma)$ and $\|g\|_\infty \leq C_*$; in the third inequality, we use Cauchy Schwarz inequality; while in the last inequality, we use Assumption \ref{ass:iv-sl-res} with $(\Delta, \Theta) \in \ci^0_{\alpha_0} \times \ci^1_{\alpha_1}$. 
Now, by plugging \eqref{eq:iv-ee4-vf} and \eqref{eq:iv-ee5-vf} into \eqref{eq:iv-ee3-vf}, with probability at least $1 - \delta$, it holds for any $(\Delta, \Theta, \pi) \in \ci^0_{\alpha_0} \times \ci^1_{\alpha_1} \times \Pi$ that 
\$
\max_{g\in \cW} \Phi_\vf^\pi(\hat v^\pi_{\Delta, \Theta}, g; \Delta, \Theta) \leq c\cdot \frac{C_{\Delta^*}^2 C_{\Theta^*}^2 C_*}{1-\gamma} (\xi_0 + \xi_1) L_\Pi \sqrt{\frac{\pdim_{\cF_0,\cF_1,\cW,\cV,\Pi}}{NT \kappa}\log\frac{1}{\delta} \log(NT)},
\$
which concludes the proof of the lemma.
\end{proof}

\section{Proofs of Results in \S\ref{sec:mis-theory}}
\subsection{Proof of Theorem \ref{thm:iv-mis}}
\label{prf:thm:iv-mis}
\begin{proof}
Before the proof of the theorem, we first introduce some supporting results as follows. We define the population counterpart of $\hat L_\mis(w, \pi; \Delta, \Theta)$ as  
\$
L_\mis(w, \pi; \Delta, \Theta) = \EE \left[ \frac{1}{T} \sum_{t = 0}^{T-1} \frac{Z_t^\top A_t \pi(A_t\given S_t)}{\Delta(S_t,A_t) \Theta(S_t,Z_t)} w(S_t) R_t \right]
\$ 
for any $(w, \pi, \Delta, \Theta)$.

\begin{lemma}
\label{lemma:iv-link-phi-l}
It holds for any $(\pi, w)\in \Pi \times \cW$ that
\$
L_\mis(w^\pi, \pi; \Delta^*, \Theta^*) - L_\mis(w, \pi; \Delta^*, \Theta^*) = \Phi_\mis^\pi(w, V^\pi; \Delta^*, \Theta^*),
\$
where $V^\pi$ is the state-value function defined in \eqref{eq:iv-val-func}. 
\end{lemma}
\begin{proof}
See \S\ref{prf:lemma:iv-link-phi-l} for a detailed proof. 
\end{proof}

\begin{lemma}
\label{lemma:iv-mis-min-delta-close}
Suppose that $(\alpha_0, \alpha_1, \alpha_\mis)$ is defined in Lemmas \ref{lemma:iv-w-pi-in-conf}. With probability at least $1 - \delta$, it holds for any $(\Delta, \Theta) \in \ci^0_{\alpha_0} \times \ci^1_{\alpha_1}$ that 
\$
& \left | \min_{w\in \ci^\mis_{\alpha_\mis}(\Delta^*, \Theta^*, \pi^*)} L_\mis(w, \pi^*; \Delta^*, \Theta^*) - \min_{w\in \ci^\mis_{\alpha_\mis}(\Delta, \Theta, \pi^*)} L_\mis(w, \pi^*; \Delta, \Theta) \right |  \\
& \qquad \leq c\cdot \frac{ C_{\Delta^*}^2 C_{\Theta^*}^2 C_*}{1-\gamma} (\xi_0  + \xi_1) \sqrt{\frac{1}{NT \kappa} \pdim_{\cF_0,\cF_1,\cW,\cV} \log\frac{1}{\delta} \log(NT)} = \varepsilon^*_L. 
\$
\end{lemma}
\begin{proof}
See \S\ref{prf:lemma:iv-mis-min-delta-close} for a detailed proof. 
\end{proof}

\begin{lemma}
\label{lemma:iv-mis-hatl-close}
With probability at least $1 - \delta$, it holds for any $(w, \Delta, \Theta, \pi)\in \cW \times \cF_0 \times \cF_1 \times \Pi$ that 
\$
& \left | L_\mis(w, \pi; \Delta, \Theta) - \hat L_\mis(w, \pi; \Delta, \Theta) \right | \\
& \qquad \leq c\cdot C_{\Delta^*} C_{\Theta^*} C_* \sqrt{\frac{1}{NT \kappa} \pdim_{\cF_0,\cF_1,\cW,\Pi} \log\frac{1}{\delta} \log(NT)} = \hat \varepsilon_L. 
\$
\end{lemma}
\begin{proof}
See \S\ref{prf:lemma:iv-mis-hatl-close} for a detailed proof. 
\end{proof}

Now we start the proof of the theorem. By the definition of $J(\pi)$, it holds with probability at least $1 - \delta$ that
\#\label{eq:thm-mis-f11}
J(\pi^*) - J(\hat \pi_\mis) & = L_\mis(w^{\pi^*}, \pi^*; \Delta^*, \Theta^*) - L_\mis(w^{\hat \pi_\mis}, \hat \pi_\mis; \Delta^*, \Theta^*) \\
& \leq L_\mis(w^{\pi^*}, \pi^*; \Delta^*, \Theta^*) - \min_{w\in \ci^\mis_{\alpha_\mis}(\Delta^*, \Theta^*, \hat \pi_\mis)} L_\mis(w, \hat \pi_\mis; \Delta^*, \Theta^*)  \\
& \leq L_\mis(w^{\pi^*}, \pi^*; \Delta^*, \Theta^*)  \\
& \qquad - \min_{(\Delta, \Theta) \in \ci^0_{\alpha_0}\times \ci^1_{\alpha_1}} \min_{w\in \ci^\mis_{\alpha_\mis}(\Delta, \Theta, \hat \pi_\mis)} L_\mis(w, \hat \pi_\mis; \Delta, \Theta)  \\
& \leq L_\mis(w^{\pi^*}, \pi^*; \Delta^*, \Theta^*)  \\
& \qquad - \min_{(\Delta, \Theta) \in \ci^0_{\alpha_0}\times \ci^1_{\alpha_1}} \min_{w\in \ci^\mis_{\alpha_\mis}(\Delta, \Theta, \hat \pi_\mis)} \hat L_\mis(w, \hat \pi_\mis; \Delta, \Theta) + \hat \varepsilon_L  \\
& \leq L_\mis(w^{\pi^*}, \pi^*; \Delta^*, \Theta^*)  \\
& \qquad - \min_{(\Delta, \Theta) \in \ci^0_{\alpha_0}\times \ci^1_{\alpha_1}} \min_{w\in \ci^\mis_{\alpha_\mis}(\Delta, \Theta, \pi^*)} \hat L_\mis(w, \pi^*; \Delta, \Theta) + \hat \varepsilon_L,
\#
where we use Lemma \ref{lemma:iv-w-pi-in-conf} in the first inequality; we use Assumption \ref{ass:iv-sl-res} that $(\Delta^*, \Theta^*) \in \ci^0_{\alpha_0}\times \ci^1_{\alpha_1}$ with probability at least $1 - \delta$ in the second inequality; we use Lemma \ref{lemma:iv-mis-hatl-close} in the third inequality; while we use the optimality of $\hat \pi_\mis$ in the last inequality. Now, by applying Lemmas \ref{lemma:iv-mis-min-delta-close} and \ref{lemma:iv-mis-hatl-close}, we obtain from \eqref{eq:thm-mis-f11} that 
\#\label{eq:thm-mis-f12}
J(\pi^*) - J(\hat \pi_\mis) & \leq L_\mis(w^{\pi^*}, \pi^*; \Delta^*, \Theta^*) - \min_{w\in \ci^\mis_{\alpha_\mis}(\Delta^*, \Theta^*, \pi^*)} L_\mis(w, \pi^*; \Delta^*, \Theta^*) + \varepsilon_L^* + 2\hat \varepsilon_L  \\
& \leq \max_{w\in \ci^\mis_{\alpha_\mis}(\Delta^*, \Theta^*, \pi^*)} \left |\Phi_\mis^{\pi^*}(w, V^{\pi^*}; \Delta^*, \Theta^*) \right| + \varepsilon_L^* + 2\hat \varepsilon_L  \\
& \leq \max_{w\in \ci^\mis_{\alpha_\mis}(\Delta^*, \Theta^*, \pi^*)} \max_{f\in \cV} \left |\Phi_\mis^{\pi^*}(w, f; \Delta^*, \Theta^*) \right| + \varepsilon_L^* + 2\hat \varepsilon_L  \\ 
& \leq \max_{w\in \ci^\mis_{\alpha_\mis}(\Delta^*, \Theta^*, \pi^*)} \max_{f\in \cV} \max \left \{ \Phi_\mis^{\pi^*}(w, f; \Delta^*, \Theta^*), -\Phi_\mis^{\pi^*}(w, f; \Delta^*, \Theta^*) \right \} \\
& \qquad + \varepsilon_L^* + 2\hat \varepsilon_L  \\ 
& \leq \max_{w\in \ci^\mis_{\alpha_\mis}(\Delta^*, \Theta^*, \pi^*)} \max_{f\in \cV} \max \left \{ \Phi_\mis^{\pi^*}(w, f; \Delta^*, \Theta^*), \Phi_\mis^{\pi^*}(w, -f; \Delta^*, \Theta^*) \right \}  \\
& \qquad + \varepsilon_L^* + 2\hat \varepsilon_L  \\
& \leq \max_{w\in \ci^\mis_{\alpha_\mis}(\Delta^*, \Theta^*, \pi^*)} \max_{f\in \cV}  \Phi_\mis^{\pi^*}(w, f; \Delta^*, \Theta^*) + \varepsilon_L^* + 2\hat \varepsilon_L, 
\#
where in the second inequality, we use Lemma \ref{lemma:iv-link-phi-l}; in the third inequality, we use Assumption \ref{ass:iv-mis-realizable}; while in the last inequality, we use the fact that $\cV$ is symmetric. 
Now, by Lemma \ref{lemma:iv-w-in-conf-good} and plugging the definition of $\hat \varepsilon_L$ and $\varepsilon^*_L$ into \eqref{eq:thm-mis-f12}, it holds with probability at least $1 - \delta$ that
\$
J(\pi^*) - J(\hat \pi_\mis) \leq c\cdot \frac{ C_{\Delta^*}^2 C_{\Theta^*}^2 C_*}{1-\gamma} (\xi_0 + \xi_1) \sqrt{\frac{1}{NT \kappa}\cdot \pdim_{\cF_0,\cF_1,\cW,\cV,\Pi} \cdot \log\frac{NT }{\delta} }, 
\$
which concludes the proof of the theorem. 
\end{proof}

\subsection{Proof of Lemma \ref{lemma:iv-w-pi-in-conf}}\label{prf:lemma:iv-w-pi-in-conf}
\begin{proof}
First, by Assumption \ref{ass:iv-mis-realizable}, we know that $w^\pi \in \cW$. 
For notation simplicity, we denote by $\Phi_\mis^\pi(w, f; *) = \Phi_\mis^\pi(w, f; \Delta^*, \Theta^*)$ and $\hat w^\pi_* = \hat w^\pi_{\Delta^*, \Theta^*}$ for any $(\pi,w,f)$.
Note that 
\#\label{eq:iv-ee1}
& \max_{f\in \cV} \hat \Phi_\mis^\pi(w^\pi, f; *) - \max_{f\in \cV} \hat \Phi_\mis^\pi(\hat w^\pi_*, f; *)  \\
& \qquad = \max_{f\in \cV} \hat \Phi_\mis^\pi(w^\pi, f; *) - \max_{f\in \cV}  \Phi_\mis^\pi(w^\pi, f; *)  + \max_{f\in \cV}  \Phi_\mis^\pi(w^\pi, f; *) - \max_{f\in \cV}  \Phi_\mis^\pi(\hat w^\pi_*, f; *)  \\
& \qquad \qquad + \max_{f\in \cV} \Phi_\mis^\pi(\hat w^\pi_*, f; *) - \max_{f\in \cV} \hat \Phi_\mis^\pi(\hat w^\pi_*, f; *)  \\
& \qquad \leq \max_{f\in \cV} \hat \Phi_\mis^\pi(w^\pi, f; *) - \max_{f\in \cV}  \Phi_\mis^\pi(w^\pi, f; *) + \max_{f\in \cV} \Phi_\mis^\pi(\hat w^\pi_*, f; *) - \max_{f\in \cV} \hat \Phi_\mis^\pi(\hat w^\pi_*, f; *)  \\
& \qquad \leq 2 \max_{w\in \cW} \left | \max_{f\in \cV} \hat \Phi_\mis^\pi(w, f; *) - \max_{f\in \cV}  \Phi_\mis^\pi(w, f; *) \right |  \\
& \qquad \leq 2 \max_{w\in \cW} \max_{f\in \cV} \left |  \hat \Phi_\mis^\pi(w, f; *) - \Phi_\mis^\pi(w, f; *) \right |,
\#
where in the first inequality, we use the fact that $w^\pi = \argmin_{w\in \cW} \max_{f\in \cV}  \Phi_\mis^\pi(w, f; *)$; while in the second inequality, we use $w^\pi \in \cW$ by Assumption \ref{ass:iv-mis-realizable}. In the meanwhile, by Theorem \ref{thm:hoeffding-mixing}, with probability at least $1 - \delta$, it holds for any $(w,f,\pi)\in \cW\times \cV \times \Pi$ that
\#\label{eq:iv-ee2}
\left |  \hat \Phi_\mis^\pi(w, f; *) - \Phi_\mis^\pi(w, f; *) \right | \leq c\cdot \frac{C_{\Delta^*} C_{\Theta^*} C_*}{1-\gamma} \sqrt{\frac{1}{NT \kappa} \pdim_{\cV,\cW,\Pi}\log\frac{1}{\delta} \log(NT)}, 
\#
where we use Assumption \ref{ass:upper-bound-delta}. 
Now, combining \eqref{eq:iv-ee1} and \eqref{eq:iv-ee2}, with probability at least $1 - \delta$, we have
\$
& \max_{f\in \cV} \hat \Phi_\mis^\pi(w^\pi, f; *) - \max_{f\in \cV} \hat \Phi_\mis^\pi(\hat w^\pi_*, f; *) \\
& \qquad \leq c\cdot \frac{C_{\Delta^*} C_{\Theta^*} C_*}{1-\gamma} \sqrt{\frac{1}{NT \kappa} \pdim_{\cV,\cW,\Pi}\log\frac{1}{\delta} \log(NT)} = \alpha_\mis, 
\$
which implies that $w^\pi \in \ci^\mis_{\alpha_\mis}(\Delta^*, \Theta^*, \pi)$. This concludes the proof of the lemma. 
\end{proof}

\subsection{Proof of Lemma \ref{lemma:iv-w-in-conf-good}}
\label{prf:lemma:iv-w-in-conf-good}
\begin{proof}
Since $w \in \cup_{(\Delta, \Theta) \in \ci^0_{\alpha_0} \times \ci^1_{\alpha_1}} \ci^\mis_{\alpha_\mis}(\Delta, \Theta, \pi)$, there exists a pair $(\tilde \Delta, \tilde\Theta) \in \ci^0_{\alpha_0} \times \ci^1_{\alpha_1}$ such that $w \in \ci^\mis_{\alpha_\mis}(\tilde \Delta, \tilde \Theta, \pi)$. For the simplicity of notations, we denote by 
\#\label{eq:iv-dd0}
\tilde w \in \argmin_{w\in \cW} \max_{f\in \cV} \hat \Phi^\pi_\mis(w, f;\tilde \Delta, \tilde \Theta),
\#
i.e., $\tilde w = \hat w^\pi_{\tilde \Delta, \tilde \Theta}$, which is defined in \eqref{eq:w-pi-def}. By the definition of $\tilde w$ and $w \in \ci^\mis_{\alpha_\mis}(\tilde \Delta, \tilde \Theta, \pi)$, with probability at least $1 - \delta$, it holds for any $\pi\in \Pi$ and $w \in \ci^\mis_{\alpha_\mis}(\tilde \Delta, \tilde \Theta, \pi)$ that
\#\label{eq:iv-dd1}
\max_{f\in \cV} \hat \Phi^\pi_\mis(w,f;\tilde \Delta, \tilde \Theta) - \max_{f\in \cV} \hat \Phi^\pi_\mis(\tilde w,f;\tilde \Delta, \tilde \Theta) \leq \alpha_\mis. 
\#
Further, we observe that
\#\label{eq:iv-dd2}
& \max_{f\in \cV} \Phi^\pi_\mis(w,f;\Delta^*, \Theta^*)  \\
& \qquad \leq \underbrace{\max_{(w,f,\Delta,\Theta)\in (\cW, \cV, \cF_0, \cF_1)} \left | \Phi^\pi_\mis(w,f;\Delta, \Theta) - \hat \Phi^\pi_\mis(w,f;\Delta, \Theta) \right|}_{\text{Term (I)}} + \underbrace{\max_{f\in \cV} \Phi^\pi_\mis(\tilde w,f;\tilde \Delta, \tilde \Theta)}_{\text{Term (II)}}  \\
& \qquad \qquad + \underbrace{\max_{f\in \cV} \left | \hat \Phi^\pi_\mis(w,f;\Delta^*, \Theta^*) - \hat \Phi^\pi_\mis(w,f;\tilde \Delta, \tilde \Theta) \right|}_{\text{Term (III)}} + \alpha_\mis, 
\#
where we use \eqref{eq:iv-dd1} in the last inequality. Now we upper bound terms (I), (II), and (III) on the RHS of \eqref{eq:iv-dd2}. 

\vskip5pt
\noindent\textbf{Upper Bounding Term (I).}
By Theorem \ref{thm:hoeffding-mixing}, with probability at least $1 - \delta$, it holds for any $(w,f,\Delta,\Theta,\pi) \in (\cW, \cV, \cF_0, \cF_1, \Pi)$ that
\$
\left |  \hat \Phi_\mis^\pi(w, f; \Delta, \Theta) - \Phi_\mis^\pi(w, f; \Delta, \Theta) \right | \leq c\cdot \frac{C_{\Delta^*} C_{\Theta^*} C_*}{1-\gamma} \sqrt{\frac{1}{NT \kappa}\pdim_{\cF_0,\cF_1,\cW,\cV,\Pi}\log\frac{1}{\delta} \log(NT)},
\$
which implies that with probability at least $1 - \delta$, we have
\#\label{eq:iv-dd3}
\text{Term (I)} \leq c\cdot \frac{C_{\Delta^*} C_{\Theta^*} C_*}{1-\gamma} \sqrt{\frac{1}{NT \kappa}\pdim_{\cF_0,\cF_1,\cW,\cV,\Pi}\log\frac{1}{\delta} \log(NT)}.
\#

\vskip5pt
\noindent\textbf{Upper Bounding Term (II).} 
We introduce the following lemma to help upper bound term (II).

\begin{lemma}
\label{lemma:iv-any-w-hat-small-loss}
Suppose $(\alpha_0,\alpha_1)$ is defined in Assumption \ref{ass:iv-sl-res}. 
With probability at least $1 - \delta$, for any $(\Delta, \Theta) \in \ci^0_{\alpha_0} \times \ci^1_{\alpha_1}$ and $\pi\in \Pi$, we have
\$
\max_{f\in \cV} \Phi_\mis^\pi(\hat w^\pi_{\Delta, \Theta}, f; \Delta, \Theta) \leq c\cdot \frac{C_{\Delta^*}^2 C_{\Theta^*}^2 C_*}{1-\gamma} (\xi_0 + \xi_1 ) \sqrt{\frac{1}{NT \kappa}\pdim_{\cF_0,\cF_1,\cW,\cV,\Pi} \log\frac{1}{\delta} \log(NT)},
\$
where $\hat w^\pi_{\Delta, \Theta}$ is defined in \eqref{eq:w-pi-def}, $\xi_0$ and $\xi_1$ are constants defined in Assumption \ref{ass:iv-sl-res}. 
\end{lemma}
\begin{proof}
See \S\ref{prf:lemma:iv-any-w-hat-small-loss} for a detailed proof. 
\end{proof}

By the definition of $\tilde w$ in \eqref{eq:iv-dd0} and Lemma \ref{lemma:iv-any-w-hat-small-loss}, with probability at least $1 - \delta$, we have
\#\label{eq:iv-dd4}
\text{Term (II)} \leq c\cdot \frac{C_{\Delta^*}^2 C_{\Theta^*}^2 C_*}{1-\gamma} (\xi_0 + \xi_1 ) \sqrt{\frac{1}{NT \kappa}\pdim_{\cF_0,\cF_1,\cW,\cV,\Pi} \log\frac{1}{\delta} \log(NT)}.
\#

\vskip5pt
\noindent\textbf{Upper Bounding Term (III).} 
Note that 
\#\label{eq:iv-dd5}
& \left | \hat \Phi^\pi_\mis(w,f;\Delta^*, \Theta^*) - \hat \Phi^\pi_\mis(w,f;\tilde \Delta, \tilde \Theta) \right|  \\
& \leq \left | \left(\hat \EE - \EE\right) \left[ \frac{1}{T} \sum_{t = 0}^{T-1} \left ( \frac{Z_t^\top A_t\pi(A_t\given S_t) w(S_t) }{\Delta^*(S_t,A_t)\Theta^*(S_t,Z_t)} - \frac{Z_t^\top A_t\pi(A_t\given S_t) w(S_t)}{\tilde \Delta(S_t,A_t) \tilde \Theta(S_t,Z_t)} \right) \left(f(S_t) - \gamma f(S_{t+1}) \right) \right ] \right |  \\
& \qquad  + \left |  \EE \left[ \frac{1}{T} \sum_{t = 0}^{T-1} \left ( \frac{Z_t^\top A_t\pi(A_t\given S_t) w(S_t) }{\Delta^*(S_t,A_t)\Theta^*(S_t,Z_t)} - \frac{Z_t^\top A_t\pi(A_t\given S_t) w(S_t)}{\tilde \Delta(S_t,A_t) \tilde \Theta(S_t,Z_t)} \right) \left(f(S_t) - \gamma f(S_{t+1}) \right) \right ] \right |. 
\#
For the first term on the RHS of \eqref{eq:iv-dd5}, by Theorem \ref{thm:hoeffding-mixing}, with probability at least $1- \delta$, it holds for any $(w,f,\pi)\in \cW\times \cV \times \Pi$ that
\#\label{eq:iv-dd5-1}
& \left | \left(\hat \EE - \EE\right) \left[ \frac{1}{T} \sum_{t = 0}^{T-1} \left ( \frac{Z_t^\top A_t\pi(A_t\given S_t) w(S_t) }{\Delta^*(S_t,A_t)\Theta^*(S_t,Z_t)} - \frac{Z_t^\top A_t\pi(A_t\given S_t) w(S_t)}{\tilde \Delta(S_t,A_t) \tilde \Theta(S_t,Z_t)} \right) \left(f(S_t) - \gamma f(S_{t+1}) \right) \right ] \right |  \\
& \qquad \leq c\cdot \frac{C_{\Delta^*} C_{\Theta^*} C_*}{1-\gamma} \sqrt{\frac{\pdim_{\cF_0,\cF_1,\cW,\cV,\Pi}}{NT \kappa}\log\frac{1}{\delta} \log(NT)}. 
\#
For the second term on the RHS of \eqref{eq:iv-dd5}, by a similar argument as in \eqref{eq:iv-dd5-vf-2}, it holds with probability at least $1 - \delta$ that 
\#\label{eq:iv-dd5-2}
& \left | \EE \left[ \frac{1}{T} \sum_{t = 0}^{T-1} \left ( \frac{Z_t^\top A_t\pi(A_t\given S_t) w(S_t)}{\Delta^*(S_t,A_t)\Theta^*(S_t,Z_t)} - \frac{Z_t^\top A_t\pi(A_t\given S_t)w(S_t)}{\tilde \Delta(S_t,A_t) \tilde \Theta(S_t,Z_t)} \right) \left(f(S_t) - \gamma f(S_{t+1}) \right) \right ] \right |  \\
& \leq \frac{2 C_{\Delta^*} C_{\Theta^*} C_*}{1-\gamma} \left(\xi_0 C_{\Theta^*} \sqrt{\frac{C_{\Delta^*}}{NT \kappa} \pdim_{\cF_0}\log\frac{1}{\delta} \log(NT)} + \xi_1 C_{\Delta^*} \sqrt{\frac{C_{\Theta^*}}{NT \kappa} \pdim_{\cF_1} \log\frac{1}{\delta} \log(NT)} \right), 
\#
where in the first inequality, we use the fact that $\|f\|_\infty \leq 1/ (1 - \gamma)$ and $\|w\|_\infty \leq C_*$; in the third inequality, we use Cauchy Schwarz inequality; while in the last inequality, we use Assumption \ref{ass:iv-sl-res} with the fact that $(\tilde \Delta, \tilde\Theta) \in \ci^0_{\alpha_0} \times \ci^1_{\alpha_1}$. Now, by plugging \eqref{eq:iv-dd5-1} and \eqref{eq:iv-dd5-2} into \eqref{eq:iv-dd5}, with probability at least $1 - \delta$,  it holds for any $w \in \cup_{(\Delta, \Theta) \in \ci^0_{\alpha_0} \times \ci^1_{\alpha_1}} \ci^\mis_{\alpha_\mis}(\Delta, \Theta, \pi)$ and $(f,\pi)\in \cV\times \Pi$ that 
\#\label{eq:732847}
& \left | \hat \Phi^\pi_\mis(w,f;\Delta^*, \Theta^*) - \hat \Phi^\pi_\mis(w,f;\tilde \Delta, \tilde \Theta) \right|  \\
& \qquad \leq c \cdot \frac{ C_{\Delta^*}^2 C_{\Theta^*}^2 C_*}{1-\gamma} (\xi_0 + \xi_1) \sqrt{\frac{1}{NT \kappa}\cdot \pdim_{\cF_0,\cF_1,\cW,\cV,\Pi} \cdot \log\frac{1}{\delta} \log(NT) }. 
\#

\vskip5pt
Now, by plugging \eqref{eq:iv-dd3}, \eqref{eq:iv-dd4}, and \eqref{eq:732847} into 
\eqref{eq:iv-dd2}, with probability at least $1 - \delta$, it holds for any $\pi \in \Pi$ and $w \in \cup_{(\Delta, \Theta) \in \ci^0_{\alpha_0} \times \ci^1_{\alpha_1}} \ci^\mis_{\alpha_\mis}(\Delta, \Theta, \pi)$ that
\$
& \max_{f\in \cV} \Phi^\pi_\mis(w,f;\Delta^*, \Theta^*) \leq c\cdot \frac{ C_{\Delta^*}^2 C_{\Theta^*}^2 C_*}{1-\gamma} (\xi_0 + \xi_1) \sqrt{\frac{1}{NT \kappa}\cdot \pdim_{\cF_0,\cF_1,\cW,\cV,\Pi} \cdot \log\frac{1}{\delta} \log(NT) },
\$
which concludes the proof of the lemma. 
\end{proof}

\subsection{Proof of Lemma \ref{lemma:iv-link-phi-l}}\label{prf:lemma:iv-link-phi-l}
\begin{proof}
Since $\Phi^\pi_\mis(w^\pi, V^\pi; \Delta^*, \Theta^*) = 0$, we have
\$
& \Phi_\mis^\pi(w, V^\pi; \Delta^*, \Theta^*) \\
& \qquad = \Phi_\mis^\pi(w, V^\pi; \Delta^*, \Theta^*) - \Phi_\mis^\pi(w^\pi, V^\pi; \Delta^*, \Theta^*)  \\
& \qquad = \EE \left[ \frac{1}{T} \sum_{t = 0}^{T-1} \frac{Z_t^\top A_t \pi(A_t\given S_t)}{\Delta^*(S_t,A_t) \Theta^*(S_t,Z_t)} \left(w^\pi(S_t)-w(S_t)\right) \left (V^\pi(S_t) - \gamma V^\pi(S_{t+1}) \right ) \right]  \\
& \qquad = \EE \left[ \frac{1}{T} \sum_{t = 0}^{T-1} \left(w^\pi(S_t)-w(S_t)\right) \EE_{\pi} \left [V^\pi(S_t) - \gamma V^\pi(S_{t+1}) \given S_t \right ] \right]  \\
& \qquad = \EE \left[ \frac{1}{T} \sum_{t = 0}^{T-1} \left(w^\pi(S_t)-w(S_t)\right) \EE_\pi[R_t \given S_t] \right] \\
& \qquad = \EE \left[\frac{1}{T} \sum_{t = 0}^{T-1} \frac{Z_t^\top A_t \pi(A_t\given S_t)}{\Delta^*(S_t,A_t) \Theta^*(S_t,Z_t)} \left(w^\pi(S_t) - w(S_t)\right) R_t \right]  \\
& \qquad = L_\mis(w^\pi, \pi; \Delta^*, \Theta^*) - L_\mis(w, \pi; \Delta^*, \Theta^*), 
\$
which concludes the proof of the lemma. 
\end{proof}

\subsection{Proof of Lemma \ref{lemma:iv-mis-min-delta-close}}
\label{prf:lemma:iv-mis-min-delta-close}
\begin{proof}
With a slight abuse of notations, we denote by
\$
w_0 \in \argmin_{w\in \ci^\mis_{\alpha_\mis}(\Delta^*, \Theta^*, \pi^*)} L_\mis(w, \pi^*; \Delta^*, \Theta^*), \qquad w_1 \in \argmin_{w\in \ci^\mis_{\alpha_\mis}(\Delta, \Theta, \pi^*)} L_\mis(w, \pi^*; \Delta, \Theta). 
\$
Then we have
\#\label{eq:iv-kkk1}
& \left | \min_{w\in \ci^\mis_{\alpha_\mis}(\Delta^*, \Theta^*, \pi^*)} L_\mis(w, \pi^*; \Delta^*, \Theta^*) - \min_{w\in \ci^\mis_{\alpha_\mis}(\Delta, \Theta, \pi^*)} L_\mis(w, \pi^*; \Delta, \Theta) \right | \\
& \qquad = \left | L_\mis(w_0, \pi^*; \Delta^*, \Theta^*) - L_\mis(w_1, \pi^*; \Delta, \Theta) \right|  \\
& \qquad \leq \left | L_\mis(w_0, \pi^*; \Delta^*, \Theta^*) - L_\mis(w^{\pi^*}, \pi^*; \Delta^*, \Theta^*) \right |  \\
& \qquad \qquad + \left | L_\mis(w^{\pi^*}, \pi^*; \Delta^*, \Theta^*) - L_\mis(w_1, \pi^*; \Delta^*, \Theta^*) \right |  \\
& \qquad \qquad + \left |L_\mis(w_1, \pi^*; \Delta^*, \Theta^*) - L_\mis(w_1, \pi^*; \Delta, \Theta) \right |  \\
& \qquad = \underbrace{\left | \Phi_\mis^{\pi^*}(w_0, V^{\pi^*}; \Delta^*, \Theta^*) \right|}_{\text{Term (I)}} + \underbrace{\left| \Phi_\mis^{\pi^*}(w_1, V^{\pi^*}; \Delta^*, \Theta^*) \right|}_{\text{Term (II)}} \\
& \qquad \qquad + \underbrace{\left | L_\mis(w_1, \pi^*; \Delta^*, \Theta^*) - L_\mis(w_1, \pi^*; \Delta, \Theta) \right|}_{\text{Term (III)}}. 
\#
We upper bound terms (I), (II), and (III) on the RHS of \eqref{eq:iv-kkk1}, respectively. 

\vskip5pt
\noindent\textbf{Upper Bounding Term (I).}
Note that with probability at least $1 - \delta$, we have
\#\label{eq:iv-kkk2}
\left | \Phi_\mis^{\pi^*}(w_0, V^{\pi^*}; \Delta^*, \Theta^*) \right| & \leq \max_{f\in \cV} \left | \Phi_\mis^{\pi^*}(w_0, f; \Delta^*, \Theta^*) \right|  \\
& = \max_{f\in \cV} \max \left\{  \Phi_\mis^{\pi^*}(w_0, f; \Delta^*, \Theta^*), 
-\Phi_\mis^{\pi^*}(w_0, f; \Delta^*, \Theta^*)  \right \}  \\
& = \max_{f\in \cV} \max \left\{  \Phi_\mis^{\pi^*}(w_0, f; \Delta^*, \Theta^*), 
\Phi_\mis^{\pi^*}(w_0, -f; \Delta^*, \Theta^*)  \right \}  \\
& = \max_{f\in \cV} \Phi_\mis^{\pi^*}(w_0, f; \Delta^*, \Theta^*)  \\
& \leq c\cdot \frac{ C_{\Delta^*}^2 C_{\Theta^*}^2 C_*}{1-\gamma} (\xi_0  + \xi_1) \sqrt{\frac{1}{NT \kappa} \pdim_{\cF_0,\cF_1,\cW,\cV} \log\frac{1}{\delta} \log(NT)}, 
\#
where in the first inequality, we use the fact that $V^{\pi^*} \in \cV$; in the third equality, we use the fact that $\cV$ is symmetric; in the last inequality, by noting that $w_0 \in \cup_{(\Delta, \Theta) \in \ci^0_{\alpha_0} \times \ci^1_{\alpha_1}} \ci^\mis_{\alpha_\mis}(\Delta, \Theta, \pi)$, we use Lemma \ref{lemma:iv-w-in-conf-good}. This upper bounds term (I) on the RHS of \eqref{eq:iv-kkk1}.

\vskip5pt
\noindent\textbf{Upper Bounding Term (II).}
Similar to \eqref{eq:iv-kkk2}, note that $w_1 \in \cup_{(\Delta, \Theta) \in \ci^0_{\alpha_0} \times \ci^1_{\alpha_1}} \ci^\mis_{\alpha_\mis}(\Delta, \Theta, \pi)$, it holds with probability at least $1 - \delta$ that
\#\label{eq:iv-kkk3}
\left | \Phi_\mis^{\pi^*}(w_1, V^{\pi^*}; \Delta^*, \Theta^*) \right| \leq c\cdot \frac{ C_{\Delta^*}^2 C_{\Theta^*}^2 C_*}{1-\gamma} (\xi_0  + \xi_1) \sqrt{\frac{1}{NT \kappa} \pdim_{\cF_0,\cF_1,\cW,\cV} \log\frac{1}{\delta} \log(NT)}, 
\#
which upper bounds term (II) on the RHS of \eqref{eq:iv-kkk1}. 

\vskip5pt
\noindent\textbf{Upper Bounding Term (III).}
Note that with probability at least $1 - \delta$, it holds for any $(\Delta, \Theta) \in \ci^0_{\alpha_0} \times \ci^1_{\alpha_1}$ that
\#\label{eq:iv-kkk4}
& \left | L_\mis(w_1, \pi^*; \Delta^*, \Theta^*) - L_\mis(w_1, \pi^*; \Delta, \Theta) \right|  \\
& \quad = \EE \left[ \frac{1}{T} \sum_{t = 0}^{T-1} \left( \frac{Z_t^\top A_t \pi^*(A_t\given S_t)}{\Delta^*(S_t,A_t) \Theta^*(S_t,Z_t)} - \frac{Z_t^\top A_t \pi^*(A_t\given S_t)}{\Delta(S_t,A_t) \Theta(S_t,Z_t)} \right) w_1(S_t) R_t \right]  \\
& \quad \leq C_{\Delta^*} C_{\Theta^*} C_* \left(\xi_0 C_{\Theta^*} \sqrt{\frac{C_{\Delta^*}}{NT \kappa} \pdim_{\cF_0} \log\frac{1}{\delta} \log(NT)} + \xi_1 C_{\Delta^*} \sqrt{\frac{C_{\Theta^*}}{NT \kappa} \pdim_{\cF_1} \log\frac{1}{\delta} \log(NT)} \right),
\#
where we use Cauchy-Schwarz inequality and Assumption \ref{ass:iv-sl-res} in the last inequality. 

\vskip5pt
Now, by plugging \eqref{eq:iv-kkk2}, \eqref{eq:iv-kkk3}, and \eqref{eq:iv-kkk4} into \eqref{eq:iv-kkk1}, with probability at least $1 - \delta$, it holds for any $(\Delta, \Theta) \in \ci^0_{\alpha_0} \times \ci^1_{\alpha_1}$ that 
\$
& \left | \min_{w\in \ci^\mis_{\alpha_\mis}(\Delta^*, \Theta^*, \pi^*)} L_\mis(w, \pi^*; \Delta^*, \Theta^*) - \min_{w\in \ci^\mis_{\alpha_\mis}(\Delta, \Theta, \pi^*)} L_\mis(w, \pi^*; \Delta, \Theta) \right |  \\
& \qquad \leq c\cdot \frac{ C_{\Delta^*}^2 C_{\Theta^*}^2 C_*}{1-\gamma} (\xi_0  + \xi_1) \sqrt{\frac{1}{NT \kappa} \pdim_{\cF_0,\cF_1,\cW,\cV} \log\frac{1}{\delta} \log(NT)}, 
\$
which concludes the proof of the lemma. 
\end{proof}

\subsection{Proof of Lemma \ref{lemma:iv-mis-hatl-close}}
\label{prf:lemma:iv-mis-hatl-close}
\begin{proof}
By Theorem \ref{thm:hoeffding-mixing}, with probability at least $1 - \delta$, it holds for any $(w, \Delta, \Theta, \pi)\in \cW \times \cF_0 \times \cF_1 \times \Pi$ that
\$
\left | L_\mis(w, \pi; \Delta, \Theta) - \hat L_\mis(w, \pi; \Delta, \Theta) \right | \leq c\cdot C_{\Delta^*} C_{\Theta^*} C_* \sqrt{\frac{1}{NT \kappa} \pdim_{\cF_0,\cF_1,\cW,\Pi} \log\frac{1}{\delta} \log(NT)}, 
\$
which concludes the proof of the lemma. 
\end{proof}

\subsection{Proof of Lemma \ref{lemma:iv-any-w-hat-small-loss}}
\label{prf:lemma:iv-any-w-hat-small-loss}
\begin{proof}
Note that 
\#\label{eq:iv-ee3}
& \max_{f\in \cV} \Phi_\mis^\pi(\hat w^\pi_{\Delta, \Theta}, f; \Delta, \Theta)  \\
& \qquad = \max_{f\in \cV} \Phi_\mis^\pi(\hat w^\pi_{\Delta, \Theta}, f; \Delta, \Theta) - \max_{f\in \cV} \hat \Phi_\mis^\pi(\hat w^\pi_{\Delta, \Theta}, f; \Delta, \Theta) + \max_{f\in \cV} \hat \Phi_\mis^\pi(\hat w^\pi_{\Delta, \Theta}, f; \Delta, \Theta)  \\
& \qquad \qquad - \max_{f\in \cV} \hat \Phi_\mis^\pi( w^\pi, f; \Delta, \Theta)  + \max_{f\in \cV} \hat \Phi_\mis^\pi( w^\pi, f; \Delta, \Theta) - \max_{f\in \cV}  \Phi_\mis^\pi(w^\pi, f; \Delta, \Theta)  \\
& \qquad \qquad + \max_{f\in \cV}  \Phi_\mis^\pi(w^\pi, f; \Delta, \Theta) - \max_{f\in \cV}  \Phi_\mis^\pi(w^\pi, f; \Delta^*, \Theta^*)  \\
& \qquad \leq 2 \max_{w\in \cW} \max_{f\in \cV} \left |  \Phi_\mis^\pi(w, f; \Delta, \Theta) -  \hat \Phi_\mis^\pi(w, f; \Delta, \Theta) \right |  \\
& \qquad \qquad + \max_{f\in \cV} \left |  \Phi_\mis^\pi(w^\pi, f; \Delta, \Theta) -   \Phi_\mis^\pi(w^\pi, f; \Delta^*, \Theta^*) \right |,
\#
where we use the fact that $\hat w^\pi_{\Delta, \Theta} \in \argmin_{w\in \cW} \max_{f\in \cV} \hat \Phi^\pi_\mis(w, f; \Delta, \Theta)$ in the last inequality. In the meanwhile, by Theorem \ref{thm:hoeffding-mixing}, with probability at least $1 - \delta$, it holds for any $(w,f,\pi)\in \cW\times \cV \times \Pi$ that
\#\label{eq:iv-ee4}
\left |  \hat \Phi_\mis^\pi(w, f; \Delta, \Theta) - \Phi_\mis^\pi(w, f; \Delta, \Theta) \right | \leq c\cdot \frac{C_{\Delta^*} C_{\Theta^*} C_*}{1-\gamma} \sqrt{\frac{1}{NT \kappa}\pdim_{\cV,\cW,\Pi}\log\frac{1}{\delta} \log(NT)}. 
\#
Also, we upper bound the second term on the RHS of \eqref{eq:iv-ee3} with probability at least $1- \delta$ as follows, 
\#\label{eq:iv-ee5}
& \left |  \Phi_\mis^\pi(w^\pi, f; \Delta, \Theta) -   \Phi_\mis^\pi(w^\pi, f; \Delta^*, \Theta^*) \right |  \\
& \quad = \left | \EE\left[ \frac{1}{T} \sum_{t = 0}^{T-1} \left ( \frac{Z_t^\top A_t\pi(A_t\given S_t)w(S_t) }{\Delta^*(S_t,A_t)\Theta^*(S_t,Z_t)} - \frac{Z_t^\top A_t\pi(A_t\given S_t) w(S_t)}{\Delta(S_t,A_t)\Theta(S_t,Z_t)} \right) \left(f(S_t) - \gamma f(S_{t+1})\right) \right ] \right |  \\
& \quad \leq \frac{C_{\Delta^*} C_{\Theta^*} C_*}{1-\gamma} \left(\xi_0 C_{\Theta^*} \sqrt{\frac{C_{\Delta^*}}{NT \kappa}\pdim_{\cF_0}\log\frac{1}{\delta} \log(NT)} + \xi_1 C_{\Delta^*} \sqrt{\frac{C_{\Theta^*}}{NT \kappa} \pdim_{\cF_1} \log\frac{1}{\delta} \log(NT)} \right),
\#
where we use Cauchy-Schwarz inequality and Assumption \ref{ass:iv-sl-res} in the last inequality.

Now, by plugging \eqref{eq:iv-ee4} and \eqref{eq:iv-ee5} into \eqref{eq:iv-ee3}, with probability at least $1 - \delta$, it holds for any $(\Delta, \Theta) \in \ci^0_{\alpha_0} \times \ci^1_{\alpha_1}$ and $\pi\in \Pi$ that 
\$
\max_{f\in \cV} \Phi_\mis^\pi(\hat w^\pi_{\Delta, \Theta}, f; \Delta, \Theta) \leq c\cdot \frac{C_{\Delta^*}^2 C_{\Theta^*}^2 C_*}{1-\gamma} (\xi_0 + \xi_1 ) \sqrt{\frac{1}{NT \kappa}\pdim_{\cF_0,\cF_1,\cW,\cV,\Pi} \log\frac{1}{\delta} \log(NT)},
\$
which concludes the proof of the lemma.
\end{proof}

\section{Proof of Results in \S\ref{sec:dr-theory}}

\subsection{Proof of Theorem \ref{thm:iv-dr}}\label{prf:thm:iv-dr}
\begin{proof}
We split the proof into two case: (i) Assumption \ref{ass:iv-vf-realizable} holds; (ii) Assumption \ref{ass:iv-mis-realizable} holds.

\vskip5pt
\noindent\textbf{Case (i): Assumption \ref{ass:iv-vf-realizable} holds.} We introduce the following supporting lemmas. 

\begin{lemma}
\label{lemma:iv-mis-hatl-close-dr}
For any policy $\pi$, with probability at least $1 - \delta$ with $c / (NT)^2 \leq \delta \leq 1$, it holds for any $(w,v, \Delta, \Theta, \pi)\in \cW \times \cV \times \cF_0 \times \cF_1 \times \Pi$ that 
\$
& \left | L_\dr(w, v, \pi; \Delta, \Theta) - \hat L_\dr(w, v, \pi; \Delta, \Theta) \right | \\
& \qquad \leq c\cdot \frac{C_{\Delta^*} C_{\Theta^*} C_*}{1-\gamma} \sqrt{\frac{1}{NT \kappa}\pdim_{\cF_0,\cF_1,\cW,\cV,\Pi} \log\frac{1}{\delta} \log(NT)} = \hat \epsilon_L. 
\$
\end{lemma}
\begin{proof}
See \S\ref{prf:lemma:iv-mis-hatl-close-dr} for a detailed proof. 
\end{proof}

\begin{lemma}
\label{lemma:iv-mis-min-delta-close-dr}
Suppose that $(\alpha_0, \alpha_1, \alpha_\mis, \alpha_\vf)$ is defined in Assumption \ref{ass:iv-sl-res}, Lemmas \ref{lemma:iv-w-pi-in-conf}, and \ref{lemma:iv-v-pi-in-conf}. With probability at least $1 - \delta$, it holds for any $(\Delta, \Theta) \in \ci^0_{\alpha_0} \times \ci^1_{\alpha_1}$ that 
\$
& \left | \min_{(w,v)\in \ci_{\alpha_\mis,\alpha_\vf}(\Delta^*, \Theta^*, \pi^*)} L_\dr(w, v, \pi^*; \Delta^*, \Theta^*) - \min_{(w,v)\in \ci_{\alpha_\mis,\alpha_\vf}(\Delta, \Theta, \pi^*)} L_\dr(w, v, \pi^*; \Delta, \Theta) \right |  \\
& \qquad \leq c\cdot \frac{ C_{\Delta^*}^2 C_{\Theta^*}^2 C_*}{1-\gamma} (\xi_0 + \xi_1) \sqrt{\frac{1}{NT \kappa} \pdim_{\cF_0,\cF_1,\cW,\cV} \log\frac{1}{\delta} \log(NT)} = \epsilon_L^*. 
\$
\end{lemma}
\begin{proof}
See \S\ref{prf:lemma:iv-mis-min-delta-close-dr} for a detailed proof. 
\end{proof}

By the definition of $L_{\dr}$, it holds with probability at least $1 - \delta$ that
\#\label{eq:koko1}
& J(\pi^*) - J(\hat\pi_\dr) \\
& \qquad = J(\pi^*) - L_\dr(w, V^{\hat \pi_\dr}, \hat \pi_\dr; \Delta^*, \Theta^*)  \\
& \qquad \leq J(\pi^*) - \min_{(\Delta, \Theta) \in \ci^0_{\alpha_0}\times \ci^1_{\alpha_1}} \min_{(w,v) \in \ci_{\alpha_\mis, \alpha_\vf}(\Delta, \Theta, \hat \pi_\dr)} L_\dr(w, v, \hat \pi_\dr; \Delta, \Theta)  \\
& \qquad \leq J(\pi^*) - \min_{(\Delta, \Theta) \in \ci^0_{\alpha_0}\times \ci^1_{\alpha_1}} \min_{(w,v) \in \ci_{\alpha_\mis, \alpha_\vf}(\Delta, \Theta, \hat \pi_\dr)} \hat L_\dr(w, v, \hat \pi_\dr; \Delta, \Theta) + \hat \epsilon_L  \\
& \qquad \leq J(\pi^*) - \min_{(\Delta, \Theta) \in \ci^0_{\alpha_0}\times \ci^1_{\alpha_1}} \min_{(w,v) \in \ci_{\alpha_\mis, \alpha_\vf}(\Delta, \Theta, \pi^*)} \hat L_\dr(w, v, \pi^*; \Delta, \Theta) + \hat \epsilon_L  \\
& \qquad \leq J(\pi^*) - \min_{(\Delta, \Theta) \in \ci^0_{\alpha_0}\times \ci^1_{\alpha_1}} \min_{(w,v) \in \ci_{\alpha_\mis, \alpha_\vf}(\Delta, \Theta, \pi^*)} L_\dr(w, v, \pi^*; \Delta, \Theta) + 2 \hat \epsilon_L,
\#
where in the first inequality, we use Assumption \ref{ass:iv-sl-res} that $(\Delta^*, \Theta^*) \in \ci^0_{\alpha_0}\times \ci^1_{\alpha_1}$ with probability at least $1 - \delta$, and Lemma \ref{lemma:iv-v-pi-in-conf} with Assumption \ref{ass:iv-vf-realizable} that $V^{\hat \pi_\dr}\in \ci^\vf_{\alpha_\vf}(\Delta^*, \Theta^*, \hat \pi_\dr)$ with probability at least $1 - \delta$; in the second inequality, we use Lemma \ref{lemma:iv-mis-hatl-close-dr}; in the third inequality, we use the optimality of $\hat \pi_\dr$; while in the last inequality, we use Lemma \ref{lemma:iv-mis-hatl-close-dr} again. 
By combining Lemma \ref{lemma:iv-mis-min-delta-close-dr} and \eqref{eq:koko1}, we have
\#\label{eq:koko3}
& J(\pi^*) - J(\hat\pi_\dr)  \\
& \qquad \leq J(\pi^*) -  \min_{(w,v) \in \ci_{\alpha_\mis, \alpha_\vf}(\Delta^*, \Theta^*, \pi^*)} L_\dr(w, v, \pi^*; \Delta^*, \Theta^*) + 2 \hat \epsilon_L + \epsilon_L^*  \\
& \qquad = L_\dr(w, V^{\pi^*}, \pi^*; \Delta^*, \Theta^*) -  \min_{(w,v) \in \ci_{\alpha_\mis, \alpha_\vf}(\Delta^*, \Theta^*, \pi^*)} L_\dr(w, v, \pi^*; \Delta^*, \Theta^*) + 2 \hat \epsilon_L + \epsilon_L^*  \\
& \qquad = \max_{(w,v) \in \ci_{\alpha_\mis, \alpha_\vf}(\Delta^*, \Theta^*, \pi^*)} \left | L_\dr(w, V^{\pi^*}, \pi^*; \Delta^*, \Theta^*) - L_\dr(w, v, \pi^*; \Delta^*, \Theta^*) \right | + 2 \hat \epsilon_L + \epsilon_L^*  \\
& \qquad = \max_{(w,v) \in \ci_{\alpha_\mis, \alpha_\vf}(\Delta^*, \Theta^*, \pi^*)} \left | \Phi_\vf^{\pi^*}(v, w^{\pi^*}; \Delta^*, \Theta^*) - \Phi_\vf^{\pi^*}(v, w; \Delta^*, \Theta^*) \right | + 2 \hat \epsilon_L + \epsilon_L^*,
\#
where we use the following fact in the last equality, 
\$
L_\dr(w, V^{\pi^*}, \pi^*; \Delta^*, \Theta^*) - L_\dr(w, v, \pi^*; \Delta^*, \Theta^*) = \Phi_\vf^{\pi^*}(v, w^{\pi^*}; \Delta^*, \Theta^*) - \Phi_\vf^{\pi^*}(v, w; \Delta^*, \Theta^*). 
\$
In the meanwhile, note that by Assumption \ref{ass:iv-vf-realizable} that $w^{\pi^*} \in \cW$, we obtain from \eqref{eq:koko3} that 
\#\label{eq:iv-dr-case2}
J(\pi^*) - J(\hat\pi_\dr) & \leq 2 \max_{(w,v) \in \ci_{\alpha_\mis, \alpha_\vf}(\Delta^*, \Theta^*, \pi^*)} \left | \Phi_\vf^{\pi^*}(v, w; \Delta^*, \Theta^*) \right | + 2 \hat \epsilon_L + \epsilon_L^*  \\
& \leq c\cdot \frac{C_{\Delta^*}^2 C_{\Theta^*}^2 C_*}{1-\gamma} (\xi_0 + \xi_1) \sqrt{\frac{1}{NT \kappa} \pdim_{\cF_0,\cF_1,\cW,\cV,\Pi} \log\frac{NT }{\delta}}, 
\#
where we use Lemma \ref{lemma:iv-v-in-conf-good} and plug in the definition of $\epsilon_L$ and $\epsilon_L^*$ in the last inequality. This concludes the proof of case (i).

\vskip5pt
\noindent\textbf{Case (ii): Assumption \ref{ass:iv-mis-realizable} holds.}
It holds with probability at least $1 - \delta$ that
\$
J(\pi^*) - J(\hat\pi_\dr) & = J(\pi^*) - L_\dr(w^{\hat \pi_\dr}, v, \hat \pi_\dr; \Delta^*, \Theta^*)  \\
& \leq J(\pi^*) - \min_{(\Delta, \Theta) \in \ci^0_{\alpha_0}\times \ci^1_{\alpha_1}} \min_{(w,v) \in \ci_{\alpha_\mis, \alpha_\vf}(\Delta, \Theta, \hat \pi_\dr)} L_\dr(w, v, \hat \pi_\dr; \Delta, \Theta)  \\
& \leq J(\pi^*) - \min_{(\Delta, \Theta) \in \ci^0_{\alpha_0}\times \ci^1_{\alpha_1}} \min_{(w,v) \in \ci_{\alpha_\mis, \alpha_\vf}(\Delta, \Theta, \hat \pi_\dr)} \hat L_\dr(w, v, \hat \pi_\dr; \Delta, \Theta) + \hat \epsilon_L  \\
& \leq J(\pi^*) - \min_{(\Delta, \Theta) \in \ci^0_{\alpha_0}\times \ci^1_{\alpha_1}} \min_{(w,v) \in \ci_{\alpha_\mis, \alpha_\vf}(\Delta, \Theta, \pi^*)} \hat L_\dr(w, v, \pi^*; \Delta, \Theta) + \hat \epsilon_L,
\$
where we use Assumption \ref{ass:iv-sl-res} that $(\Delta^*, \Theta^*) \in \ci^0_{\alpha_0}\times \ci^1_{\alpha_1}$ with probability at least $1-\delta$ and Assumption \ref{ass:iv-mis-realizable} that $w^\pi \in \cW$ for any $\pi\in \Pi$ in the first inequality, we use Lemma \ref{lemma:iv-mis-hatl-close-dr} in the second inequality, and we use the optimality of $\hat \pi_\dr$ in the last inequality. Further, by Lemmas \ref{lemma:iv-mis-hatl-close-dr} and \ref{lemma:iv-mis-min-delta-close-dr}, we have 
\$
& J(\pi^*) - J(\hat\pi_\dr) \\
& \qquad \leq J(\pi^*) - \min_{(\Delta, \Theta) \in \ci^0_{\alpha_0}\times \ci^1_{\alpha_1}} \min_{(w,v) \in \ci_{\alpha_\mis, \alpha_\vf}(\Delta, \Theta, \pi^*)}  L_\dr(w, v, \pi^*; \Delta, \Theta) + 2 \hat \epsilon_L\\
& \qquad \leq J(\pi^*) - \min_{(w,v) \in \ci_{\alpha_\mis, \alpha_\vf}(\Delta^*, \Theta^*, \pi^*)}  L_\dr(w, v, \pi^*; \Delta^*, \Theta^*) + 2\hat \epsilon_L + \epsilon^*_L \\
& \qquad \leq \max_{(w,v) \in \ci_{\alpha_\mis, \alpha_\vf}(\Delta^*, \Theta^*, \pi^*)} \left | L_\dr(w^{\pi^*}, v, \pi^*; \Delta^*, \Theta^*)  - L_\dr(w, v, \pi^*; \Delta^*, \Theta^*) \right | + 2\hat \epsilon_L + \epsilon^*_L  \\
& \qquad \leq \max_{(w,v) \in \ci_{\alpha_\mis, \alpha_\vf}(\Delta^*, \Theta^*, \pi^*)} \left | \Phi_\mis^{\pi^*}(w, V^{\pi^*}; \Delta^*, \Theta^*) - \Phi_\mis^{\pi^*}(w, v; \Delta^*, \Theta^*) \right | + 2\hat \epsilon_L + \epsilon^*_L  \\
& \qquad \leq 2 \max_{(w,v) \in \ci_{\alpha_\mis, \alpha_\vf}(\Delta^*, \Theta^*, \pi^*)} \left | \Phi_\mis^{\pi^*}(w, v; \Delta^*, \Theta^*) \right | + 2 \hat \epsilon_L + \epsilon^*_L  \\
& \qquad \leq 2 \max_{w \in \ci^\mis_{\alpha_\mis} (\Delta^*, \Theta^*, \pi^*)} \max_{v\in \cV} \left | \Phi_\mis^{\pi^*}(w, v; \Delta^*, \Theta^*) \right | + 2 \hat \epsilon_L + \epsilon^*_L
\$
where in the third inequality, we use the fact that $J(\pi^*) = L_\dr(w^{\pi^*}, v, \pi^*; \Delta^*, \Theta^*)$; 
in the forth inequality, we use the following fact
\$
L_\dr(w,v,\pi^*; \Delta^*, \Theta^*) - L_\dr(w^{\pi^*},v,\pi^*; \Delta^*, \Theta^*) = -\Phi_\mis^{\pi^*}(w, V^{\pi^*}; \Delta^*, \Theta^*) + \Phi_\mis^{\pi^*}(w, v; \Delta^*, \Theta^*)
\$
for any $(w,v)\in \cW\times \cV$; in the fifth inequality, we use the fact that $V^{\pi^*}\in \cV$ by Assumption \ref{ass:iv-mis-realizable} and $V^{\pi^*}\in \ci_{\alpha_\vf}^\vf(\Delta^*, \Theta^*, \pi^*)$ with probability at least $1 - \delta$ by Lemma \ref{lemma:iv-v-pi-in-conf}. 
Now, by Lemma \ref{lemma:iv-w-in-conf-good} and the fact that $\cV$ is symmetric, we obtain that 
\#\label{eq:iv-dr-case1}
J(\pi^*) - J(\hat\pi_\dr) \leq c\cdot \frac{C_{\Delta^*}^2 C_{\Theta^*}^2 C_*}{1-\gamma} (\xi_0 + \xi_1) \sqrt{\frac{1}{NT \kappa} \pdim_{\cF_0,\cF_1,\cW,\cV,\Pi} \log\frac{NT }{\delta}}, 
\#
which concludes the proof of case (ii).

\vskip5pt
By combining \eqref{eq:iv-dr-case2} and \eqref{eq:iv-dr-case1}, we conclude the proof of the theorem. 
\end{proof}

\subsection{Proof of Theorem \ref{thm:iv-dr-spec}}\label{prf:thm:iv-dr-spec}
\begin{proof}
Recall that 
\$
\tilde v^\pi \in \argmin_{v\in \cV} \max_{w\in \cW} \Phi_\vf^\pi(v,w; \Delta^*, \Theta^*), \qquad \tilde w^\pi\in \argmin_{w\in \cW} \max_{v\in \cV} \Phi_\mis^\pi(w, v; \Delta^*, \Theta^*). 
\$
We split the proof into the following two parts. 

\noindent\textbf{Part (i).}
We first introduce the following lemmas. 
\begin{lemma}
\label{lemma:vinconf-spec}
Suppose $\alpha_\vf$ is defined in Lemma \ref{lemma:iv-v-pi-in-conf}
and $c / (NT)^2 \leq \delta \leq 1$. 
Then under Assumptions \ref{ass:spaces}\ref{ass:upper-bound-delta} and \ref{ass:ergodic}, with probability at least $1 - \delta$, it holds for any $\pi\in \Pi$ that $\tilde v^\pi \in \ci^\vf_{\alpha_\vf}(\Delta^*, \Theta^*, \pi)$. 
\end{lemma}
\begin{proof}
See \S\ref{prf:lemma:vinconf-spec} for a detailed proof. 
\end{proof}

\begin{lemma}
\label{lemma:vinconf-good-spec}
Suppose that $(\alpha_0, \alpha_1, \alpha_\vf)$ is defined in Assumption \ref{ass:iv-sl-res} and Lemma \ref{lemma:iv-v-pi-in-conf} and $c / (NT)^2 \leq \delta \leq 1$.
Then under Assumptions \ref{ass:iv-common}, \ref{ass:10}, \ref{ass:ergodic} and \ref{ass:iv-sl-res}, with probability at least $1 - \delta$, it holds for any policy $\pi\in \Pi$ and $v\in \cup_{(\Delta, \Theta) \in \ci^0_{\alpha_0} \times \ci^1_{\alpha_1}} \ci^\vf_{\alpha_\vf}(\Delta, \Theta, \pi)$ that
\$
\max_{g\in \cW} \Phi^\pi_\vf(v,g;\Delta^*, \Theta^*) & \leq  c\cdot \frac{C_{\Delta^*}^2 C_{\Theta^*}^2 C_*}{1-\gamma} (\xi_0 + \xi_1) \sqrt{\frac{1}{NT\kappa}\cdot \pdim_{\cF_0,\cF_1,\cW,\cV,\Pi} \cdot \log\frac{1}{\delta} \log(NT)}  \\
& \qquad + \max_{g\in \cW}  \Phi_\vf^\pi(\tilde v^\pi, g; \Delta^*, \Theta^*). 
\$
\end{lemma}
\begin{proof}
See \S\ref{prf:lemma:vinconf-good-spec} for a detailed proof. 
\end{proof}

By the definition of $L_{\dr}$, it holds that
\#\label{eq:0622-0}
& J(\pi^*) - J(\hat\pi_\dr)  \\
& \qquad = J(\pi^*) - L_\dr(w, V^{\hat \pi_\dr}, \hat \pi_\dr; \Delta^*, \Theta^*)  \\
& \qquad = J(\pi^*) - L_\dr(w, \tilde v^{\hat \pi_\dr}, \hat \pi_\dr; \Delta^*, \Theta^*) + L_\dr(w, \tilde v^{\hat \pi_\dr}, \hat \pi_\dr; \Delta^*, \Theta^*) - L_\dr(w, V^{\hat \pi_\dr}, \hat \pi_\dr; \Delta^*, \Theta^*). 
\#
Note that
\#\label{eq:0622-1}
& \left | L_\dr(w, \tilde v^{\hat \pi_\dr}, \hat \pi_\dr; \Delta^*, \Theta^*) - L_\dr(w, V^{\hat \pi_\dr}, \hat \pi_\dr; \Delta^*, \Theta^*) \right | \leq C_* C_{\Delta^*} C_{\Theta^*} \vareps^\cV_\vf. 
\#
In the meanwhile, by Assumption \ref{ass:iv-sl-res} and Lemma \ref{lemma:vinconf-spec}, it holds with probability at least $1 - \delta$ that
\#\label{eq:0622-2}
L_\dr(w, \tilde v^{\hat \pi_\dr}, \hat \pi_\dr; \Delta^*, \Theta^*) & \geq \min_{(\Delta, \Theta) \in \ci^0_{\alpha_0}\times \ci^1_{\alpha_1}} \min_{(w,v) \in \ci_{\alpha_\mis, \alpha_\vf}(\Delta, \Theta, \hat \pi_\dr)} L_\dr(w, v, \hat \pi_\dr; \Delta, \Theta)  \\
& \geq \min_{(\Delta, \Theta) \in \ci^0_{\alpha_0}\times \ci^1_{\alpha_1}} \min_{(w,v) \in \ci_{\alpha_\mis, \alpha_\vf}(\Delta, \Theta, \hat \pi_\dr)} \hat L_\dr(w, v, \hat \pi_\dr; \Delta, \Theta) - \hat \eps_L  \\
& \geq \min_{(\Delta, \Theta) \in \ci^0_{\alpha_0}\times \ci^1_{\alpha_1}} \min_{(w,v) \in \ci_{\alpha_\mis, \alpha_\vf}(\Delta, \Theta, \pi^*)} \hat L_\dr(w, v, \pi^*; \Delta, \Theta) - \hat \eps_L  \\
& \geq \min_{(\Delta, \Theta) \in \ci^0_{\alpha_0}\times \ci^1_{\alpha_1}} \min_{(w,v) \in \ci_{\alpha_\mis, \alpha_\vf}(\Delta, \Theta, \pi^*)} L_\dr(w, v, \pi^*; \Delta, \Theta) - 2 \hat \eps_L, 
\#
where in the second inequality, we use Lemma \ref{lemma:iv-mis-hatl-close-dr}; 
in the third inequality, we use the optimality of $\hat \pi_\dr$; 
in the forth inequality, we again use Lemma \ref{lemma:iv-mis-hatl-close-dr}. Now, by plugging \eqref{eq:0622-1} and \eqref{eq:0622-2} into \eqref{eq:0622-0}, we have
\#\label{eq:0622-3}
& J(\pi^*) - J(\hat\pi_\dr)  \\
& \qquad \leq J(\pi^*) - \min_{(\Delta, \Theta) \in \ci^0_{\alpha_0}\times \ci^1_{\alpha_1}} \min_{(w,v) \in \ci_{\alpha_\mis, \alpha_\vf}(\Delta, \Theta, \pi^*)} L_\dr(w, v, \pi^*; \Delta, \Theta) \\
& \qquad \qquad + 2 \hat \eps_L + C_* C_{\Delta^*} C_{\Theta^*} \vareps^\cV_\vf. 
\#
By combining Lemma \ref{lemma:iv-mis-min-delta-close-dr}  and \eqref{eq:0622-3}, we have
\#\label{eq:koko3-2-spec}
& J(\pi^*) - J(\hat\pi_\dr)  \\
& \quad \leq J(\pi^*) -  \min_{(w,v)\in \ci_{\alpha_\mis,\alpha_\vf}(\Delta^*, \Theta^*, \pi^*) } L_\dr(w, v, \pi^*; \Delta^*, \Theta^*) + 2 \hat \epsilon_L + \epsilon_L^* 
+ C_* C_{\Delta^*} C_{\Theta^*} \vareps^\cV_\vf  \\
& \quad = L_\dr(w, V^{\pi^*}, \pi^*; \Delta^*, \Theta^*) -  \min_{(w,v)\in \ci_{\alpha_\mis,\alpha_\vf}(\Delta^*, \Theta^*, \pi^*)} L_\dr(w, v, \pi^*; \Delta^*, \Theta^*)  \\
& \qquad \qquad + 2 \hat \epsilon_L + \epsilon_L^* 
+ C_* C_{\Delta^*} C_{\Theta^*} \vareps^\cV_\vf  \\
& \quad = \max_{(w,v)\in \ci_{\alpha_\mis,\alpha_\vf}(\Delta^*, \Theta^*, \pi^*)} \left | L_\dr(w, V^{\pi^*}, \pi^*; \Delta^*, \Theta^*) - L_\dr(w, v, \pi^*; \Delta^*, \Theta^*) \right |  \\
& \qquad \qquad + 2 \hat \epsilon_L + \epsilon_L^*
+ C_* C_{\Delta^*} C_{\Theta^*} \vareps^\cV_\vf  \\
& \quad = \max_{(w,v)\in \ci_{\alpha_\mis,\alpha_\vf}(\Delta^*, \Theta^*, \pi^*)} \left | \Phi_\vf^{\pi^*}(v, w^{\pi^*}; \Delta^*, \Theta^*) - \Phi_\vf^{\pi^*}(v, w; \Delta^*, \Theta^*) \right | \\
& \qquad \qquad + 2 \hat \epsilon_L + \epsilon_L^* + C_* C_{\Delta^*} C_{\Theta^*} \vareps^\cV_\vf, 
\#
where we use the following fact in the last equality, 
\$
L_\dr(w, V^{\pi^*}, \pi^*; \Delta^*, \Theta^*) - L_\dr(w, v, \pi^*; \Delta^*, \Theta^*) = \Phi_\vf^{\pi^*}(v, w^{\pi^*}; \Delta^*, \Theta^*) - \Phi_\vf^{\pi^*}(v, w; \Delta^*, \Theta^*). 
\$
We upper bound the first term on the RHS of \eqref{eq:koko3-2-spec} as follows,
\#\label{eq:456374637}
& \max_{(w,v)\in \ci_{\alpha_\mis,\alpha_\vf}(\Delta^*, \Theta^*, \pi^*)} \left | \Phi_\vf^{\pi^*}(v, w^{\pi^*}; \Delta^*, \Theta^*) - \Phi_\vf^{\pi^*}(v, w; \Delta^*, \Theta^*) \right|  \\
& \quad \leq \max_{v\in \ci^\vf_{\alpha_\vf}(\Delta^*, \Theta^*, \pi^*)} \max_{w\in \cW} \left | \Phi_\vf^{\pi^*}(v, \tilde w^{\pi^*}; \Delta^*, \Theta^*) - \Phi_\vf^{\pi^*}(v, w; \Delta^*, \Theta^*) \right|   \\
& \qquad\quad + \max_{v\in \cV} \left | \Phi_\vf^{\pi^*}(v, w^{\pi^*}; \Delta^*, \Theta^*) - \Phi_\vf^{\pi^*}(v, \tilde w^{\pi^*}; \Delta^*, \Theta^*) \right|  \\
& \quad \leq 2 \max_{v\in \ci^\vf_{\alpha_\vf}(\Delta^*, \Theta^*, \pi^*)} \max_{w\in \cW} \left |\Phi_\vf^{\pi^*}(v, w; \Delta^*, \Theta^*) \right| \\
& \qquad \quad + \max_{v\in \cV} \left | \Phi_\vf^{\pi^*}(v, w^{\pi^*}; \Delta^*, \Theta^*) - \Phi_\vf^{\pi^*}(v, \tilde w^{\pi^*}; \Delta^*, \Theta^*) \right|,
\#
where in the first inequality, we use triangle inequality; in the second inequality, we use the definition of $\tilde w^{\pi^*}$ that $\tilde w^{\pi^*}\in \cW$. By Lemma \ref{lemma:vinconf-good-spec} and Assumption \ref{ass:model-spec}, we obtain from \eqref{eq:456374637} that 
\#\label{eq:87343846}
& \max_{(w,v)\in \ci_{\alpha_\mis,\alpha_\vf}(\Delta^*, \Theta^*, \pi^*)} \left | \Phi_\vf^{\pi^*}(v, w^{\pi^*}; \Delta^*, \Theta^*) - \Phi_\vf^{\pi^*}(v, w; \Delta^*, \Theta^*) \right |  \\
& \qquad \leq c\cdot \frac{C_{\Delta^*}^2 C_{\Theta^*}^2 C_*}{1-\gamma} (\xi_0 + \xi_1) \sqrt{\frac{1}{NT\kappa}\cdot \pdim_{\cF_0,\cF_1,\cW,\cV,\Pi} \cdot \log\frac{1}{\delta} \log(NT)} \\
& \qquad \qquad + 2 \max_{g\in \cW}  \Phi_\vf^{\pi^*}(\tilde v^{\pi^*}, g; \Delta^*, \Theta^*) \\
& \qquad \qquad + C_{\Delta^*} C_{\Theta^*} \vareps^\cW_\vf/(1-\gamma).  
\#
Also, we have
\#\label{eq:3456789765}
& \max_{g\in \cW}  \Phi_\vf^{\pi^*}(\tilde v^{\pi^*}, g; \Delta^*, \Theta^*)  \\
& \qquad = \max_{g\in \cW}  \Phi_\vf^{\pi^*}(\tilde v^{\pi^*}, g; \Delta^*, \Theta^*) - \max_{g\in \cW}  \Phi_\vf^{\pi^*}(V^{\pi^*}, g; \Delta^*, \Theta^*) + \max_{g\in \cW}  \Phi_\vf^{\pi^*}(V^{\pi^*}, g; \Delta^*, \Theta^*)  \\
& \qquad= \max_{g\in \cW}  \Phi_\vf^{\pi^*}(\tilde v^{\pi^*}, g; \Delta^*, \Theta^*) - \max_{g\in \cW}  \Phi_\vf^{\pi^*}(V^{\pi^*}, g; \Delta^*, \Theta^*)  \\
& \qquad= \max_{g\in \cW} \left | \Phi_\vf^{\pi^*}(\tilde v^{\pi^*}, g; \Delta^*, \Theta^*) -  \Phi_\vf^{\pi^*}(V^{\pi^*}, g; \Delta^*, \Theta^*) \right |  \\
& \qquad \leq C_* C_{\Delta^*} C_{\Theta^*} \vareps^\cV_\vf, 
\#
where we use Assumption \ref{ass:model-spec} in the last inequality. 
By plugging \eqref{eq:87343846} and \eqref{eq:3456789765} into \eqref{eq:koko3-2-spec}, we have
\#\label{eq:iv-dr-case2-2-spec}
J(\pi^*) - J(\hat\pi_\dr) & \leq c\cdot \frac{C_{\Delta^*}^2 C_{\Theta^*}^2 C_*}{1-\gamma} (\xi_0 + \xi_1) \sqrt{\frac{1}{NT \kappa} \pdim_{\cF_0,\cF_1,\cW,\cV,\Pi} \log\frac{NT }{\delta}} \\
& \qquad + 3 C_{\Delta^*} C_{\Theta^*} \left(C_* \vareps^\cV_\vf + \vareps^\cW_\vf/(1-\gamma)\right), 
\#
where we plug in the definition of $\epsilon_L$ and $\epsilon_L^*$ in the last inequality.

\vskip5pt
\noindent\textbf{Part (ii).}
We first introduce the following lemmas. 
\begin{lemma}
\label{lemma:iv-w-pi-in-conf-spec}
Suppose $\alpha_\mis$ is defined in Lemma \ref{lemma:iv-w-pi-in-conf} and 
and $c / (NT)^2 \leq \delta \leq 1$. 
Then under Assumptions \ref{ass:upper-bound-delta} and \ref{ass:ergodic}, with probability at least $1 - \delta$, it holds for any $\pi \in \Pi$ that $\tilde w^\pi \in \ci^\mis_{\alpha_\mis}(\Delta^*, \Theta^*, \pi)$. 
\end{lemma}
\begin{proof}
See \S\ref{prf:lemma:iv-w-pi-in-conf-spec} for a detailed proof.
\end{proof}

\begin{lemma}
\label{lemma:iv-w-in-conf-good-spec}
Suppose that $(\alpha_0, \alpha_1, \alpha_\mis)$ is defined in Assumption \ref{ass:iv-sl-res} and Lemma \ref{lemma:iv-w-pi-in-conf}, and $c / (NT)^2 \leq \delta \leq 1$. 
Then under Assumptions \ref{ass:iv-common}, \ref{ass:10}, \ref{ass:ergodic}, and \ref{ass:iv-sl-res}, with probability at least $1 - \delta$, it holds for any $\pi\in \Pi$ and $w \in \cup_{(\Delta, \Theta) \in \ci^0_{\alpha_0} \times \ci^1_{\alpha_1}} \ci^\mis_{\alpha_\mis}(\Delta, \Theta, \pi)$ that
\$
\max_{f\in \cV} \Phi^\pi_\mis(w,f;\Delta^*, \Theta^*) & \leq c\cdot \frac{ C_{\Delta^*}^2 C_{\Theta^*}^2 C_*}{1-\gamma} (\xi_0 + \xi_1) \sqrt{\frac{1}{NT \kappa}\cdot \pdim_{\cF_0,\cF_1,\cW,\cV,\Pi} \cdot \log\frac{1}{\delta} \log(NT) } \\
& \qquad + \max_{f\in \cV} \Phi_\mis^\pi(\tilde w^\pi, f; \Delta^*, \Theta^*). 
\$
\end{lemma}
\begin{proof}
See \S\ref{prf:lemma:iv-w-in-conf-good-spec} for a detailed proof. 
\end{proof}

By the definition of $L_\dr$, we have
\#\label{eq:0623-1}
& J(\pi^*) - J(\hat \pi_\dr)  \\
& \qquad = J(\pi^*) - L_\dr(w^{\hat \pi_\dr}, v, \hat \pi_\dr; \Delta^*, \Theta^*)   \\
& \qquad = J(\pi^*) - L_\dr(\tilde w^{\hat \pi_\dr}, v, \hat \pi_\dr; \Delta^*, \Theta^*) + L_\dr(\tilde w^{\hat \pi_\dr}, v, \hat \pi_\dr; \Delta^*, \Theta^*) - L_\dr(w^{\hat \pi_\dr}, v, \hat \pi_\dr; \Delta^*, \Theta^*). 
\#
By Assumption \ref{ass:model-spec}, we have
\#\label{eq:0623-2}
\left | L_\dr(\tilde w^{\hat \pi_\dr}, v, \hat \pi_\dr; \Delta^*, \Theta^*) - L_\dr(w^{\hat \pi_\dr}, v, \hat \pi_\dr; \Delta^*, \Theta^*)\right | \leq C_{\Delta^*} C_{\Theta^*} \vareps^\cW_\mis / (1 - \gamma). 
\#
In the meanwhile, by Assumption \ref{ass:iv-sl-res} and Lemma \ref{lemma:iv-w-pi-in-conf-spec}, it holds with probability at least $1 - \delta$ that
\#\label{eq:0623-3}
L_\dr(\tilde w^{\hat \pi_\dr}, v, \hat \pi_\dr; \Delta^*, \Theta^*) & \geq \min_{(\Delta, \Theta) \in \ci^0_{\alpha_0}\times \ci^1_{\alpha_1}} \min_{(w,v) \in \ci_{\alpha_\mis, \alpha_\vf}(\Delta, \Theta, \hat \pi_\dr)} L_\dr(w, v, \hat \pi_\dr; \Delta, \Theta)  \\
& \geq \min_{(\Delta, \Theta) \in \ci^0_{\alpha_0}\times \ci^1_{\alpha_1}} \min_{(w,v) \in \ci_{\alpha_\mis, \alpha_\vf}(\Delta, \Theta, \hat \pi_\dr)} \hat L_\dr(w, v, \hat \pi_\dr; \Delta, \Theta) - \hat \eps_L  \\
& \geq \min_{(\Delta, \Theta) \in \ci^0_{\alpha_0}\times \ci^1_{\alpha_1}} \min_{(w,v) \in \ci_{\alpha_\mis, \alpha_\vf}(\Delta, \Theta, \pi^*)} \hat L_\dr(w, v, \pi^*; \Delta, \Theta) - \hat \eps_L  \\
& \geq \min_{(\Delta, \Theta) \in \ci^0_{\alpha_0}\times \ci^1_{\alpha_1}} \min_{(w,v) \in \ci_{\alpha_\mis, \alpha_\vf}(\Delta, \Theta, \pi^*)} L_\dr(w, v, \pi^*; \Delta, \Theta) - 2\hat \eps_L,
\#
where in the second inequality, we use Lemma \ref{lemma:iv-mis-hatl-close-dr}; in the third inequality, we use the optimality of $\hat \pi_\dr$; in the forth inequality, we again use Lemma \ref{lemma:iv-mis-hatl-close-dr}. Further, combining Lemma \ref{lemma:iv-mis-min-delta-close-dr} and \eqref{eq:0623-3}, we have 
\#\label{eq:0623-4}
L_\dr(\tilde w^{\hat \pi_\dr}, v, \hat \pi_\dr; \Delta^*, \Theta^*) \geq \min_{(w,v) \in \ci_{\alpha_\mis, \alpha_\vf}(\Delta^*, \Theta^*, \pi^*)} L_\dr(w, v, \pi^*; \Delta^*, \Theta^*) - 2\hat \eps_L - \eps_L^*.
\#
Now, by plugging \eqref{eq:0623-2} and \eqref{eq:0623-4} into \eqref{eq:0623-1}, we have
\#\label{eq:0623-5}
& J(\pi^*) - J(\hat \pi_\dr)  \\
& \qquad \leq J(\pi^*) - \min_{(w,v) \in \ci_{\alpha_\mis, \alpha_\vf}(\Delta^*, \Theta^*, \pi^*)} L_\dr(w, v, \pi^*; \Delta^*, \Theta^*)  \\
& \qquad \qquad + 2\hat \eps_L + \eps_L^* + C_{\Delta^*} C_{\Theta^*} \vareps^\cW_\mis / (1 - \gamma) \\
& \qquad \leq \max_{(w,v) \in \ci_{\alpha_\mis, \alpha_\vf}(\Delta^*, \Theta^*, \pi^*)} \left | L_\dr(w^{\pi^*}, v, \pi^*; \Delta^*, \Theta^*)  - L_\dr(w, v, \pi^*; \Delta^*, \Theta^*) \right |  \\
& \qquad \qquad + 2\hat \epsilon_L + \epsilon^*_L  +  C_{\Delta^*} C_{\Theta^*} \vareps^\cW_\mis / (1 - \gamma)  \\
& \qquad = \max_{(w,v) \in \ci_{\alpha_\mis, \alpha_\vf}(\Delta^*, \Theta^*, \pi^*)} \left | \Phi_\mis^{\pi^*}(w, V^{\pi^*}; \Delta^*, \Theta^*) - \Phi_\mis^{\pi^*}(w, v; \Delta^*, \Theta^*) \right | \\
& \qquad \qquad + 2\hat \epsilon_L + \epsilon^*_L + C_{\Delta^*} C_{\Theta^*} \vareps^\cW_\mis / (1 - \gamma), 
\#
where in the second inequality, we use the fact that $J(\pi^*) = L_\dr(w^{\pi^*}, v, \pi^*; \Delta^*, \Theta^*)$; 
in the last equality, we use the following fact
\$
L_\dr(w,v,\pi^*; \Delta^*, \Theta^*) - L_\dr(w^{\pi^*},v,\pi^*; \Delta^*, \Theta^*) = -\Phi_\mis^{\pi^*}(w, V^{\pi^*}; \Delta^*, \Theta^*) + \Phi_\mis^{\pi^*}(w, v; \Delta^*, \Theta^*)
\$
for any $(w,v)\in \cW\times \cV$. 
We upper bound the first term on the RHS of \eqref{eq:0623-5} as follows, 
\#\label{eq:0623-6}
& \max_{(w,v) \in \ci_{\alpha_\mis, \alpha_\vf}(\Delta^*, \Theta^*, \pi^*)} \left | \Phi_\mis^{\pi^*}(w, V^{\pi^*}; \Delta^*, \Theta^*) - \Phi_\mis^{\pi^*}(w, v; \Delta^*, \Theta^*) \right |  \\
& \qquad \leq \max_{w \in \ci_{\alpha_\mis}^\mis(\Delta^*, \Theta^*, \pi^*)} \max_{v\in \cV} \left | \Phi_\mis^{\pi^*}(w, v; \Delta^*, \Theta^*) - \Phi_\mis^{\pi^*}(w, \tilde v^{\pi^*}; \Delta^*, \Theta^*) \right |   \\
& \qquad \qquad +  \max_{w\in \cW}  \left | \Phi_\mis^{\pi^*}(w, V^{\pi^*}; \Delta^*, \Theta^*) - \Phi_\mis^{\pi^*}(w, \tilde v^{\pi^*}; \Delta^*, \Theta^*) \right |  \\
& \qquad \leq 2 \max_{w \in \ci_{\alpha_\mis}^\mis(\Delta^*, \Theta^*, \pi^*)} \max_{v\in \cV} \left | \Phi_\mis^{\pi^*}(w, v; \Delta^*, \Theta^*) \right |  \\
& \qquad \qquad +  \max_{w\in \cW}  \left | \Phi_\mis^{\pi^*}(w, V^{\pi^*}; \Delta^*, \Theta^*) - \Phi_\mis^{\pi^*}(w, \tilde v^{\pi^*}; \Delta^*, \Theta^*) \right |  \\
& \qquad \leq c\cdot \frac{ C_{\Delta^*}^2 C_{\Theta^*}^2 C_*}{1-\gamma} (\xi_0 + \xi_1) \sqrt{\frac{1}{NT \kappa}\cdot \pdim_{\cF_0,\cF_1,\cW,\cV,\Pi} \cdot \log\frac{1}{\delta} \log(NT) }  \\
& \qquad \qquad + 2\max_{f\in \cV} \Phi_\mis^{\pi^*}(\tilde w^{\pi^*}, f; \Delta^*, \Theta^*) + C_* C_{\Delta^*} C_{\Theta^*} \vareps_\mis^\cV, 
\#
where in the first inequality, we use triangle inequality; in the second inequality, we use the definition of $\tilde v^{\pi^*}$ that $\tilde v^{\pi^*}\in \cV$; in the last inequality, we use Lemma \ref{lemma:iv-w-in-conf-good-spec} and Assumption \ref{ass:model-spec}. 
In the meanwhile, by Assumption \ref{ass:model-spec}, we have
\#\label{eq:0623-7}
& \max_{f\in \cV} \Phi_\mis^{\pi^*}(\tilde w^\pi, f; \Delta^*, \Theta^*)  \\
& \qquad \leq \max_{f\in \cV} \Phi_\mis^{\pi^*}(\tilde w^{\pi^*}, f; \Delta^*, \Theta^*) - \max_{f\in \cV} \Phi_\mis^{\pi^*}(w^{\pi^*}, f; \Delta^*, \Theta^*) + \max_{f\in \cV} \Phi_\mis^{\pi^*}(w^{\pi^*}, f; \Delta^*, \Theta^*)  \\
& \qquad = \max_{f\in \cV} \Phi_\mis^{\pi^*}(\tilde w^{\pi^*}, f; \Delta^*, \Theta^*) - \max_{f\in \cV} \Phi_\mis^{\pi^*}(w^{\pi^*}, f; \Delta^*, \Theta^*)  \\
& \qquad \leq \max_{f\in \cV} \left | \Phi_\mis^{\pi^*}(\tilde w^{\pi^*}, f; \Delta^*, \Theta^*) - \Phi_\mis^{\pi^*}(w^{\pi^*}, f; \Delta^*, \Theta^*) \right |  \\
& \qquad \leq C_{\Delta^*} C_{\Theta^*} \vareps^\cW_\mis / (1 -\gamma). 
\#
Now, by plugging \eqref{eq:0623-6} and \eqref{eq:0623-7} into \eqref{eq:0623-5}, we have
\#\label{eq:iv-dr-case1-2-spec}
J(\pi^*) - J(\hat\pi_\dr) & \leq c\cdot \frac{C_{\Delta^*}^2 C_{\Theta^*}^2 C_*}{1-\gamma} (\xi_0 + \xi_1) \sqrt{\frac{1}{NT \kappa} \pdim_{\cF_0,\cF_1,\cW,\cV,\Pi} \log\frac{NT }{\delta}} \\
& \qquad + 3 C_{\Delta^*} C_{\Theta^*} \left(C_* \vareps^\cV_\mis + \vareps^\cW_\mis/(1-\gamma)\right). 
\#

\vskip5pt
By combining \eqref{eq:iv-dr-case2-2-spec} and \eqref{eq:iv-dr-case1-2-spec}, we have
\$
J(\pi^*) - J(\hat\pi_\dr) & \leq c\cdot \frac{C_{\Delta^*}^2 C_{\Theta^*}^2 C_*}{1-\gamma} (\xi_0 + \xi_1) \sqrt{\frac{1}{NT \kappa} \pdim_{\cF_0,\cF_1,\cW,\cV,\Pi} \log\frac{NT }{\delta}}  \\
& \qquad + 3 C_{\Delta^*} C_{\Theta^*} \min \left\{ C_* \vareps^\cV_\vf + \vareps^\cW_\vf/(1-\gamma), ~C_* \vareps^\cV_\mis + \vareps^\cW_\mis/(1-\gamma) \right \}
\$
which conclude the proof. 
\end{proof}

\subsection{Proof of Lemma \ref{lemma:iv-mis-hatl-close-dr}}
\label{prf:lemma:iv-mis-hatl-close-dr}
\begin{proof}
By Theorem \ref{thm:hoeffding-mixing}, with probability at least $1 - \delta$, it holds for any $(w, v, \Delta, \Theta, \pi)\in \cW \times \cV \times \cF_0 \times \cF_1 \times \Pi$ that
\$
& \left | L_\dr(w, v, \pi; \Delta, \Theta) - \hat L_\dr(w, v, \pi; \Delta, \Theta) \right | \\
& \qquad \leq c\cdot \frac{C_{\Delta^*} C_{\Theta^*} C_*}{1-\gamma} \sqrt{\frac{1}{NT \kappa}\pdim_{\cF_0,\cF_1,\cW,\cV,\Pi} \log\frac{1}{\delta} \log(NT)}, 
\$
which concludes the proof of the lemma. 
\end{proof}

\subsection{Proof of Lemma \ref{lemma:iv-mis-min-delta-close-dr}}
\label{prf:lemma:iv-mis-min-delta-close-dr}
\begin{proof}
With a slight abuse of notations, we denote by
\$
& (w_0,v_0) \in \argmin_{(w,v)\in \ci_{\alpha_\mis,\alpha_\vf}(\Delta^*, \Theta^*, \pi^*)} L_\dr(w, v, \pi^*; \Delta^*, \Theta^*), \\
& (w_1,v_1) \in \argmin_{(w,v)\in \ci_{\alpha_\mis,\alpha_\vf}(\Delta, \Theta, \pi^*)} L_\dr(w, v, \pi^*; \Delta, \Theta). 
\$
Then we have
\#\label{eq:iv-kkk1-dr}
& \left | \min_{(w,v)\in \ci_{\alpha_\mis,\alpha_\vf}(\Delta^*, \Theta^*, \pi^*)} L_\dr(w, v, \pi^*; \Delta^*, \Theta^*) - \min_{(w,v)\in \ci_{\alpha_\mis,\alpha_\vf}(\Delta, \Theta, \pi^*)} L_\dr(w, v, \pi^*; \Delta, \Theta) \right | \\
& \qquad = \left | L_\dr(w_0, v_0, \pi^*; \Delta^*, \Theta^*) - L_\dr(w_1, v_1, \pi^*; \Delta, \Theta) \right|  \\
& \qquad \leq \left | L_\dr(w_0, v_0, \pi^*; \Delta^*, \Theta^*) - L_\dr(w^{\pi^*}, v_0, \pi^*; \Delta^*, \Theta^*) \right |  \\
& \qquad \qquad + \left | L_\dr(w^\pi, v_1, \pi^*; \Delta^*, \Theta^*) - L_\dr(w_1, v_1, \pi^*; \Delta^*, \Theta^*) \right |  \\
& \qquad \qquad + \left |L_\dr(w_1, v_1, \pi; \Delta^*, \Theta^*) - L_\dr(w_1, v_1, \pi; \Delta, \Theta) \right |  \\
& \qquad = \underbrace{\left | \Phi_\mis^{\pi^*}(w_0, V^{\pi^*}; \Delta^*, \Theta^*)-\Phi_\mis^{\pi^*}(w_0, v_0; \Delta^*, \Theta^*) \right|}_{\text{Term (I)}} \\
& \qquad \qquad + \underbrace{\left| \Phi_\mis^{\pi^*}(w_1, V^{\pi^*}; \Delta^*, \Theta^*) - \Phi_\mis^{\pi^*}(w_1, v_1; \Delta^*, \Theta^*) \right|}_{\text{Term (II)}} \\
& \qquad \qquad + \underbrace{\left | L_\dr(w_1, v_1, \pi^*; \Delta^*, \Theta^*) - L_\dr(w_1, v_1, \pi^*; \Delta, \Theta) \right|}_{\text{Term (III)}},  
\#
where in the first inequality, we use triangle inequality and the fact that $L_\dr(w^{\pi^*}, v, \pi^*; \Delta^*, \Theta^*) = J(\pi^*)$ for any function $v$; while in the last equality, we use the following equality for any $(w,v)$, 
\$
L_\dr(w,v,\pi^*; \Delta^*, \Theta^*) - L_\dr(w^{\pi^*},v,\pi^*; \Delta^*, \Theta^*) = -\Phi_\mis^{\pi^*}(w, V^{\pi^*}; \Delta^*, \Theta^*) + \Phi_\mis^{\pi^*}(w, v; \Delta^*, \Theta^*). 
\$
We upper bound terms (I), (II), and (III) on the RHS of \eqref{eq:iv-kkk1-dr}, respectively. 

\vskip5pt
\noindent\textbf{Upper Bounding Term (I).}
Note that with probability at least $1 - \delta$, we have
\$
\left | \Phi_\mis^{\pi^*}(w_0, V^{\pi^*}; \Delta^*, \Theta^*) \right| & \leq \max_{f\in \cV} \left | \Phi_\mis^{\pi^*}(w_0, f; \Delta^*, \Theta^*) \right|  \\
& = \max_{f\in \cV} \max \left\{  \Phi_\mis^{\pi^*}(w_0, f; \Delta^*, \Theta^*), 
-\Phi_\mis^{\pi^*}(w_0, f; \Delta^*, \Theta^*)  \right \}  \\
& = \max_{f\in \cV} \max \left\{  \Phi_\mis^{\pi^*}(w_0, f; \Delta^*, \Theta^*), 
\Phi_\mis^{\pi^*}(w_0, -f; \Delta^*, \Theta^*)  \right \}  \\
& = \max_{f\in \cV} \Phi_\mis^{\pi^*}(w_0, f; \Delta^*, \Theta^*)  \\
& \leq c\cdot \frac{ C_{\Delta^*}^2 C_{\Theta^*}^2 C_*}{1-\gamma} (\xi_0 + \xi_1) \sqrt{\frac{1}{NT \kappa} \pdim_{\cF_0,\cF_1,\cW,\cV} \log\frac{1}{\delta} \log(NT)}, 
\$
where in the first inequality, we use the fact that $V^{\pi^*} \in \cV$; in the third equality, we use the fact that $\cV$ is symmetric; in the last inequality, by noting that $w_0 \in \cup_{(\Delta, \Theta) \in \ci^0_{\alpha_0} \times \ci^1_{\alpha_1}} \ci^\pi_{\alpha_\mis}(\Delta, \Theta)$, we use Lemma \ref{lemma:iv-w-in-conf-good}. Similarly, we have 
\$
\left | \Phi_\mis^{\pi^*}(w_0, v_0; \Delta^*, \Theta^*) \right| \leq c\cdot \frac{ C_{\Delta^*}^2 C_{\Theta^*}^2 C_*}{1-\gamma} (\xi_0 + \xi_1) \sqrt{\frac{1}{NT \kappa} \pdim_{\cF_0,\cF_1,\cW,\cV} \log\frac{1}{\delta} \log(NT)}, 
\$
which implies that
\#\label{eq:iv-kkk2-dr}
\text{Term (I)} \leq c\cdot \frac{ C_{\Delta^*}^2 C_{\Theta^*}^2 C_*}{1-\gamma} (\xi_0 + \xi_1) \sqrt{\frac{1}{NT \kappa} \pdim_{\cF_0,\cF_1,\cW,\cV} \log\frac{1}{\delta} \log(NT)}. 
\#

\vskip5pt
\noindent\textbf{Upper Bounding Term (II).}
Similar to \eqref{eq:iv-kkk2-dr}, with probability at least $1 - \delta$, we have
\#\label{eq:iv-kkk3-dr}
\text{Term (II)} \leq c\cdot \frac{ C_{\Delta^*}^2 C_{\Theta^*}^2 C_*}{1-\gamma} (\xi_0 + \xi_1) \sqrt{\frac{1}{NT \kappa} \pdim_{\cF_0,\cF_1,\cW,\cV} \log\frac{1}{\delta} \log(NT)}. 
\#

\vskip5pt
\noindent\textbf{Upper Bounding Term (III).}
Note that with probability at least $1 - \delta$, it holds for any $(\Delta, \Theta) \in \ci^0_{\alpha_0} \times \ci^1_{\alpha_1}$ that
\#\label{eq:iv-kkk4-dr}
& \text{Term (III)}  \\
& = \EE \left[\frac{1}{T}\sum_{t = 0}^{T-1} \left( \frac{Z_t^\top A_t \pi^*(A_t\given S_t)}{\Delta^*(S_t,A_t) \Theta^*(S_t,Z_t)} - \frac{Z_t^\top A_t \pi^*(A_t\given S_t)}{\Delta(S_t,A_t) \Theta(S_t,Z_t)} \right) w_1(S_t) \left(R_t + \gamma v_1(S_{t+1}) - v_1(S_t) \right) \right]  \\
& \leq \frac{C_{\Delta^*} C_{\Theta^*} C_*}{1-\gamma} \left(\xi_0 C_{\Theta^*} \sqrt{\frac{C_{\Delta^*}}{NT \kappa} \pdim_{\cF_0} \log\frac{1}{\delta} \log(NT)} + \xi_1 C_{\Delta^*} \sqrt{\frac{C_{\Theta^*}}{NT \kappa} \pdim_{\cF_1} \log\frac{1}{\delta} \log(NT)} \right),
\#
where we use triangle inequality and Assumption \ref{ass:iv-sl-res} in the last inequality. 

\vskip5pt
Now, by plugging \eqref{eq:iv-kkk2-dr}, \eqref{eq:iv-kkk3-dr}, and \eqref{eq:iv-kkk4-dr} into \eqref{eq:iv-kkk1-dr}, we conclude the proof of the lemma. 
\end{proof}

\subsection{Proof of Lemma \ref{lemma:vinconf-spec}}\label{prf:lemma:vinconf-spec}

\begin{proof}
By the definition of $\tilde v^\pi$ in \eqref{eq:v-tilde-def}, we know that $\tilde v^\pi \in \cV$. Thus, to show that $\tilde v^\pi \in \ci^\vf_{\alpha_\vf}(\Delta^*, \Theta^*, \pi)$ with a high probability, it suffices to show that 
\#\label{eq:iv-vpi-in-ci-wtp-spec}
\max_{g\in \cW} \hat \Phi_\vf^\pi(\tilde v^\pi, g; \Delta^*, \Theta^*) - \max_{g\in \cW} \hat \Phi_\vf^\pi(\hat v_{\Delta^*, \Theta^*}^\pi, g; \Delta^*, \Theta^*) \leq \alpha_\vf. 
\#
In the follows, we show that \eqref{eq:iv-vpi-in-ci-wtp-spec} holds with a high probability. 
For the simplicity of notations, we denote by $\Phi_\vf^\pi(v, g; *) = \Phi_\vf^\pi(v, g; \Delta^*, \Theta^*)$ and $\hat v^\pi_* = \hat v^\pi_{\Delta^*, \Theta^*}$ for any $(\pi,v,g)$.
Note that 
\#\label{eq:iv-ee1-vf-spec}
& \max_{g \in \cW} \hat \Phi_\vf^\pi(\tilde v^\pi, g; *) - \max_{g \in \cW} \hat \Phi_\vf^\pi(\hat v^\pi_*, g; *)  \\
& \qquad = \max_{g \in \cW} \hat \Phi_\vf^\pi(\tilde v^\pi, g; *) - \max_{g \in \cW}  \Phi_\vf^\pi(\tilde v^\pi, g; *)  + \max_{g \in \cW}  \Phi_\vf^\pi(\tilde v^\pi, g; *) - \max_{g \in \cW}  \Phi_\vf^\pi(\hat v^\pi_*, g; *)  \\
& \qquad \qquad + \max_{g \in \cW} \Phi_\vf^\pi(\hat v^\pi_*, g; *) - \max_{g \in \cW} \hat \Phi_\vf^\pi(\hat v^\pi_*, g; *)  \\
& \qquad \leq \max_{g \in \cW} \hat \Phi_\vf^\pi(\tilde v^\pi, g; *) - \max_{g \in \cW}  \Phi_\vf^\pi(\tilde v^\pi, g; *) + \max_{g \in \cW} \Phi_\vf^\pi(\hat v^\pi_*, g; *) - \max_{g \in \cW} \hat \Phi_\vf^\pi(\hat v^\pi_*, g; *)  \\
& \qquad \leq 2 \max_{v\in \cV} \left | \max_{g \in \cW} \hat \Phi_\vf^\pi(v, g; *) - \max_{g \in \cW}  \Phi_\vf^\pi(v, g; *) \right |  \\
& \qquad \leq 2 \max_{v\in \cV} \max_{g \in \cW} \left |  \hat \Phi_\vf^\pi(v, g; *) - \Phi_\vf^\pi(v, g; *) \right |,
\#
where in the first inequality, we use the fact that $\max_{g \in \cW}  \Phi_\vf^\pi(\tilde v^\pi, g; *) \leq  \max_{g \in \cW}  \Phi_\vf^\pi(v, g; *)$ for any $v$ due to the definition of $\tilde v^\pi$ in \eqref{eq:v-tilde-def}; 
while in the second inequality, we use the fact that $\tilde v^\pi, \hat v^\pi_* \in \cV$. 
In the meanwhile, by Theorem \ref{thm:hoeffding-mixing}, with probability at least $1 - \delta$, it holds for any $(\pi,v,g)\in \Pi \times \cV \times \cW$ that
\#\label{eq:iv-ee2-vf-spec}
\left |  \hat \Phi_\vf^\pi(v, g; *) - \Phi_\vf^\pi(v, g; *) \right | \leq c\cdot \frac{C_{\Delta^*} C_{\Theta^*} C_*}{1-\gamma} \sqrt{\frac{\pdim_{\cW,\cV,\Pi}}{NT\kappa} \cdot \log\frac{1}{\delta} \log(NT) }, 
\#
where we use Assumption \ref{ass:upper-bound-delta} and $\|g\|_\infty \leq C_*$ for any $g \in \cW$. 
Now, combining \eqref{eq:iv-ee1-vf-spec} and \eqref{eq:iv-ee2-vf-spec}, with probability at least $1 - \delta$, we have
\$
& \max_{g \in \cW} \hat \Phi_\vf^\pi(\tilde v^\pi, g; *) - \max_{g \in \cW} \hat \Phi_\vf^\pi(\hat v^\pi_*, g; *) \leq c\cdot \frac{C_{\Delta^*} C_{\Theta^*} C_*}{1-\gamma} \sqrt{\frac{\pdim_{\cW,\cV,\Pi}}{NT\kappa} \cdot \log\frac{1}{\delta} \log(NT) } = \alpha_\vf, 
\$
which implies that $\tilde v^\pi \in \ci^\vf_{\alpha_\vf}(\Delta^*, \Theta^*, \pi)$ for any $\pi\in \Pi$. This concludes the proof of the lemma. 
\end{proof}

\subsection{Proof of Lemma \ref{lemma:vinconf-good-spec}}\label{prf:lemma:vinconf-good-spec}
\begin{proof}
Since $v \in \cup_{(\Delta, \Theta) \in \ci^0_{\alpha_0} \times \ci^1_{\alpha_1}} \ci^\vf_{\alpha_\vf}(\Delta, \Theta, \pi)$, there exists a pair $(\tilde \Delta, \tilde\Theta) \in \ci^0_{\alpha_0} \times \ci^1_{\alpha_1}$ such that $v \in \ci^\vf_{\alpha_\vf}(\tilde \Delta, \tilde \Theta, \pi)$. For the simplicity of notations, we denote by 
\$
v^\dagger \in \argmin_{v\in \cV} \max_{g\in \cW} \hat \Phi^\pi_\vf(v, g; \tilde \Delta, \tilde \Theta),
\$
i.e., $v^\dagger = \hat v^\pi_{\tilde \Delta, \tilde \Theta}$, which is defined in \eqref{eq:v-pi-def}. By the definition of $v^\dagger$ and $v \in \ci^\vf_{\alpha_\vf}(\tilde \Delta, \tilde \Theta, \pi)$, we know that 
\#\label{eq:iv-dd1-vf-spec}
\max_{g\in \cW} \hat \Phi^\pi_\vf(v,g;\tilde \Delta, \tilde \Theta) - \max_{g\in \cW} \hat \Phi^\pi_\vf(v^\dagger,g;\tilde \Delta, \tilde \Theta) \leq \alpha_\vf. 
\#
Note that 
\#\label{eq:iv-dd2-vf-spec}
& \max_{g\in \cW} \Phi^\pi_\vf(v,g;\Delta^*, \Theta^*)  \\
& \qquad \leq 2 \underbrace{\max_{(v,g,\Delta,\Theta)\in (\cV, \cW, \cF_0, \cF_1)} \left | \Phi^\pi_\vf(v,g;\Delta, \Theta) - \hat \Phi^\pi_\vf(v,g;\Delta, \Theta) \right|}_{\text{Term (I)}} + \underbrace{\max_{g\in \cW} \Phi^\pi_\vf(v^\dagger,g;\tilde \Delta, \tilde \Theta)}_{\text{Term (II)}}  \\
& \qquad \qquad + \underbrace{\max_{g\in \cW} \left | \hat \Phi^\pi_\vf(v,g;\Delta^*, \Theta^*) - \hat \Phi^\pi_\vf(v,g;\tilde \Delta, \tilde \Theta) \right|}_{\text{Term (III)}} + \alpha_\vf,
\#
where we use \eqref{eq:iv-dd1-vf-spec} in the last inequality. Now we upper bound terms (I), (II), and (III) on the RHS of \eqref{eq:iv-dd2-vf-spec}. 

\vskip5pt
\noindent\textbf{Upper Bounding Term (I).}
By Theorem \ref{thm:hoeffding-mixing}, with probability at least $1 - \delta$, it holds for any $(v,g,\Delta,\Theta, \pi) \in (\cV, \cW, \cF_0, \cF_1, \Pi)$ that
\$
\left |  \hat \Phi_\vf^\pi(v, g; \Delta, \Theta) - \Phi_\vf^\pi(v, g; \Delta, \Theta) \right | \leq c\cdot \frac{C_{\Delta^*} C_{\Theta^*} C_*}{1-\gamma} \sqrt{\frac{\pdim_{\cF_0,\cF_1,\cW,\cV, \Pi}}{NT \kappa}\log\frac{1}{\delta} \log(NT)},
\$
which implies that with probability at least $1 - \delta$, we have
\#\label{eq:iv-dd3-vf-spec}
\text{Term (I)} \leq c\cdot \frac{C_{\Delta^*} C_{\Theta^*} C_*}{1-\gamma} \sqrt{\frac{\pdim_{\cF_0,\cF_1,\cW,\cV,\Pi}}{NT \kappa}\log\frac{1}{\delta}\log(NT) }.
\#

\vskip5pt
\noindent\textbf{Upper Bounding Term (II).} 
Recall that $\tilde v^\pi \in \argmin_{v\in \cV} \max_{w\in \cW} \Phi_\vf^\pi(v,w; \Delta^*, \Theta^*)$. 
Note that 
\#\label{eq:iv-ee3-vf-spec}
& \max_{g\in \cW} \Phi_\vf^\pi(v^\dagger, g; \tilde \Delta, \tilde \Theta)  \\
& \qquad \leq 2 \max_{v\in \cV} \max_{g\in \cW} \left |  \Phi_\vf^\pi(v, g; \tilde \Delta, \tilde \Theta) -  \hat \Phi_\vf^\pi(v, g; \tilde \Delta, \tilde \Theta) \right |  \\
& \qquad \qquad + \max_{g\in \cW} \hat \Phi_\vf^\pi(v^\dagger, g; \tilde \Delta, \tilde \Theta) - \max_{g\in \cW} \hat \Phi_\vf^\pi( \tilde v^\pi, g; \tilde \Delta, \tilde \Theta) + \max_{g\in \cW}  \Phi_\vf^\pi(\tilde v^\pi, g; \tilde \Delta, \tilde \Theta)  \\
& \qquad \leq 2 \max_{v\in \cV} \max_{g\in \cW} \left |  \Phi_\vf^\pi(v, g; \tilde \Delta, \tilde \Theta) -  \hat \Phi_\vf^\pi(v, g; \tilde \Delta, \tilde \Theta) \right | \\
& \qquad \qquad + \max_{g\in \cW}  \left | \Phi_\vf^\pi(\tilde v^\pi, g; \tilde \Delta, \tilde \Theta) -  \Phi_\vf^\pi(\tilde v^\pi, g; \Delta^*, \Theta^*) \right | + \max_{g\in \cW}  \Phi_\vf^\pi(\tilde v^\pi, g; \Delta^*, \Theta^*), 
\#
where we use triangle inequality and the fact that $v^\dagger \in \argmin_{v\in \cV} \max_{g\in \cW} \hat \Phi^\pi_\vf(v, g; \Delta, \Theta)$ in the last inequality. 
In the meanwhile, by Theorem \ref{thm:hoeffding-mixing}, with probability at least $1 - \delta$, it holds for any $(v,g,\pi)\in \cV\times \cW \times \Pi$ that
\#\label{eq:iv-ee4-vf-spec}
\left | \Phi_\vf^\pi(v, g; \tilde \Delta, \tilde \Theta) - \hat \Phi_\vf^\pi(v, g; \tilde \Delta, \tilde \Theta) \right | \leq c\cdot \frac{C_{\Delta^*} C_{\Theta^*} C_*}{1-\gamma} \sqrt{\frac{\pdim_{\cV,\cW,\Pi}}{NT \kappa}\log\frac{1}{\delta} \log(NT)}.
\#
Also, we upper bound the second term on the RHS of \eqref{eq:iv-ee3-vf-spec} with probability at least $1- \delta$ by a similar argument as in \eqref{eq:iv-dd5-vf-2}, 
\#\label{eq:iv-ee5-vf-spec}
& \left |  \Phi_\vf^\pi(\tilde v^\pi, g; \tilde \Delta, \tilde \Theta) -   \Phi_\vf^\pi(\tilde v^\pi, g; \Delta^*, \Theta^*) \right |  \\
& \leq \frac{C_{\Delta^*} C_{\Theta^*} C_*}{1-\gamma} \left(\xi_0 C_{\Theta^*} \sqrt{\frac{C_{\Delta^*}}{NT \kappa} \pdim_{\cF_0}\log\frac{1}{\delta} \log(NT)} + \xi_1 C_{\Delta^*} \sqrt{\frac{C_{\Theta^*}}{NT \kappa} \pdim_{\cF_1} \log\frac{1}{\delta} \log(NT)} \right),
\#
where in the first inequality, we use the fact that $\|v\|_\infty \leq 1/ (1 - \gamma)$ and $\|g\|_\infty \leq C_*$; in the third inequality, we use Cauchy Schwarz inequality; while in the last inequality, we use Assumption \ref{ass:iv-sl-res} with $(\Delta, \Theta) \in \ci^0_{\alpha_0} \times \ci^1_{\alpha_1}$. 
Now, by plugging \eqref{eq:iv-ee4-vf-spec} and \eqref{eq:iv-ee5-vf-spec} into \eqref{eq:iv-ee3-vf-spec}, with probability at least $1 - \delta$, it holds for any $(\Delta, \Theta, \pi) \in \ci^0_{\alpha_0} \times \ci^1_{\alpha_1} \times \Pi$ that 
\#\label{eq:iv-dd4-vf-spec}
\text{Term (II)} \leq c\cdot \frac{C_{\Delta^*}^2 C_{\Theta^*}^2 C_*}{1-\gamma} (\xi_0 + \xi_1) \sqrt{\frac{\pdim_{\cF_0,\cF_1,\cW,\cV,\Pi}}{NT \kappa}\log\frac{1}{\delta} \log(NT)} + \max_{g\in \cW}  \Phi_\vf^\pi(\tilde v^\pi, g; \Delta^*, \Theta^*). 
\#

\vskip5pt
\noindent\textbf{Upper Bounding Term (III).} 
Note that 
\#\label{eq:iv-dd5-vf-spec}
& \left | \hat \Phi^\pi_\vf(v,g;\Delta^*, \Theta^*) - \hat \Phi^\pi_\vf(v,g;\tilde \Delta, \tilde \Theta) \right|  \\
& \leq \left | \left(\hat \EE - \EE\right) \left[ \frac{1}{T} \sum_{t = 0}^{T-1} g(S_t) \left ( \frac{Z_t^\top A_t\pi(A_t\given S_t) }{\Delta^*(S_t,A_t)\Theta^*(S_t,Z_t)} - \frac{Z_t^\top A_t\pi(A_t\given S_t)}{\tilde \Delta(S_t,A_t) \tilde \Theta(S_t,Z_t)} \right) \left( R_t + \gamma v(S_{t+1}) \right) \right ] \right |  \\
& \qquad + \left | \EE \left[ \frac{1}{T} \sum_{t = 0}^{T-1} g(S_t) \left ( \frac{Z_t^\top A_t\pi(A_t\given S_t) }{\Delta^*(S_t,A_t)\Theta^*(S_t,Z_t)} - \frac{Z_t^\top A_t\pi(A_t\given S_t)}{\tilde \Delta(S_t,A_t) \tilde \Theta(S_t,Z_t)} \right) \left( R_t + \gamma v(S_{t+1}) \right) \right ] \right |. 
\#
For the first term on the RHS of \eqref{eq:iv-dd5-vf}, by Theorem \ref{thm:hoeffding-mixing}, with probability at least $1- \delta$, it holds for any $(v,g,\pi)\in \cV\times \cW \times \Pi$ that
\#\label{eq:iv-dd5-vf-1-spec}
& \left | \left(\hat \EE - \EE\right) \left[ \frac{1}{T} \sum_{t = 0}^{T-1} g(S_t) \left ( \frac{Z_t^\top A_t\pi(A_t\given S_t) }{\Delta^*(S_t,A_t)\Theta^*(S_t,Z_t)} - \frac{Z_t^\top A_t\pi(A_t\given S_t)}{\tilde \Delta(S_t,A_t) \tilde \Theta(S_t,Z_t)} \right) \left( R_t + \gamma v(S_{t+1}) \right) \right ] \right |  \\
& \qquad \leq c\cdot \frac{C_{\Delta^*} C_{\Theta^*} C_*}{1-\gamma} \sqrt{\frac{\pdim_{\cF_0,\cF_1,\cW,\cV,\Pi}}{NT \kappa}\log\frac{1}{\delta} \log(NT)}. 
\#
For the second term on the RHS of \eqref{eq:iv-dd5-vf-spec}, by a similar argument as in \eqref{eq:iv-dd5-vf-2}, with probability at least $1- \delta$, it holds that 
\#\label{eq:iv-dd5-vf-2-spec}
& \left | \EE \left[ \frac{1}{T} \sum_{t = 0}^{T-1} g(S_t) \left ( \frac{Z_t^\top A_t\pi(A_t\given S_t) }{\Delta^*(S_t,A_t)\Theta^*(S_t,Z_t)} - \frac{Z_t^\top A_t\pi(A_t\given S_t)}{\tilde \Delta(S_t,A_t) \tilde \Theta(S_t,Z_t)} \right) \left( R_t + \gamma v(S_{t+1}) \right) \right ] \right |  \\
& \leq \frac{C_{\Delta^*} C_{\Theta^*} C_*}{1-\gamma} \left(\xi_0 C_{\Theta^*} \sqrt{\frac{C_{\Delta^*}}{NT \kappa} \pdim_{\cF_0}\log\frac{1}{\delta} \log(NT)} + \xi_1 C_{\Delta^*} \sqrt{\frac{C_{\Theta^*}}{NT \kappa} \pdim_{\cF_1} \log\frac{1}{\delta} \log(NT)} \right),
\#
where in the first inequality, we use the fact that $\|v\|_\infty \leq 1/ (1 - \gamma)$ and $\|g\|_\infty \leq C_*$; in the third inequality, we use Cauchy Schwarz inequality; while in the last inequality, we use Assumption \ref{ass:iv-sl-res} with the fact that $(\tilde \Delta, \tilde\Theta) \in \ci^0_{\alpha_0} \times \ci^1_{\alpha_1}$. Now, by plugging \eqref{eq:iv-dd5-vf-1-spec} and \eqref{eq:iv-dd5-vf-2-spec} into \eqref{eq:iv-dd5-vf-spec}, it holds with probability at least $1 - \delta$ that 
\#\label{eq:879434-spec}
& \left | \hat \Phi^\pi_\vf(v,g;\Delta^*, \Theta^*) - \hat \Phi^\pi_\vf(v,g;\tilde \Delta, \tilde \Theta) \right|  \\
& \qquad \leq c\cdot \frac{C_{\Delta^*}^2 C_{\Theta^*}^2 C_*}{1-\gamma} (\xi_0 + \xi_1) \sqrt{\frac{1}{NT\kappa}\cdot \pdim_{\cF_0,\cF_1,\cW,\cV,\Pi} \cdot \log\frac{1}{\delta} \log(NT)}. 
\#

\vskip5pt
Now, by plugging \eqref{eq:iv-dd3-vf-spec}, \eqref{eq:iv-dd4-vf-spec}, and \eqref{eq:879434-spec} into 
\eqref{eq:iv-dd2-vf-spec}, with probability at least $1 - \delta$, it holds for any $v \in \cup_{(\Delta, \Theta) \in \ci^0_{\alpha_0} \times \ci^1_{\alpha_1}} \ci^\vf_{\alpha_\vf}(\Delta, \Theta, \pi)$ and $\pi \in \Pi$ that
\$
\max_{g\in \cW} \Phi^\pi_\vf(v,g;\Delta^*, \Theta^*) & \leq  c\cdot \frac{C_{\Delta^*}^2 C_{\Theta^*}^2 C_*}{1-\gamma} (\xi_0 + \xi_1) \sqrt{\frac{1}{NT\kappa}\cdot \pdim_{\cF_0,\cF_1,\cW,\cV,\Pi} \cdot \log\frac{1}{\delta} \log(NT)}  \\
& \qquad + \max_{g\in \cW}  \Phi_\vf^\pi(\tilde v^\pi, g; \Delta^*, \Theta^*),
\$
which concludes the proof of the lemma.
\end{proof}

\subsection{Proof of Lemma \ref{lemma:iv-w-pi-in-conf-spec}}\label{prf:lemma:iv-w-pi-in-conf-spec}
\begin{proof}
First, by the definition of $\tilde w^\pi$ in \eqref{eq:v-tilde-def}, we know that $\tilde w^\pi \in \cW$. 
For notation simplicity, we denote by $\Phi_\mis^\pi(w, f; *) = \Phi_\mis^\pi(w, f; \Delta^*, \Theta^*)$ and $\hat w^\pi_* = \hat w^\pi_{\Delta^*, \Theta^*}$ for any $(\pi,w,f)$.
Note that 
\#\label{eq:iv-ee1-spec}
& \max_{f\in \cV} \hat \Phi_\mis^\pi(\tilde w^\pi, f; *) - \max_{f\in \cV} \hat \Phi_\mis^\pi(\hat w^\pi_*, f; *)  \\
& \qquad = \max_{f\in \cV} \hat \Phi_\mis^\pi(\tilde w^\pi, f; *) - \max_{f\in \cV}  \Phi_\mis^\pi(\tilde w^\pi, f; *)  + \max_{f\in \cV}  \Phi_\mis^\pi(\tilde w^\pi, f; *) - \max_{f\in \cV}  \Phi_\mis^\pi(\hat w^\pi_*, f; *)  \\
& \qquad \qquad + \max_{f\in \cV} \Phi_\mis^\pi(\hat w^\pi_*, f; *) - \max_{f\in \cV} \hat \Phi_\mis^\pi(\hat w^\pi_*, f; *)  \\
& \qquad \leq \max_{f\in \cV} \hat \Phi_\mis^\pi(\tilde w^\pi, f; *) - \max_{f\in \cV}  \Phi_\mis^\pi(\tilde w^\pi, f; *) + \max_{f\in \cV} \Phi_\mis^\pi(\hat w^\pi_*, f; *) - \max_{f\in \cV} \hat \Phi_\mis^\pi(\hat w^\pi_*, f; *)  \\
& \qquad \leq 2 \max_{w\in \cW} \left | \max_{f\in \cV} \hat \Phi_\mis^\pi(w, f; *) - \max_{f\in \cV}  \Phi_\mis^\pi(w, f; *) \right |  \\
& \qquad \leq 2 \max_{w\in \cW} \max_{f\in \cV} \left |  \hat \Phi_\mis^\pi(w, f; *) - \Phi_\mis^\pi(w, f; *) \right |,
\#
where in the first inequality, we use the fact that $\max_{f\in \cV}  \Phi_\mis^\pi(\tilde w^\pi, f; *) \leq \max_{f\in \cV}  \Phi_\mis^\pi(\hat w_*^\pi, f; *)$ by the definition of $\tilde w^\pi$ in \eqref{eq:v-tilde-def}; while in the second inequality, we use the fact that $\tilde w^\pi, \hat w^\pi_* \in \cW$. In the meanwhile, by Theorem \ref{thm:hoeffding-mixing}, with probability at least $1 - \delta$, it holds for any $(w,f,\pi)\in \cW\times \cV \times \Pi$ that
\#\label{eq:iv-ee2-spec}
\left |  \hat \Phi_\mis^\pi(w, f; *) - \Phi_\mis^\pi(w, f; *) \right | \leq c\cdot \frac{C_{\Delta^*} C_{\Theta^*} C_*}{1-\gamma} \sqrt{\frac{1}{NT \kappa} \pdim_{\cV,\cW,\Pi}\log\frac{1}{\delta} \log(NT)}, 
\#
where we use Assumption \ref{ass:upper-bound-delta}. 
Now, combining \eqref{eq:iv-ee1-spec} and \eqref{eq:iv-ee2-spec}, with probability at least $1 - \delta$, we have
\$
& \max_{f\in \cV} \hat \Phi_\mis^\pi(\tilde w^\pi, f; *) - \max_{f\in \cV} \hat \Phi_\mis^\pi(\hat w^\pi_*, f; *) \\
& \qquad \leq c\cdot \frac{C_{\Delta^*} C_{\Theta^*} C_*}{1-\gamma} \sqrt{\frac{1}{NT \kappa} \pdim_{\cV,\cW,\Pi}\log\frac{1}{\delta} \log(NT)} = \alpha_\mis, 
\$
which implies that $\tilde w^\pi \in \ci^\mis_{\alpha_\mis}(\Delta^*, \Theta^*, \pi)$. This concludes the proof of the lemma. 
\end{proof}

\subsection{Proof of Lemma \ref{lemma:iv-w-in-conf-good-spec}}
\label{prf:lemma:iv-w-in-conf-good-spec}
\begin{proof}
Since $w \in \cup_{(\Delta, \Theta) \in \ci^0_{\alpha_0} \times \ci^1_{\alpha_1}} \ci^\mis_{\alpha_\mis}(\Delta, \Theta, \pi)$, there exists a pair $(\tilde \Delta, \tilde\Theta) \in \ci^0_{\alpha_0} \times \ci^1_{\alpha_1}$ such that $w \in \ci^\mis_{\alpha_\mis}(\tilde \Delta, \tilde \Theta, \pi)$. For the simplicity of notations, we denote by 
\$
w^\dagger \in \argmin_{w\in \cW} \max_{f\in \cV} \hat \Phi^\pi_\mis(w, f;\tilde \Delta, \tilde \Theta),
\$
i.e., $w^\dagger = \hat w^\pi_{\tilde \Delta, \tilde \Theta}$, which is defined in \eqref{eq:w-pi-def}. By the definition of $w^\dagger $ and $w \in \ci^\mis_{\alpha_\mis}(\tilde \Delta, \tilde \Theta, \pi)$, with probability at least $1 - \delta$, it holds for any $\pi\in \Pi$ and $w \in \ci^\mis_{\alpha_\mis}(\tilde \Delta, \tilde \Theta, \pi)$ that
\#\label{eq:iv-dd1-spec}
\max_{f\in \cV} \hat \Phi^\pi_\mis(w,f;\tilde \Delta, \tilde \Theta) - \max_{f\in \cV} \hat \Phi^\pi_\mis(w^\dagger,f;\tilde \Delta, \tilde \Theta) \leq \alpha_\mis. 
\#
Further, we observe that
\#\label{eq:iv-dd2-spec}
& \max_{f\in \cV} \Phi^\pi_\mis(w,f;\Delta^*, \Theta^*)  \\
& \qquad \leq \underbrace{\max_{(w,f,\Delta,\Theta)\in (\cW, \cV, \cF_0, \cF_1)} \left | \Phi^\pi_\mis(w,f;\Delta, \Theta) - \hat \Phi^\pi_\mis(w,f;\Delta, \Theta) \right|}_{\text{Term (I)}} + \underbrace{\max_{f\in \cV} \Phi^\pi_\mis(w^\dagger,f;\tilde \Delta, \tilde \Theta)}_{\text{Term (II)}} \\
& \qquad \qquad + \underbrace{\max_{f\in \cV} \left | \hat \Phi^\pi_\mis(w,f;\Delta^*, \Theta^*) - \hat \Phi^\pi_\mis(w,f;\tilde \Delta, \tilde \Theta) \right|}_{\text{Term (III)}} + \alpha_\mis, 
\#
where we use \eqref{eq:iv-dd1-spec} in the last inequality. Now we upper bound terms (I), (II), and (III) on the RHS of \eqref{eq:iv-dd2-spec}. 

\vskip5pt
\noindent\textbf{Upper Bounding Term (I).}
By Theorem \ref{thm:hoeffding-mixing}, with probability at least $1 - \delta$, it holds for any $(w,f,\Delta,\Theta,\pi) \in (\cW, \cV, \cF_0, \cF_1, \Pi)$ that
\$
\left |  \hat \Phi_\mis^\pi(w, f; \Delta, \Theta) - \Phi_\mis^\pi(w, f; \Delta, \Theta) \right | \leq c\cdot \frac{C_{\Delta^*} C_{\Theta^*} C_*}{1-\gamma} \sqrt{\frac{1}{NT \kappa}\pdim_{\cF_0,\cF_1,\cW,\cV,\Pi}\log\frac{1}{\delta} \log(NT)},
\$
which implies that with probability at least $1 - \delta$, we have
\#\label{eq:iv-dd3-spec}
\text{Term (I)} \leq c\cdot \frac{C_{\Delta^*} C_{\Theta^*} C_*}{1-\gamma} \sqrt{\frac{1}{NT \kappa}\pdim_{\cF_0,\cF_1,\cW,\cV,\Pi}\log\frac{1}{\delta} \log(NT)}.
\#

\vskip5pt
\noindent\textbf{Upper Bounding Term (II).} Recall that $\tilde w^\pi\in \argmin_{w\in \cW} \max_{v\in \cV} \Phi_\mis^\pi(w, v; \Delta^*, \Theta^*)$. 
Note that 
\#\label{eq:iv-ee3-spec}
& \max_{f\in \cV} \Phi_\mis^\pi(w^\dagger, f; \tilde \Delta, \tilde \Theta)  \\
& \qquad = \max_{f\in \cV} \Phi_\mis^\pi(w^\dagger, f; \tilde \Delta, \tilde \Theta) - \max_{f\in \cV} \hat \Phi_\mis^\pi(w^\dagger, f; \tilde \Delta, \tilde \Theta) + \max_{f\in \cV} \hat \Phi_\mis^\pi(w^\dagger, f; \tilde \Delta, \tilde \Theta)  \\
& \qquad \qquad - \max_{f\in \cV} \hat \Phi_\mis^\pi( \tilde w^\pi, f; \tilde \Delta, \tilde \Theta)  + \max_{f\in \cV} \hat \Phi_\mis^\pi( \tilde w^\pi, f; \tilde \Delta, \tilde \Theta) - \max_{f\in \cV}  \Phi_\mis^\pi(\tilde w^\pi, f; \tilde \Delta, \tilde \Theta)  \\
& \qquad \qquad + \max_{f\in \cV}  \Phi_\mis^\pi(\tilde w^\pi, f; \tilde \Delta, \tilde \Theta)  \\
& \qquad \leq 2 \max_{w\in \cW} \max_{f\in \cV} \left |  \Phi_\mis^\pi(w, f; \tilde \Delta, \tilde \Theta) -  \hat \Phi_\mis^\pi(w, f; \tilde \Delta, \tilde \Theta) \right |  \\
& \qquad \qquad + \max_{f\in \cV} \left |  \Phi_\mis^\pi(\tilde w^\pi, f; \tilde \Delta, \tilde \Theta) -   \Phi_\mis^\pi(\tilde w^\pi, f; \Delta^*, \Theta^*) \right | + \max_{f\in \cV} \Phi_\mis^\pi(\tilde w^\pi, f; \Delta^*, \Theta^*), 
\#
where we use triangle inequality and the fact that $w^\dagger \in \argmin_{w\in \cW} \max_{f\in \cV} \hat \Phi^\pi_\mis(w, f; \Delta, \Theta)$ in the last inequality. In the meanwhile, by Theorem \ref{thm:hoeffding-mixing}, with probability at least $1 - \delta$, it holds for any $(w,f,\pi)\in \cW\times \cV \times \Pi$ that
\#\label{eq:iv-ee4-spec}
\left | \Phi_\mis^\pi(w, f; \tilde \Delta, \tilde \Theta) - \hat \Phi_\mis^\pi(w, f; \tilde  \Delta, \tilde \Theta) \right | \leq c\cdot \frac{C_{\Delta^*} C_{\Theta^*} C_*}{1-\gamma} \sqrt{\frac{1}{NT \kappa}\pdim_{\cV,\cW,\Pi}\log\frac{1}{\delta} \log(NT)}. 
\#
Also, we upper bound the second term on the RHS of \eqref{eq:iv-ee3-spec} with probability at least $1- \delta$ as follows, 
\#\label{eq:iv-ee5-spec}
& \left |  \Phi_\mis^\pi(\tilde w^\pi, f; \tilde  \Delta, \tilde \Theta) -   \Phi_\mis^\pi( \tilde w^\pi, f; \Delta^*, \Theta^*) \right |  \\
& = \left | \EE\left[ \frac{1}{T} \sum_{t = 0}^{T-1} \left ( \frac{Z_t^\top A_t\pi(A_t\given S_t)w(S_t) }{\Delta^*(S_t,A_t)\Theta^*(S_t,Z_t)} - \frac{Z_t^\top A_t\pi(A_t\given S_t) w(S_t)}{\tilde \Delta(S_t,A_t) \tilde \Theta(S_t,Z_t)} \right) \left(f(S_t) - \gamma f(S_{t+1})\right) \right ] \right |  \\
& \leq \frac{C_{\Delta^*} C_{\Theta^*} C_*}{1-\gamma} \left(\xi_0 C_{\Theta^*} \sqrt{\frac{C_{\Delta^*}}{NT \kappa}\pdim_{\cF_0}\log\frac{1}{\delta} \log(NT)} + \xi_1 C_{\Delta^*} \sqrt{\frac{C_{\Theta^*}}{NT \kappa} \pdim_{\cF_1} \log\frac{1}{\delta} \log(NT)} \right),
\#
where we use Cauchy-Schwarz inequality and Assumption \ref{ass:iv-sl-res} in the last inequality.
Now, by plugging \eqref{eq:iv-ee4-spec} and \eqref{eq:iv-ee5-spec} into \eqref{eq:iv-ee3-spec}, it holds with probability at least $1 - \delta$ that 
\#\label{eq:iv-dd4-spec}
\text{Term (II)} \leq c\cdot \frac{C_{\Delta^*}^2 C_{\Theta^*}^2 C_*}{1-\gamma} (\xi_0 + \xi_1 ) \sqrt{\frac{\pdim_{\cF_0,\cF_1,\cW,\cV,\Pi}}{NT \kappa} \log\frac{1}{\delta} \log(NT)} + \max_{f\in \cV} \Phi_\mis^\pi(\tilde w^\pi, f; \Delta^*, \Theta^*).
\#

\vskip5pt
\noindent\textbf{Upper Bounding Term (III).} 
Note that 
\#\label{eq:iv-dd5-spec}
& \left | \hat \Phi^\pi_\mis(w,f;\Delta^*, \Theta^*) - \hat \Phi^\pi_\mis(w,f;\tilde \Delta, \tilde \Theta) \right|  \\
&  \leq \left | \left(\hat \EE - \EE\right) \left[ \frac{1}{T} \sum_{t = 0}^{T-1} \left ( \frac{Z_t^\top A_t\pi(A_t\given S_t) w(S_t) }{\Delta^*(S_t,A_t)\Theta^*(S_t,Z_t)} - \frac{Z_t^\top A_t\pi(A_t\given S_t) w(S_t)}{\tilde \Delta(S_t,A_t) \tilde \Theta(S_t,Z_t)} \right) \left(f(S_t) - \gamma f(S_{t+1}) \right) \right ] \right |  \\
&  \quad + \left |  \EE \left[ \frac{1}{T} \sum_{t = 0}^{T-1} \left ( \frac{Z_t^\top A_t\pi(A_t\given S_t) w(S_t) }{\Delta^*(S_t,A_t)\Theta^*(S_t,Z_t)} - \frac{Z_t^\top A_t\pi(A_t\given S_t) w(S_t)}{\tilde \Delta(S_t,A_t) \tilde \Theta(S_t,Z_t)} \right) \left(f(S_t) - \gamma f(S_{t+1}) \right) \right ] \right |. 
\#
For the first term on the RHS of \eqref{eq:iv-dd5-spec}, by Theorem \ref{thm:hoeffding-mixing}, with probability at least $1- \delta$, it holds for any $(w,f,\pi)\in \cW\times \cV \times \Pi$ that
\#\label{eq:iv-dd5-1-spec}
& \left | \left(\hat \EE - \EE\right) \left[ \frac{1}{T} \sum_{t = 0}^{T-1} \left ( \frac{Z_t^\top A_t\pi(A_t\given S_t) w(S_t) }{\Delta^*(S_t,A_t)\Theta^*(S_t,Z_t)} - \frac{Z_t^\top A_t\pi(A_t\given S_t) w(S_t)}{\tilde \Delta(S_t,A_t) \tilde \Theta(S_t,Z_t)} \right) \left(f(S_t) - \gamma f(S_{t+1}) \right) \right ] \right |  \\
& \qquad \leq c\cdot \frac{C_{\Delta^*} C_{\Theta^*} C_*}{1-\gamma} \sqrt{\frac{\pdim_{\cF_0,\cF_1,\cW,\cV,\Pi}}{NT \kappa}\log\frac{1}{\delta} \log(NT)}. 
\#
For the second term on the RHS of \eqref{eq:iv-dd5-spec}, by a similar argument as in \eqref{eq:iv-dd5-vf-2},  it holds with probability at least $1 - \delta$ that 
\#\label{eq:iv-dd5-2-spec}
& \left | \EE \left[ \frac{1}{T} \sum_{t = 0}^{T-1} \left ( \frac{Z_t^\top A_t\pi(A_t\given S_t) w(S_t)}{\Delta^*(S_t,A_t)\Theta^*(S_t,Z_t)} - \frac{Z_t^\top A_t\pi(A_t\given S_t)w(S_t)}{\tilde \Delta(S_t,A_t) \tilde \Theta(S_t,Z_t)} \right) \left(f(S_t) - \gamma f(S_{t+1}) \right) \right ] \right |  \\
& \leq \frac{2 C_{\Delta^*} C_{\Theta^*} C_*}{1-\gamma} \left(\xi_0 C_{\Theta^*} \sqrt{\frac{C_{\Delta^*}}{NT \kappa} \pdim_{\cF_0}\log\frac{1}{\delta} \log(NT)} + \xi_1 C_{\Delta^*} \sqrt{\frac{C_{\Theta^*}}{NT \kappa} \pdim_{\cF_1} \log\frac{1}{\delta} \log(NT)} \right), 
\#
where in the first inequality, we use the fact that $\|f\|_\infty \leq 1/ (1 - \gamma)$ and $\|w\|_\infty \leq C_*$; in the third inequality, we use Cauchy Schwarz inequality; while in the last inequality, we use Assumption \ref{ass:iv-sl-res} with the fact that $(\tilde \Delta, \tilde\Theta) \in \ci^0_{\alpha_0} \times \ci^1_{\alpha_1}$. Now, by plugging \eqref{eq:iv-dd5-1-spec} and \eqref{eq:iv-dd5-2-spec} into \eqref{eq:iv-dd5-spec}, with probability at least $1 - \delta$,  it holds for any $w \in \cup_{(\Delta, \Theta) \in \ci^0_{\alpha_0} \times \ci^1_{\alpha_1}} \ci^\mis_{\alpha_\mis}(\Delta, \Theta, \pi)$ and $(f,\pi)\in \cV\times \Pi$ that 
\#\label{eq:732847-spec}
& \left | \hat \Phi^\pi_\mis(w,f;\Delta^*, \Theta^*) - \hat \Phi^\pi_\mis(w,f;\tilde \Delta, \tilde \Theta) \right|  \\
& \qquad \leq c \cdot \frac{ C_{\Delta^*}^2 C_{\Theta^*}^2 C_*}{1-\gamma} (\xi_0 + \xi_1) \sqrt{\frac{1}{NT \kappa}\cdot \pdim_{\cF_0,\cF_1,\cW,\cV,\Pi} \cdot \log\frac{1}{\delta} \log(NT) }. 
\#

\vskip5pt
Now, by plugging \eqref{eq:iv-dd3-spec}, \eqref{eq:iv-dd4-spec}, and \eqref{eq:732847-spec} into 
\eqref{eq:iv-dd2-spec}, with probability at least $1 - \delta$, it holds for any $\pi \in \Pi$ and $w \in \cup_{(\Delta, \Theta) \in \ci^0_{\alpha_0} \times \ci^1_{\alpha_1}} \ci^\mis_{\alpha_\mis}(\Delta, \Theta, \pi)$ that
\$
\max_{f\in \cV} \Phi^\pi_\mis(w,f;\Delta^*, \Theta^*) & \leq c\cdot \frac{ C_{\Delta^*}^2 C_{\Theta^*}^2 C_*}{1-\gamma} (\xi_0 + \xi_1) \sqrt{\frac{1}{NT \kappa}\cdot \pdim_{\cF_0,\cF_1,\cW,\cV,\Pi} \cdot \log\frac{1}{\delta} \log(NT) } \\
& \qquad + \max_{f\in \cV} \Phi_\mis^\pi(\tilde w^\pi, f; \Delta^*, \Theta^*),
\$
which concludes the proof of the lemma. 
\end{proof}

\section{Auxiliary Results}

We introduce auxiliary results used in the paper. We provide the proofs of these results in \S\ref{sec:prf:aux}. We first introduce the following definition of $\beta$-mixing coefficient. 

\begin{definition}
Let $\{Z_t \}_{t \geq 0}$ be a sequence of random variables. For any $i,j\in \NN \cup\{\infty\}$, we denote by $\sigma_i^j$ the sigma algebra generated by $\{Z_k\}_{i\leq k \leq j}$. The $\beta$-mixing coefficient of $\{Z_t \}_{t \geq 0}$ is defined as $\beta(t) = \sup_n\EE_{B\in \sigma_0^n} [\sup_{A\in \sigma_{n + t}^\infty} |\PP(A\given B) - \PP(A) |]$.
\end{definition}

We introduce the following form of $\beta$-mixing coefficient for Markov chains. 

\begin{lemma}\label{lemma:davydov}
Suppose $\{Z_t \}_{t \geq 0}$ is a Markov chain with initial distribution $\zeta$. It holds that
\$
\beta(t) \leq \frac{1}{2} \int \|p_{t'}(\cdot \given z) - p_\stat(\cdot)\|_\tv \ud p_\stat(z) + \frac{3}{2} \int \|p_{t'}(\cdot \given z) - p_\stat(\cdot)\|_\tv \ud \zeta(z), 
\$
where $t' = \lfloor t/2 \rfloor$ and $p_n(\cdot \given z)$ is the the marginal the distribution of $Z_n$ given $Z_0 = z$ for any $n\in [N]$. 
\end{lemma}
\begin{proof}
See the proof of Lemma 1 in \cite{meitz2021subgeometric} for a detailed proof. 
\end{proof}

Following from Lemma \ref{lemma:davydov}, we can upper bound the $\beta$-mixing coefficient for a Markov chain $\{Z_t \}_{t \geq 0}$. Before that, we impose the following assumption on $\{Z_t \}_{t \geq 0}$. 

\begin{assumption}\label{ass:mc-ass}
The Markov chain $\{Z_t \}_{t \geq 0}$ with initial distribution $\zeta$ admits a unique stationary distribution $p_\stat$ over $\cZ$ and is geometrically ergodic, i.e., there exists a function $\varphi\colon \cZ \to [0, \infty)$ and a constant $\kappa > 0$ such that 
\$
\left\| p_\stat(\cdot) - p_t(\cdot \given z_0) \right\|_\tv  \leq \varphi(z_0) \cdot \exp\left(-2\kappa t\right), 
\$
where $p_t(\cdot \given z_0)$ is the marginal distribution of $Z_t$ given $Z_0 = z_0$ and there exists a positive absolute constant $c$ such that $\int \varphi(z) \ud \zeta(z) \leq c$ and $\int \varphi(z) \ud p_\stat(z) \leq c$. 
\end{assumption}

\begin{lemma}\label{lemma:davydov-adapt}
Suppose $\{Z_t \}_{t \geq 0}$ is a Markov chain satisfying Assumption \ref{ass:mc-ass}. Then we have $\beta(t) \leq c\cdot \exp(-\kappa t)$ for any $t \geq 0$. 
\end{lemma}
\begin{proof}
For any $t \geq 0$, by Lemma \ref{lemma:davydov}, we have
\$
\beta(t) & \leq \frac{1}{2} \int \|p_{t'}(\cdot \given z) - p_\stat(\cdot)\|_\tv \ud p_\stat(z) + \frac{3}{2} \int \|p_{t'}(\cdot \given z) - p_\stat(\cdot)\|_\tv \ud \zeta(z)  \\
& \leq \frac{1}{2} \int \varphi(z) \cdot \exp\left(-\kappa t\right) \ud p_\stat(z) + \frac{3}{2} \int \varphi(z) \cdot \exp\left(-\kappa t\right) \ud \zeta(z)  \\
& \leq c\cdot \exp(-\kappa t), 
\$
where in the second and last inequalities, we use Assumption \ref{ass:mc-ass}. This concludes the proof of the lemma. 
\end{proof}

\subsection{Concentration Inequality for Geometrically Ergodic Non-Stationary Sequence}

We first introduce the following lemma, which is a straight-forward genelization of Berbee's lemma \citep{Berbee1979RandomWW}. 

\begin{lemma}\label{lemma:berbee}
For any $k > 0$ and a random sequence $\{Y_\ell\}_{\ell=1}^{k}$, there exists a random sequence $\{\tilde Y_\ell\}_{\ell=1}^{k}$ such that 
\begin{enumerate}
    \item $\{\tilde Y_\ell\}_{\ell=1}^{k}$ are independent;
    \item for any $1\leq \ell\leq k$,  $\tilde Y_\ell$ and $Y_\ell$ have the same distribution; 
    \item for any $1\leq \ell\leq k$, $\PP(\tilde Y_\ell \neq Y_\ell) = \beta(\sigma(\{Y_{\ell'}\}_{\ell'=1}^{\ell-1}), \sigma(\{Y_\ell\}))$. 
\end{enumerate}
\end{lemma}
\begin{proof}
    See Lemma 2.10 in \cite{barrera2021generalization} for a detailed proof. 
\end{proof}

We introduce the following Hoeffding's Inequality and Bernstein's Inequality for geometrically ergodic non-stationary sequences. 

\begin{theorem}[Hoeffding's Inequality for geometrically ergodic non-stationary sequences]\label{thm:hoeffding-mixing}
We denote by $\{X_t\}_{t\geq 0}\subseteq \cX$ a Markov chain satisfying Assumption \ref{ass:mc-ass}. Then for any function $f\colon \cX\to [-f_{\max}, f_{\max}]$, it holds with probability at least $1 - \delta$ with $c / (NT)^2 \leq \delta \leq 1$ that 
\$
\left| \frac{1}{NT} \sum_{i\in [N]}\sum_{t = 0}^{T-1} f(X_t^i) - \EE\left[ \frac{1}{T} \sum_{t = 0}^{T-1} f(X_t) \right]  \right| \leq c\cdot f_{\max} \sqrt{\frac{1}{NT\kappa} \log\frac{2}{\delta}\log(NT)}, 
\$
where $\{\{X_t^i\}_{t=0}^{T-1}\}_{i\in [N]}$ consists of $N$ i.i.d. trajectories with length $T > 0$ generated from the same distribution as $\{X_t\}_{t\geq 0}$. 
\end{theorem}
\begin{proof}
    See \S\ref{prf:thm:hoeffding-mixing} for a detailed proof. 
\end{proof}

\begin{theorem}[Bernstein's Inequality for geometrically ergodic non-stationary sequences]\label{thm:bernstein-mixing}
We denote by $\{X_t\}_{t\geq 0}\subseteq \cX$ a Markov chain satisfying Assumption \ref{ass:mc-ass}. Then for any function $f\colon \cX\to [-f_{\max}, f_{\max}]$, it holds with probability at least $1 - \delta$ with $c / (NT)^2 \leq \delta \leq 1$ that 
\$
& \frac{1}{NT} \sum_{i\in [N]}\sum_{t = 0}^{T-1} f(X_t^i) - \EE\left[ \frac{1}{T} \sum_{t = 0}^{T-1} f(X_t) \right]  \\
& \qquad \leq c_1 \cdot \frac{ f_{\max}}{NT\kappa} \log \frac{2}{\delta} \log (NT) + c_2 \cdot \sqrt{\frac{1}{NT\kappa} \EE\left[ \frac{1}{T} \sum_{t = 0}^{T-1} f(X_t)^2 \right] \log\frac{2}{\delta} \log(NT) }, 
\$
where $c_1$ and $c_2$ are positive absolute constants, and $\{\{X_t^i\}_{t=0}^{T-1}\}_{i\in [N]}$ consists of $N$ i.i.d. trajectories with length $T > 0$ generated from the same distribution as $\{X_t\}_{t\geq 0}$. 
\end{theorem}
\begin{proof}
    See \S\ref{prf:thm:bernstein-mixing} for a detailed proof. 
\end{proof}

\subsection{Empirical Processes for Geometrically Ergodic Non-Stationary Sequence}

For any conditional probabilities $p_1(y\given x)$ and $p_2(y\given x)$ such that $(x,y)\in \cX\times \cY$, we define the squared Hellinger distance as follows, 
\$
h^2\left( p_1(\cdot \given x), p_2(\cdot \given x) \right) = \frac{1}{2} \int \left( \sqrt{p_1(y\given x)} - \sqrt{p_2(y\given x)} \right)^2 \ud y.
\$
We further assume that $\cY$ is a discrete space. 
We denote by $p^*(y\given x)$ the true conditional probability of $y\in \cY$ given $x\in \cX$. Also, let $\{(X_t, Y_t)\}_{t\geq 0}\subset \cX\times \cY$ be a Markov chain such that $Y_t \sim p^*(\cdot \given X_t)$ and satisfies Assumption \ref{ass:mc-ass}. 
Further, we denote by $\mu_t$ the marginal distribution of $X_t$ for any $t\geq 0$. 
In the meanwhile, with $\mu = 1/T\cdot \sum_{t = 0}^{T-1}\mu_t$, we define the generalized squared Hellinger distance over $\mu$ as follows, 
\$
H^2(p_1, p_2) = \EE_{X\sim \mu}\left[ h^2\left( p_1(\cdot \given X), p_2(\cdot \given X) \right) \right]. 
\$
In the meanwhile, we are given a data set $\{\{(X_t^i, Y_t^i)\}_{t=0}^{T-1}\}_{i\in[N]}$ consisting of $N$ independent trajectories of length $T$, where $\{(X_t^i, Y_t^i)\}_{t=0}^{T-1}$ is generated from the same distribution as $\{(X_t, Y_t)\}_{t\geq 0}$.  
We construct the following maximum likelihood estimator for $p^*$, 
\#\label{eq:def-p-hat}
\hat p \in \argmax_{p\in \cP} \hat \EE\left[ \log p(Y\given X) \right] = \frac{1}{NT}\sum_{i\in[N]} \sum_{t = 0}^{T-1} \log p(Y_t^i\given X_t^i). 
\#
We also define 
\$
g_p(x,y) = \frac{1}{2} \log \frac{p(y\given x) + p^*(y\given x)}{2 p^*(y\given x)}
\$
for any $(x,y)\in \cX \times \cY$. 

Now, we are ready to introduce the following lemma.

\begin{lemma}
\label{lemma:vi-vdg-lemma}
We have
\$
& H^2\left(\frac{\hat p + p^*}{2}, p^*\right) \leq \left( \hat \EE - \EE \right)\left[g_{\hat p}(X,Y)\right], \\
& H^2\left(\frac{p_1 + p^*}{2}, \frac{p_2 + p^*}{2}\right) \leq \frac{1}{2} H^2(p_1, p_2), \\
& H^2(p,p^*) \leq 16 H^2\left( \frac{p + p^*}{2}, p^* \right), \\
& \left\|p_1(\cdot \given x) - p_2(\cdot \given x) \right\|_1 \leq 2\sqrt{2} h\left( p_1(\cdot \given x), p_2(\cdot \given x) \right). 
\$
\end{lemma}
\begin{proof}
See \S\ref{prf:lemma:vi-vdg-lemma} for a detailed proof.
\end{proof}

We define the entropy integral as follows, 
\$
J_B(\delta, \overline\cP^{1/2}(\delta)) = \max\left\{ \int_{\delta^2 / 2^{10}}^\delta \left( H_{B}(u, \overline\cP^{1/2}(\delta)) \right)^{1/2} \ud u, \delta \right \}, 
\$
where $H_{B}(u, \overline\cP^{1/2}(\delta))$ is the entropy of the space $\overline\cP^{1/2}(\delta)$ with bracketing, and $\overline\cP^{1/2}(\delta)$ is defined as follows, 
\$
\overline\cP^{1/2}(\delta) = \{\overline{p}^{1/2}\colon p\in \cP \text{ and } H^2(\overline{p}, p^*) \leq \delta^2 \}. 
\$
Now, we introduce the following theorem, which upper bounds the distance between $\hat p$ and $p^*$. 
\begin{theorem}
\label{thm:iv-vdg-bound}
We take $\Psi(\delta) \geq J_B(\delta, \overline\cP^{1/2}(\delta))$ in such a way that $\Psi(\delta) / \delta^2$ is a non-increasing function of $\delta$. Then for a universal constant $c$ and any $\delta \geq \delta_{NT}$, where $\delta_{NT}$ satisfies that $\sqrt{NT} \delta_{NT}^2 \geq c \Psi(\delta_{NT})$, it holds with probability at least $1 - c/\kappa \cdot \exp(-NT \kappa\delta^2/(c^2 \log(NT))) - c / (N^2 T^2) \cdot \log(4/\delta)$ that
\$
H^2(\hat p, p^*) \leq \delta^2, 
\$
where $\hat p$ is defined in \eqref{eq:def-p-hat}. 
\end{theorem}
\begin{proof}
See \S\ref{prf:thm:iv-vdg-bound} for a detailed proof. 
\end{proof}

We study the following case, where $\cP$ is a parametric class.

\begin{corollary}[Parametric Class]\label{cor:vdg-param}
Suppose $\cP = \{p_\theta\colon \theta\in \RR^d \text{ and } \|\theta\|_2 \leq \theta_{\max}\}$. Then with probability at least $1 - \delta$ with $c / (N^2 T^2) \cdot \log (NT) \leq \delta \leq 1$, we have
\$
H^2(\hat p, p^*) \leq c \cdot \frac{d}{NT\kappa} \log\frac{\theta_{\max}}{\delta} \log(NT),
\$
where $c>0$ is an absolute constant, which may vary from lines to lines.
\end{corollary}
\begin{proof}
Note that 
\$
J_B(\delta, \overline\cP^{1/2}(\delta), d) \leq \delta \sqrt{d \log \frac{\theta_{\max}}{\delta}}. 
\$ 
By taking $\Psi(\delta) = \delta \sqrt{d \log(\theta_{\max}/\delta)}$, we have 
\$
\PP\left( H^2(\hat p, p^*) \leq c \cdot \frac{d}{NT \kappa} \log\frac{\theta_{\max}}{\delta} \log(NT) \right) \geq 1 - \delta
\$
with $c / (N^2 T^2) \cdot \log (NT) \leq \delta \leq 1$, which concludes the proof of the corollary. 
\end{proof}

\section{Proofs of Auxiliary Results}\label{sec:prf:aux}

\subsection{Proof of Theorem \ref{thm:hoeffding-mixing}}\label{prf:thm:hoeffding-mixing}
\begin{proof}
We take $\tau = \min\{T, 3/\kappa \cdot \log(NT)\}$, and denote by $\EE_\stat[\cdot]$ the expectation taken with respect to the stationary distribution of $\{X_t\}_{t\geq 0}$. 
We observe the following decomposition, 
\$
&  \frac{1}{NT} \sum_{i\in [N]}\sum_{t = 0}^{T-1} f(X_t^i) - \EE\left[ \frac{1}{T} \sum_{t = 0}^{T-1} f(X_t) \right]   \\
& \qquad = \frac{\tau}{T} \underbrace{ \left( \frac{1}{N\tau} \sum_{i\in [N]}\sum_{t = 0}^{\tau-1} f(X_t^i) - \EE\left[\frac{1}{\tau} \sum_{t = 0}^{\tau-1} f(X_t)\right] \right)}_{\text{(I)}} ]  \\
& \qquad\qquad  + \frac{T - \tau}{T} \underbrace{ \left( \frac{1}{N(T - \tau)} \sum_{i\in [N]}\sum_{t = \tau}^{T-1} f(X_t^i) - \EE_\stat\left[f(X)\right] \right) }_{\text{(II)}}  \\ 
& \qquad \qquad + \frac{T -\tau}{T} \underbrace{ \left( \EE_\stat\left[f(X)\right] - \EE\left[\frac{1}{T - \tau} \sum_{t = \tau}^{T-1} f(X_t)\right] \right)}_{\text{(III)}}. 
\$
In the follows, we upper bound terms (I), (II), and (III), respectively. 

\vskip5pt
\noindent\textbf{Upper Bounding Term (I).}
Since $\{\{X_t^i\}_{t=0}^{T-1}\}_{i\in [N]}$ are i.i.d. accross each trajectory, by standard Hoeffding's inequality, with probability at least $1 - \delta$, we have
\#\label{eq:hoeffding-1}
\left|\text{(I)}\right| \leq f_{\max} \sqrt{\frac{2}{N} \log\frac{2}{\delta}}. 
\#

\vskip5pt
\noindent\textbf{Upper Bounding Term (II).}
We consider an auxiliary Markov chain $\{\{\tilde X_t^i\}_{t=0}^{T-1}\}_{i\in [N]}$, where the $i$-th trajectory $\{\tilde X_t^i\}_{t=0}^{T-1}$ is sampled such that $\tilde X_0^i\sim p_\stat$. Here $p_\stat$ is the stationary distribution of $\{X_t\}_{t\geq 0}$. Similarly, we define the following quantity, 
\$
\text{($\tilde{\text{II}}$)} = \frac{1}{N(T - \tau)} \sum_{i\in [N]}\sum_{t = \tau}^{T-1} f(\tilde X_t^i) - \EE_\stat\left[f(X)\right].
\$
Now, for any $x\geq 0$, we upper bound the difference $\PP( \text{(II)} \geq x) - \PP(\text{($\tilde{\text{II}}$)} \geq x)$ as follows,
\#\label{eq:hhhhh1}
\PP\left( \text{(II)}  \geq x \right) - \PP\left( \text{($\tilde{\text{II}}$)}  \geq x\right) \leq N \sum_{t = \tau}^{T-1} \EE \left[\| p_t(\cdot \given X_0) - p_\stat(\cdot ) \|_\tv \right] \leq c\cdot NT \exp(-\kappa \tau). 
\#
Thus, by \eqref{eq:hhhhh1}, to upper bound $\PP( \text{(II)} \geq x )$, it suffices to upper bound $\PP(\text{($\tilde{\text{II}}$)} \geq x)$. 

To upper bound $\PP( \text{($\tilde{\text{II}}$)} \geq x)$, we take $T -\tau = 2 k s$, where $k$ and $s$ are two positive integers for the simplicity of presentation. We partition the set $\{\tau, \tau+1, \ldots, T-1\}$ as follows, 
\$
& J_1 = \{\tau, \tau+1, \ldots, \tau+s-1\}, \quad J_2 = \{\tau+s, \tau+s+1, \ldots, \tau+2s-1\}, \quad \ldots,  \\
& J_{2k-1} = \{T-2s, T-2s+1, \ldots, T-s-1\}, \quad J_{2k} = \{T-s, T-s+1, \ldots, T-1\}. 
\$
Under such a partition, we see that $\cup_{\ell\in[2k]}J_\ell = \{\tau, \tau+1, \ldots, T-1\}$ and $J_\ell\cap J_{\ell'} = \emptyset$ for any $\ell\neq \ell'$. Also, for any $i\in [N]$, we define 
\$
Z_\ell^i = (\tilde X_t^i)_{t\in J_\ell}
\$
for any $\ell\in [2k]$. Now, for any $i\in [N]$, by Lemma \ref{lemma:berbee}, there exists a sequence $\{W_\ell^i\}_{\ell\in [2k]}$, where $W_\ell^i = (\tilde Y_t^i)_{t\in J_\ell}$ such that
\begin{enumerate}
    \item $\{W_\ell^i\}_{\ell\in[2k]}$ are independent;
    \item for any $\ell\in[2k]$,  $W_\ell^i$ and $Z_\ell^i$ have the same distribution; 
    \item for any $\ell\in[2k]$, $\PP(W_\ell^i \neq Z_\ell^i) = \beta(\sigma(\{Z_{\ell'}^i\}_{\ell'\in[\ell-1]}), \sigma(\{Z_\ell^i\}))$. 
\end{enumerate}
Note that the following inclusion relation holds, 
\$
\left\{ \frac{1}{s}\sum_{t\in J_\ell} f(\tilde X_t^i) - \EE_\stat[f(X)] \geq x_\ell \right\} \subseteq \left\{ \frac{1}{s}\sum_{t\in J_\ell} f(\tilde Y_t^i) - \EE_\stat[f(X)] \geq x_\ell \right\} \cup \{W_\ell^i \neq Z_\ell^i\}
\$
for any $x_\ell\in \RR$. 
Thus, we have 
\$
\PP\left(\text{($\tilde{\text{II}}$)} \geq x\right) & \leq \PP\left( \frac{1}{kN} \sum_{i\in[N], \text{$\ell$ is odd}} \frac{1}{s}\sum_{t\in J_\ell} f(\tilde X_t^i) - \EE_\stat[f(X)] \geq x\right)  \\
&  \qquad + \PP\left( \frac{1}{kN} \sum_{i\in[N], \text{$\ell$ is even}} \frac{1}{s}\sum_{t\in J_\ell} f(\tilde X_t^i) - \EE_\stat[f(X)] \geq x\right)  \\
&  \leq \PP\left( \frac{1}{kN} \sum_{i\in[N], \text{$\ell$ is odd}} \frac{1}{s}\sum_{t\in J_\ell} f(\tilde Y_t^i) - \EE_\stat[f(X)] \geq x\right)  \\
&  \qquad + \PP\left( \frac{1}{kN} \sum_{i\in[N], \text{$\ell$ is even}} \frac{1}{s}\sum_{t\in J_\ell} f(\tilde Y_t^i) - \EE_\stat[f(X)] \geq x\right) + \sum_{\ell\in [2k]} \PP(W_\ell^i \neq Z_\ell^i)  \\
&  \leq 2 \exp\left(-\frac{kNx^2}{2f_{\max}^2}\right) + 2k \beta(s), 
\$
where we use Hoeffding's inequality in the last line. 
Similarly, we have 
\$
\PP\left(\text{($\tilde{\text{II}}$)} \leq -x\right) \leq 2 \exp\left(-\frac{kNx^2}{2f_{\max}^2}\right) + 2k \beta(s). 
\$
Thus, we have 
\$
\PP\left(  \left| \text{(${\text{II}}$)} \right| \geq x\right) \leq 4 \exp\left(-\frac{kNx^2}{2f_{\max}^2}\right) + 4k \beta(s) + c\cdot NT \exp(-\kappa \tau). 
\$
Now, by taking $s = 3\log( NT) / \kappa$, it holds with probability at least $1 - \delta$ with $c / (NT)^2 \leq \delta \leq 1$ that 
\#\label{eq:hoeffding-2}
\left| \text{(${\text{II}}$)} \right| \leq f_{\max} \sqrt{\frac{24}{NT\kappa} \log\frac{4}{\delta}\log(NT)}. 
\#

\vskip5pt
\noindent\textbf{Upper Bounding Term (III).}
We observe that
\#\label{eq:hoeffding-3}
\left |\text{(III)} \right | \leq f_{\max} \cdot \sum_{t = \tau}^{T-1} c\cdot \exp(-\kappa t) \leq c\cdot \frac{f_{\max}}{N^2T^2}. 
\#

\vskip5pt
\noindent\textbf{Combining Everything.}
Now, by combining \eqref{eq:hoeffding-1}, \eqref{eq:hoeffding-2}, and \eqref{eq:hoeffding-3}, it holds with probability at least $1 - \delta$ with $c / (NT)^2 \leq \delta \leq 1$ that
\$
\frac{1}{NT} \sum_{i\in [N]}\sum_{t = 0}^{T-1} f(X_t^i) - \EE\left[ \frac{1}{T} \sum_{t = 0}^{T-1} f(X_t) \right] \leq c\cdot f_{\max} \sqrt{\frac{48}{NT \kappa} \log\frac{4}{\delta}\log(NT)}, 
\$
which concludes the proof of the theorem.
\end{proof}

\subsection{Proof of Theorem \ref{thm:bernstein-mixing}}\label{prf:thm:bernstein-mixing}
\begin{proof}
The proof follows from the proof of Theorem \ref{thm:hoeffding-mixing} in \S\ref{prf:thm:hoeffding-mixing}. For the completeness of the paper, we present it here. 
We take $\tau = \min\{T, 3/\kappa\cdot \log(NT)\}$, and denote by $\EE_\stat[\cdot]$ the expectation taken with respect to the stationary distribution of $\{X_t\}_{t\geq 0}$. 
We observe the following decomposition, 
\$
&  \frac{1}{NT} \sum_{i\in [N]}\sum_{t = 0}^{T-1} f(X_t^i) - \EE\left[ \frac{1}{T} \sum_{t = 0}^{T-1} f(X_t) \right]   \\
& \qquad = \frac{\tau}{T} \underbrace{ \left( \frac{1}{N\tau} \sum_{i\in [N]}\sum_{t = 0}^{\tau-1} f(X_t^i) - \EE\left[\frac{1}{\tau} \sum_{t = 0}^{\tau-1} f(X_t)\right] \right)}_{\text{(I)}} ] \\
& \qquad \qquad + \frac{T - \tau}{T} \underbrace{ \left( \frac{1}{N(T - \tau)} \sum_{i\in [N]}\sum_{t = \tau}^{T-1} f(X_t^i) - \EE_\stat\left[f(X)\right] \right) }_{\text{(II)}}  \\ 
& \qquad \qquad + \frac{T -\tau}{T} \underbrace{ \left( \EE_\stat\left[f(X)\right] - \EE\left[\frac{1}{T - \tau} \sum_{t = \tau}^{T-1} f(X_t)\right] \right)}_{\text{(III)}}. 
\$
In the follows, we upper bound terms (I), (II), and (III), respectively. 

\vskip5pt
\noindent\textbf{Upper Bounding Term (I).}
Since $\{\{X_t^i\}_{t=0}^{T-1}\}_{i\in [N]}$ are i.i.d. across each trajectory, by standard Bernstein's inequality, with probability at least $1 - \delta$, we have
\#\label{eq:bern-1}
\left|\text{(I)}\right| & \leq \frac{2f_{\max}}{3N}\log\frac{2}{\delta} + 4 \sqrt{ \frac{4}{N} \EE\left[\left(\frac{1}{\tau} \sum_{t = 0}^{\tau-1} f(X_t) \right)^2 \right] \log\frac{2}{\delta} }  \\
& = \frac{2f_{\max}}{3N}\log\frac{2}{\delta} + 4 \sqrt{ \frac{4}{N} \EE\left[ \frac{1}{\tau} \sum_{t = 0}^{\tau-1} f(X_t)^2  \right] \log\frac{2}{\delta} }.
\#

\vskip5pt
\noindent\textbf{Upper Bounding Term (II).}
We consider an auxiliary dataset $\{\{\tilde X_t^i\}_{t=0}^{T-1}\}_{i\in [N]}$, where the $i$-th trajectory $\{\tilde X_t^i\}_{t=0}^{T-1}$ is sampled such that $\tilde X_0^i\sim p_\stat$. Here $p_\stat$ is the stationary distribution of $\{X_t\}_{t\geq 0}$. Similarly, we define the following quantity, 
\$
\text{($\tilde{\text{II}}$)} = \frac{1}{N(T - \tau)} \sum_{i\in [N]}\sum_{t = \tau}^{T-1} f(\tilde X_t^i) - \EE_\stat\left[f(X)\right].
\$
Now, for any $x\geq 0$, we upper bound the difference $\PP( \text{(II)} \geq x) - \PP(\text{($\tilde{\text{II}}$)} \geq x)$ as follows,
\#\label{eq:hhhhh1bern}
\PP\left( \text{(II)}  \geq x \right) - \PP\left( \text{($\tilde{\text{II}}$)}  \geq x\right) \leq N \sum_{t = \tau}^{T-1} \EE \left[\| p_t(\cdot \given X_0) - p_\stat(\cdot ) \|_\tv \right] \leq c\cdot NT \exp(-\kappa \tau). 
\#
Thus, by \eqref{eq:hhhhh1bern}, to upper bound $\PP( \text{(II)} \geq x )$, it suffices to upper bound $\PP(\text{($\tilde{\text{II}}$)} \geq x)$. 

To upper bound $\PP( \text{($\tilde{\text{II}}$)} \geq x)$, we take $T -\tau = 2 k s$, where $k$ and $s$ are two positive integers for the simplicity of presentation. We partition the set $\{\tau, \tau+1, \ldots, T-1\}$ as follows, 
\$
& J_1 = \{\tau, \tau+1, \ldots, \tau+s-1\}, \quad J_2 = \{\tau+s, \tau+s+1, \ldots, \tau+2s-1\}, \quad \ldots,  \\
& J_{2k-1} = \{T-2s, T-2s+1, \ldots, T-s-1\}, \quad J_{2k} = \{T-s, T-s+1, \ldots, T-1\}. 
\$
Under such a partition, we see that $\cup_{\ell\in[2k]}J_\ell = \{\tau, \tau+1, \ldots, T-1\}$ and $J_\ell\cap J_{\ell'} = \emptyset$ for any $\ell\neq \ell'$. Also, for any $i\in [N]$, we define 
\$
Z_\ell^i = (\tilde X_t^i)_{t\in J_\ell}
\$
for any $\ell\in [2k]$. Now, for any $i\in [N]$, by Lemma \ref{lemma:berbee}, there exists a sequence $\{W_\ell^i\}_{\ell\in [2k]}$, where $W_\ell^i = (\tilde Y_t^i)_{t\in J_\ell}$ such that
\begin{enumerate}
    \item $\{W_\ell^i\}_{\ell\in[2k]}$ are independent;
    \item for any $\ell\in[2k]$,  $W_\ell^i$ and $Z_\ell^i$ have the same distribution; 
    \item for any $\ell\in[2k]$, $\PP(W_\ell^i \neq Z_\ell^i) = \beta(\sigma(\{Z_{\ell'}^i\}_{\ell'\in[\ell-1]}), \sigma(\{Z_\ell^i\}))$. 
\end{enumerate}
Note that the following inclusion relation holds, 
\$
\left\{ \frac{1}{s}\sum_{t\in J_\ell} f(\tilde X_t^i) - \EE_\stat[f(X)] \geq x_\ell \right\} \subseteq \left\{ \frac{1}{s}\sum_{t\in J_\ell} f(\tilde Y_t^i) - \EE_\stat[f(X)] \geq x_\ell \right\} \cup \{W_\ell^i \neq Z_\ell^i\}
\$
for any $x_\ell\in \RR$. 
Thus, we have 
\#\label{eq:bern1111}
\PP\left(\text{($\tilde{\text{II}}$)} \geq x\right) & \leq \PP\left( \frac{1}{kN} \sum_{i\in[N], \text{$\ell$ is odd}} \frac{1}{s}\sum_{t\in J_\ell} f(\tilde X_t^i) - \EE_\stat[f(X)] \geq x\right)  \\
&  \qquad + \PP\left( \frac{1}{kN} \sum_{i\in[N], \text{$\ell$ is even}} \frac{1}{s}\sum_{t\in J_\ell} f(\tilde X_t^i) - \EE_\stat[f(X)] \geq x\right)  \\
&  \leq \PP\left( \frac{1}{kN} \sum_{i\in[N], \text{$\ell$ is odd}} \frac{1}{s}\sum_{t\in J_\ell} f(\tilde Y_t^i) - \EE_\stat[f(X)] \geq x\right)  \\
&  \qquad + \PP\left( \frac{1}{kN} \sum_{i\in[N], \text{$\ell$ is even}} \frac{1}{s}\sum_{t\in J_\ell} f(\tilde Y_t^i) - \EE_\stat[f(X)] \geq x\right) + \sum_{\ell\in [2k]} \PP(W_\ell^i \neq Z_\ell^i)  \\
&  \leq \exp\left(-\frac{3 k^2 N^2 x^2 }{ 6 \sum_{i\in [N], \text{$\ell$ is odd}} \EE\left[ \left( \frac{1}{s}\sum_{t\in J_\ell} f(\tilde Y_t^i) \right)^2 \right] + 2f_{\max} Nk x }\right)  \\
& \qquad + \exp\left(-\frac{3 k^2 N^2 x^2 }{ 6 \sum_{i\in [N], \text{$\ell$ is even}} \EE\left[ \left( \frac{1}{s}\sum_{t\in J_\ell} f(\tilde Y_t^i) \right)^2 \right] + 2f_{\max} Nk x }\right) + 2k \beta(s), 
\#
where we use Bernstein's inequality in the last line. Note that for any $(i, \ell)\in [N]\times[2k]$, we have
\#\label{eq:bern1112}
\EE\left[ \left( \frac{1}{s}\sum_{t\in J_\ell} f(\tilde Y_t^i) \right)^2 \right] = \EE_\stat\left[f(X)^2\right] \leq \EE\left[ \frac{1}{T-\tau} \sum_{t = \tau}^{T-1} f(X_t)^2 \right] + f_{\max}^2 T c\cdot \exp(-\kappa \tau). 
\#
Combining \eqref{eq:bern1111} and \eqref{eq:bern1112}, we have 
\$
\PP\left(\text{($\tilde{\text{II}}$)} \geq x\right) & \leq 2 \exp\left(-\frac{3 k^2 N^2 x^2 }{ 6 kN \left( \EE\left[ \frac{1}{T-\tau} \sum_{t = \tau}^{T-1} f(X_t)^2 \right] + f_{\max}^2 \beta(\tau) \right) + 2f_{\max} Nk x }\right)  + 2k \beta(s). 
\$
Similarly, we have 
\$
\PP\left(\text{($\tilde{\text{II}}$)} \leq -x\right) \leq 2 \exp\left(-\frac{3 k^2 N^2 x^2 }{ 6 kN \left( \EE\left[ \frac{1}{T-\tau} \sum_{t = \tau}^{T-1} f(X_t)^2 \right] + f_{\max}^2 \beta(\tau) \right) + 2f_{\max} Nk x }\right) + 2k \beta(s). 
\$
Thus, we have 
\$
\PP\left(  \left| \text{(${\text{II}}$)} \right| \leq x\right) & \leq 4 \exp\left(-\frac{3 k^2 N^2 x^2 }{ 6 kN \left( \EE\left[ \frac{1}{T-\tau} \sum_{t = \tau}^{T-1} f(X_t)^2 \right] + f_{\max}^2 \beta(\tau) \right) + 2f_{\max} Nk x }\right) \\
& \qquad \qquad  + 4k \beta(s) + c\cdot NT \exp(-\kappa \tau). 
\$
Now, by taking $s = 3\log( NT) / \kappa$, it holds with probability at least $1 - \delta$ with $c / (NT)^2 \leq \delta \leq 1$ that 
\#\label{eq:bern-2}
\left| \text{(${\text{II}}$)} \right| \leq \frac{2 f_{\max}}{3NT \kappa} \log \frac{2}{\delta} + 4\sqrt{\frac{4}{NT \kappa} \EE\left[ \frac{1}{T-\tau} \sum_{t = \tau}^{T-1} f(X_t)^2 \right] \log\frac{2}{\delta} }. 
\#

\vskip5pt
\noindent\textbf{Upper Bounding Term (III).}
We observe that
\#\label{eq:bern-3}
\left |\text{(III)} \right | \leq f_{\max} \cdot \sum_{t = \tau}^{T-1} c\cdot \exp(-\kappa t) \leq c\cdot \frac{f_{\max}}{N^2T^2}. 
\#

\vskip5pt
\noindent\textbf{Combining Everything.}
Now, by combining \eqref{eq:bern-1}, \eqref{eq:bern-2}, and \eqref{eq:bern-3}, it holds with probability at least $1 - \delta$ with $c / (NT)^2 \leq \delta \leq 1$ that
\$
& \frac{1}{NT} \sum_{i\in [N]}\sum_{t = 0}^{T-1} f(X_t^i) - \EE\left[ \frac{1}{T} \sum_{t = 0}^{T-1} f(X_t) \right]  \\
& \qquad \leq \frac{24 f_{\max}}{NT\kappa} \log \frac{2}{\delta} \log (NT) + 48\sqrt{\frac{1}{NT\kappa} \EE\left[ \frac{1}{T} \sum_{t = 0}^{T-1} f(X_t)^2 \right] \log\frac{2}{\delta} \log(NT) }, 
\$
which concludes the proof of the theorem.
\end{proof}

\subsection{Proof of Lemma \ref{lemma:vi-vdg-lemma}}
\label{prf:lemma:vi-vdg-lemma}
\begin{proof}
The proof follows from the proofs of Lemmas 4.1 and 4.2 in \cite{geer2000empirical}.

\noindent\textbf{First Inequality.}
By the optimality of $\hat p$, we have
\$
\hat \EE\left[ \log \hat p(Y\given X) \right] \geq \hat \EE\left[ \log p^*(Y\given X) \right], 
\$
which implies that 
\$
\int \log \frac{\hat p}{p^*} \ud \tilde p^* \ud \tilde \mu \geq 0,
\$
where $\tilde p^*$ and $\tilde \mu$ are empirical counterparts of $p^*$ and $\mu$. Now, by the concavity of $\log(\cdot)$, we have
\$
\log \frac{\hat p + p^*}{2p^*} \geq \frac{1}{2} \log\frac{\hat p}{p^*} + \frac{1}{2} \log\frac{p^*}{p^*} = \frac{1}{2} \log\frac{\hat p}{p^*}. 
\$
By the above two inequalities, we have
\#\label{iv-vd-as1}
0 & \leq \frac{1}{4} \int \log \frac{\hat p}{p^*} \ud \tilde p^* \ud \tilde \mu \leq \frac{1}{2} \int \log \frac{\hat p + p^*}{2p^*} \ud \tilde p^* \ud \tilde \mu \\
& = \frac{1}{2} \int \log \frac{\hat p + p^*}{2p^*} \left (\ud \tilde p^* \ud \tilde \mu - \ud p^* \ud \mu \right) + \frac{1}{2} \int \log \frac{\hat p + p^*}{2p^*} \ud p^* \ud \mu \\
& = \left( \hat \EE - \EE \right)\left[g_{\hat p}\right] + \frac{1}{2} \int \log \frac{\hat p + p^*}{2p^*} \ud p^* \ud \mu. 
\#
In the meanwhile, by the fact that $\log z \leq 2(\sqrt{z}- 1)$ for any $z > 0$, we have
\#\label{iv-vd-as2}
\frac{1}{2} \int \log \frac{\hat p + p^*}{2p^*} \ud p^* \ud \mu & \leq \int \left ( \sqrt{\frac{\hat p + p^*}{2p^*}} -1 \right) \ud p^* \ud\mu  \\
& = \int \left ( \sqrt{\frac{\hat p + p^*}{2} \cdot p^*} - \frac{1}{2}p^* - \frac{1}{2} \cdot \frac{\hat p + p^*}{2}\right) \ud y \ud \mu  \\
& = -\int \frac{1}{2} \left( \sqrt{\frac{\hat p + p^*}{2}} - \sqrt{p^*} \right )^2 \ud y \ud \mu  \\
& = - H^2 \left( \frac{\hat p + p^*}{2}, p^* \right). 
\#
By combining \eqref{iv-vd-as1} and \eqref{iv-vd-as2}, we conclude the proof of the first inequality. 

\vskip5pt
\noindent\textbf{Second\&Third Inequality.}
We denote by $\overline{p} = (p + p^*) / 2$ for any $p$. We note the following two facts, 
\$
& \frac{p_1^{1/2} + p_2^{1/2}}{\overline p_1^{1/2} + \overline p_2^{1/2}} \leq \sqrt{2}, \\
& \left | \overline p_1^{1/2} - \overline p_2^{1/2} \right | \left ( \overline p_1^{1/2} + \overline p_2^{1/2} \right ) = \left | \overline p_1 - \overline p_2 \right | = \left | \frac{p_1 - p_2}{2}\right | = \frac{1}{2} \left | p_1^{1/2} - p_2^{1/2}\right | \left ( p_1^{1/2} + p_2^{1/2}\right ). 
\$
Thus, we have
\$
\left | \overline p_1 ^{1/2} - \overline p_2^{1/2} \right | = \frac{1}{2} \cdot \frac{p_1^{1/2} + p_2^{1/2}}{\overline p_1^{1/2} + \overline p_2^{1/2}} \cdot \left | p_1^{1/2} - p_2^{1/2} \right | \leq \frac{\sqrt{2}}{2} \cdot \left| p_1^{1/2} - p_2^{1/2} \right |, 
\$
which implies the second inequality. The third inequality can be proved in a similar way. 

\vskip5pt
\noindent\textbf{Forth Inequality.}
We note that 
\$
\left\|p_1(\cdot \given x) - p_2(\cdot \given x) \right\|_1 & = \int \left | p_1(y \given x) - p_2(y \given x) \right | \ud y  \\
& = \int \left ( p_1(y \given x)^{1/2} - p_2(y \given x)^{1/2} \right )\left ( p_1(y \given x)^{1/2} + p_2(y \given x)^{1/2} \right ) \ud y  \\
& = \sqrt{\int \left ( p_1(y \given x)^{1/2} - p_2(y \given x)^{1/2} \right )\ud y } \sqrt{\int \left ( p_1(y \given x)^{1/2} + p_2(y \given x)^{1/2} \right ) \ud y}  \\
& \leq 2 \sqrt{\int \left ( p_1(y \given x)^{1/2} - p_2(y \given x)^{1/2} \right )\ud y } = 2\sqrt{2} h\left( p_1(\cdot \given x), p_2(\cdot \given x) \right), 
\$
which concludes the proof. 
\end{proof}

\subsection{Proof of Theorem \ref{thm:iv-vdg-bound}} 
\label{prf:thm:iv-vdg-bound}
\begin{proof}
The proof follows from the proof of Theorem 7.4 in \cite{geer2000empirical}.
We define the events
\$
\cE = \left\{ \omega \in \Omega\colon H^2(\hat p, p^*) > \delta^2 \right\}. 
\$
Conditioning on $\cE$, we have
\#\label{eq:iv-388383}
\left(\hat \EE - \EE \right)[g_{\hat p}] \geq H^2(\overline{\hat p}, p^*) \geq \frac{1}{16} H^2(\hat p, p^*) > \frac{\delta^2}{16}, 
\#
where the first two inequalities come from Lemma \ref{lemma:vi-vdg-lemma}. 
We further define
\$
\cE^\dagger = \left\{ \omega \in \Omega\colon \sup_{p\in \cP\colon H^2(\overline{p},p^*) > \delta^2/16} \left(\hat \EE - \EE \right)[g_p] - H^2(\overline{p}, p^*) \geq 0 \right\}.
\$
By \eqref{eq:iv-388383} and the definitions of $\cE$ and $\cE^\dagger$, we observe that $\cE \subseteq \cE^\dagger$. Thus, we only need to upper bound $\PP(\cE^\dagger)$. To do so, we use a peeling argument as follows. Let $L = \min\{\ell\colon 2^{\ell+1}\delta^2 / 16 > 1\}$. We observe that
\#\label{eq:8943594}
\PP(\cE^\dagger) \leq \sum_{\ell=0}^L \PP(\cE^\dagger_\ell),
\#
where
\$
\cE^\dagger_\ell = \left\{ \omega\in \Omega\colon \sup_{p\in \cP_\ell } \left(\hat \EE[g_p] - \EE[g_p] \right) \geq 2^\ell \delta^2 / 16 \right\}. 
\$
Here $\cP_\ell = \{p\in \cP\colon H^2(\overline{p},p^*) \leq 2^{\ell+1}\delta^2/16\}$. 
To upper bound $\PP(\cE_\ell^\dagger)$, we introduce the following result. 

\begin{theorem}
\label{thm:511-adapt}
Given a Markov chain $\{Z_t\}_{t\geq 0}\subset \cZ$ satisfying Assumption \ref{ass:mc-ass}, and take 
\begin{align}
& v \leq C_1\sqrt{NT} R^2 /K, \label{eq:vd-cond1}\\
& v \leq 8\sqrt{NT} R, \label{eq:vd-cond2}\\
& v \geq C_0\cdot \max\left\{ \int_{v/(2^6\sqrt{NT})}^R \left( \mathcal{H}_{B,K}(u,\cG,P) \right)^{1/2} \ud u, R \right\}, \label{eq:vd-cond3} \\
& v \geq C_2 / (NT)^2, \label{eq:vd-cond5} \\
& C_0^2 \geq C^2 (C_1 + 1), \label{eq:vd-cond4}
\end{align}
where $\mathcal{H}_{B,K}(u,\cG,P)$ is the generalized entropy with bracketing. 
Then we have
\$
& \PP\left( \sup_{g\in \cG} \sqrt{NT} \left( \frac{1}{NT} \sum_{i\in [N]} \sum_{t = 0}^{T-1} g(Z_t^i) - \EE\left[ \frac{1}{T} \sum_{t = 0}^{T-1} g(Z_t) \right] \right) \geq v \right)  \\
& \qquad \leq \frac{4C}{\kappa}\exp\left(-\frac{v^2 \kappa}{18 C^2(C_1+1)R^2 \log(NT)}\right) + \frac{2}{N^2 T^2}, 
\$
where $\{Z_t^i\}_{t=0}^{T-1}$ is generated from the same distribution as $\{Z_t\}_{t\geq 0}$ for any $i\in [N]$. 
\end{theorem}
\begin{proof}
See \S\ref{prf:thm:511-adapt} for a detailed proof. 
\end{proof}

To invoke Theorem \ref{thm:511-adapt}, we take
\$
v = \sqrt{NT} \cdot 2^\ell \delta^2 / 16, \quad K = 1, \quad R = 2^{\ell/2}\delta, \quad C_1 = 15, \quad C = c/64, \quad C_0 = c/16, \quad C_2 = c. 
\$
It is easy to verify that \eqref{eq:vd-cond1}, \eqref{eq:vd-cond2}, \eqref{eq:vd-cond5}, and \eqref{eq:vd-cond4} hold. For \eqref{eq:vd-cond3}, since $\sqrt{NT}\delta_{NT}^2 \geq c \Psi(\delta_{NT})$, which implies that 
\$
\sqrt{NT} \geq c \cdot \frac{ \Psi(\delta_{NT})}{\delta_{NT}^2} \geq c \cdot \frac{\Psi(2^{\ell/2}\delta)}{2^\ell \delta^2},
\$
where we use the fact that $\Psi(\delta) / \delta^2$ is a non-increasing function of $\delta$. Thus, we have 
\$
16 a \geq c \cdot \max \left\{ \int_{v/(2^6\sqrt{NT})}^R \left(\mathcal{H}_{B,1}\left(u,\{g_p\colon p\in \cP_\ell \}, \mu_0 \right) \right)^{1/2} \ud u , R \right\}, 
\$
which justifies \eqref{eq:vd-cond3} by noting that $K = 1$. Here, we use the fact that 
\$
\mathcal{H}_{B,1}(u, \{g_p\colon p\in \cP_\ell \}, P) \leq H_B\left(\frac{u}{\sqrt{2}} , \{\overline{p}^{1/2}\colon p\in \cP_\ell\} \right). 
\$
Thus, by using Theorem \ref{thm:511-adapt}, we have
\$
\PP(\cE_\ell^\dagger) \leq \frac{c}{\kappa} \exp\left(-\frac{NT \kappa 2^\ell \delta^2}{c^2 \log(NT)}\right) + \frac{2}{N^2 T^2}. 
\$
Further, by combining \eqref{eq:8943594}, we have
\$
\PP(\cE^\dagger) \leq \frac{c}{\kappa} \exp\left(-\frac{NT \kappa \delta^2}{c^2}\right) + \frac{c}{N^2 T^2} \log\frac{4}{\delta},
\$
which concludes the proof of the theorem. 
\end{proof}

\subsection{Proof of Theorem \ref{thm:511-adapt}}\label{prf:thm:511-adapt}
\begin{proof}
We take $\tau = \min\{T, 3/\kappa \cdot \log(\cG_{\max} NT)\}$, where $\cG_{\max} = \max\{\max_{g\in \cG}\max_{z\in \cZ} g(z), 1\}$, and denote by $\EE_\stat[\cdot]$ the expectation taken with respect to the stationary distribution of $\{Z_t\}_{t\geq 0}$. We have the following decomposition, 
\#\label{eq:van-flee-1}
& \PP\left( \sup_{g\in \cG} \sqrt{NT} \left( \frac{1}{NT} \sum_{i\in [N]} \sum_{t = 0}^{T-1} g(Z_t^i) - \EE\left[ \frac{1}{T} \sum_{t = 0}^{T-1} g(Z_t) \right] \right) \geq v \right)  \\
& \qquad = \PP\Biggl( \sup_{g\in \cG} \frac{\tau}{T} \biggl( \frac{1}{N\tau} \sum_{i\in [N]} \sum_{t = 0}^{\tau-1} g(Z_t^i) - \EE\biggl[ \frac{1}{\tau} \sum_{t = 0}^{\tau-1} g(Z_t) \biggr] \biggr)   \\
& \qquad \qquad \qquad + \frac{T-\tau}{T} \biggl( \frac{1}{N(T-\tau)} \sum_{i\in [N]} \sum_{t = \tau}^{T-1} g(Z_t^i) - \EE_\stat[g(Z)] \biggr)  \\
& \qquad \qquad \qquad  + \frac{T-\tau}{T} \biggl( \EE_\stat[g(Z)] - \EE\biggl[ \frac{1}{T-\tau} \sum_{t = \tau}^{T-1} g(Z_t) \biggr] \biggr) \geq \frac{v}{\sqrt{NT}} \Biggr)  \\
& \qquad \leq \PP\left( \sup_{g\in \cG} \frac{\tau}{T} \left( \frac{1}{N\tau} \sum_{i\in [N]} \sum_{t = 0}^{\tau-1} g(Z_t^i) - \EE\left[ \frac{1}{\tau} \sum_{t = 0}^{\tau-1} g(Z_t) \right] \right) \geq \frac{v}{3\sqrt{NT}} \right)  \\
& \qquad \qquad + \PP\left( \sup_{g\in \cG} \frac{T-\tau}{T} \left( \frac{1}{N(T - \tau)} \sum_{i\in [N]} \sum_{t = \tau}^{T-1} g(Z_t^i) - \EE_\stat\left[ g(Z) \right] \right) \geq \frac{v}{3\sqrt{NT}} \right), 
\#
where the last inequality comes from the fact that 
\$
& \sup_{g\in \cG} \frac{T-\tau}{T} \left( \EE_\stat[g(Z)] - \EE\left[ \frac{1}{T-\tau} \sum_{t = \tau}^{T-1} g(Z_t) \right] \right)  \\
& \qquad \leq \sup_{g\in \cG} \frac{T-\tau}{T} \int g(z) \frac{1}{T-\tau} \sum_{t = \tau}^{T-1} \left( p_\stat(z) - \int p_t(z\given z_0) \ud \zeta(z_0) \right) \ud z  \\
& \qquad \leq \frac{1}{T} \cG_{\max} \sum_{t = \tau}^{T-1} c\cdot \exp(-\kappa t) \leq \cG_{\max} c \exp(-\kappa\tau) \leq \frac{v}{3\sqrt{NT}}, 
\$
where we use the fact that $v \geq C_2 / (NT)^2$ for some constant $C_2$. Thus, we only need to upper bound the two terms on the RHS of \eqref{eq:van-flee-1}. We first introduce the following supporting results. 

\begin{lemma}
\label{lemma:511-geer}
Take 
\$
& v \leq C_1\sqrt{n} R^2 /K,\\
& v \leq 8\sqrt{n} R, \\
& v \geq C_0\cdot \max\left\{ \int_{v/(2^6\sqrt{n})}^R \left( \log \mathcal N_{B,K}(u,\cG,P) \right)^{1/2} \ud u, R \right\}, \\
& C_0^2 \geq C^2 (C_1 + 1). 
\$
Then we have
\$
\PP\left( \sup_{g\in \cG} \left| \sqrt{n} \left( \frac{1}{n} \sum_{i\in [n]} g(Z_i) - \EE\left[g(Z)\right] \right) \right| \geq v \right) \leq C \exp\left( -\frac{v^2}{C^2(C_1 + 1)R^2} \right), 
\$
where $\{Z_i\}_{i\in [n]}$ are i.i.d. samples drawn the same distribution as $Z$. 
\end{lemma}
\begin{proof}
See Theorem 5.11 in \cite{geer2000empirical} for a detailed proof. 
\end{proof}

\begin{lemma}
\label{lemma:211-general}
Given a $\beta$-mixing sequence $\{Z_t\}_{t\geq 0}\subset \cZ$ with coefficient $\beta(t)$ for any $t \geq 0$. There exists a sequence $\{Z_t^*\}_{t= 0}^{T-1}\subset \cZ$ and a set $\cJ$ such that
\begin{enumerate}
    \item $\cJ$ is a partition of $\{0,1,\ldots, T-1\}$, i.e., $\cup_{J\in \cJ} J = \{0, 1, \ldots, T-1\}$ and $J_1\cap J_2 = \emptyset$ for any $J_1,J_2 \in \cJ$; 
    \item for any $0\leq t \leq T-1$, $Z_t^*$ and $Z_t$ have the same distribution;
    \item for any $J\in \cJ$, $\{Z_t^*\}_{t\in J}$ is an independent sequence;
    \item it holds for any $u \in \RR$ that 
            \$
            & \PP\left(\sup_{g\in \cG} \frac{1}{T} \sum_{t=0}^{T-1} g(Z_t) - \EE\left[\frac{1}{T} \sum_{t=0}^{T-1} g(Z_t)\right] \geq u \right) \\
            & \qquad \leq \sum_{J\in \cJ} \PP\left(\sup_{g\in \cG} \frac{1}{|J|} \sum_{t\in J} g(Z_t^*) - \EE\left[\frac{1}{|J|} \sum_{t\in J} g(Z_t)\right] \geq u \right)  \\
            & \qquad \qquad + \sum_{J\in \cJ} |J| \cdot \beta\left(\min\{ |t_1-t_2|\colon t_1\neq t_2\in J \}\right).
            \$
\end{enumerate}
\end{lemma}
\begin{proof}
    See Theorem 2.11 in \cite{barrera2021generalization} for a detailed proof. 
\end{proof}

We upper bound two terms on the RHS of \eqref{eq:van-flee-1} as follows. 

\vskip5pt
\noindent\textbf{Upper Bounding the First Term on the RHS of \eqref{eq:van-flee-1}.}
To upper bound the first term, we invoke Lemma \ref{lemma:511-geer}. Since the sequence
\$
\left\{ \frac{1}{\tau} \sum_{t = 0}^{\tau-1} g(Z_t^i) \right\}_{i\in [n]}
\$
is i.i.d., we have 
\#\label{eq:van-flee-res-1}
& \PP\left( \sup_{g\in \cG} \frac{\tau}{T} \left( \frac{1}{N\tau} \sum_{i\in [N]} \sum_{t = 0}^{\tau-1} g(Z_t^i) - \EE\left[ \frac{1}{\tau} \sum_{t = 0}^{\tau-1} g(Z_t) \right] \right) \geq \frac{v}{3\sqrt{NT}} \right)  \\
& \qquad \leq C \exp\left( -\frac{v^2T}{9\tau^2 C^2 (C_1 + 1)R^2} \right) \leq C \exp\left( -\frac{v^2}{C^2 (C_1 + 1)R^2} \right). 
\#

\vskip5pt
\noindent\textbf{Upper Bounding the Second Term on the RHS of \eqref{eq:van-flee-1}.}
To upper bound the second term, we note that
\$
& \PP\left( \sup_{g\in \cG} \frac{T-\tau}{T} \left( \frac{1}{N(T - \tau)} \sum_{i\in [N]} \sum_{t = \tau}^{T-1} g(Z_t^i) - \EE_\stat\left[ g(Z) \right] \right) \geq \frac{v}{3\sqrt{NT}} \right) = \text{(i)} + \text{(ii)}, 
\$
where 
\$
& \text{(i)} = \PP\left( \sup_{g\in \cG} \frac{T-\tau}{T} \left( \frac{1}{N(T - \tau)} \sum_{i\in [N]} \sum_{t = \tau}^{T-1} g(Z_t^i) - \EE_\stat\left[ g(Z) \right] \right) \geq \frac{v}{3\sqrt{NT}} \right)  \\
& \qquad \qquad - \PP\left( \sup_{g\in \cG} \frac{T-\tau}{T} \left( \frac{1}{N(T - \tau)} \sum_{i\in [N]} \sum_{t = \tau}^{T-1} g(\tilde Z_t^i) - \EE_\stat\left[ g(Z) \right] \right) \geq \frac{v}{3\sqrt{NT}} \right),  \\
& \text{(ii)} = \PP\left( \sup_{g\in \cG} \frac{T-\tau}{T} \left( \frac{1}{N(T - \tau)} \sum_{i\in [N]} \sum_{t = \tau}^{T-1} g(\tilde Z_t^i) - \EE_\stat\left[ g(Z) \right] \right) \geq \frac{v}{3\sqrt{NT}} \right). 
\$
Here $\{\tilde Z_t^i\}_{t = 0}^{T-1}$ are an auxiliary sequence for any $i\in [N]$, where $\tilde Z_0^i$ is sampled from the stationary distribution of the sequence $\{Z_t\}_{t\geq 0}$. To upper bound $\text{(i)}$, we note that 
\#\label{eq:van-flee-res-2-1}
\text{(i)} \leq \sum_{i\in [N]} \sum_{t=\tau}^{T-1} c\cdot \exp(-\kappa t) \leq NT c \exp(-\kappa \tau) \leq \frac{1}{N^2 T^2}. 
\#
To upper bound $\text{(ii)}$, we invoke Lemma \ref{lemma:211-general} by taking 
\$
& \cJ = \{J_1, J_2, \ldots, J_s\}, \qquad J_j = \{\tau+j-1, \tau+j+s-1, \ldots, T-s+j\} \text{ for any $j\in [s]$}. 
\$
Then there exists a sequence $\{\{\tilde Z_t^{i*}\}_{t = \tau}^{T-1}\}_{i\in [N]}$ such that $\tilde Z_t^{i*}$ and $\tilde Z_t^i$ have the same distribution for any $(i,t)$; $\{\{\tilde Z_t^*\}_{t\in J}\}_{i\in [N]}$ are independent; and it holds for any $u \in \RR$ that 
\$
\text{(ii)} & \leq \sum_{j = 1}^s \PP\left(\sup_{g\in \cG} \frac{s}{N(T-\tau)} \sum_{i\in[N]} \sum_{t\in J_j} g(\tilde Z_t^{i*}) - \EE_\stat\left[ g(Z) \right] \geq \frac{v}{3\sqrt{NT}} \frac{T}{T-\tau} \right) + (T-\tau) \cdot \beta(s)  \\
& \leq s\cdot C\exp\left(-\frac{v^2}{9s C^2(C_1+1)R^2}\right) + (T-\tau)\beta(s),
\$
where we use Lemma \ref{lemma:511-geer} in the last inequality. 
Now, by taking $s = \min\{T, 3/\kappa\cdot \log( NT)\}$, we have 
\#\label{eq:van-flee-res-2-2}
\text{(ii)} \leq \frac{3C}{\kappa} \exp\left(-\frac{v^2 \kappa}{18 C^2(C_1+1)R^2 \log(NT)}\right) + \frac{1}{N^2 T^2}, 
\#
where we use Lemma \ref{lemma:davydov-adapt} to upper bound $\beta(s)$. 
Now, by combining \eqref{eq:van-flee-res-2-1} and \eqref{eq:van-flee-res-2-2}, we have
\#\label{eq:van-flee-res-2}
& \PP\left( \sup_{g\in \cG} \frac{T-\tau}{T} \left( \frac{1}{N(T - \tau)} \sum_{i\in [N]} \sum_{t = \tau}^{T-1} g(Z_t^i) - \EE_\stat\left[ g(Z) \right] \right) \geq \frac{v}{3\sqrt{NT}} \right)  \\
& \qquad \leq \frac{3C}{\kappa}\exp\left(-\frac{v^2 \kappa}{18 C^2(C_1+1)R^2 \log(NT) }\right) + \frac{2}{N^2 T^2}. 
\#

\vskip5pt
\noindent\textbf{Combining Everything.}
By plugging \eqref{eq:van-flee-res-1} and \eqref{eq:van-flee-res-2} into \eqref{eq:van-flee-1}, we have
\$
& \PP\left( \sup_{g\in \cG} \sqrt{NT} \left( \frac{1}{NT} \sum_{i\in [N]} \sum_{t = 0}^{T-1} g(Z_t^i) - \EE\left[ \frac{1}{T} \sum_{t = 0}^{T-1} g(Z_t) \right] \right) \geq v \right)  \\
& \qquad \leq \frac{4C}{\kappa}\exp\left(-\frac{v^2 \kappa}{18 C^2(C_1+1)R^2 \log(NT)}\right) + \frac{2}{N^2 T^2}, 
\$
which concludes the proof of the theorem. 
\end{proof}

\section{Numerical Simulation}\label{sec:exppppp}

We demonstrate the usefulness of the proposed methods by conducting a simulation study, in which we simulate a synthetic dataset that  mimics a real-world electronic medical record dataset for kidney transplantation patients \citep{hua2020personalized}. Kidney transplantation is the primary treatment for patients with  chronic kidney disease or end-stage renal disease \citep{Arshad2019}. After the transplant surgery, patients are usually instructed to have regular clinical visits for their long-term care. At each visit, patients' creatinine levels, an important biomarker for measuring kidney function, are measured. Then based on patients' creatinine levels, physicians prescribe immunosuppressive drugs, such as tacrolimus, to keep their immune systems from rejecting the new kidney \citep{Kasiske2010}. Due to potential compliance and resistance issues, patients' whole blood tacrolimus concentrations are also measured so that their body responses to tacrolimus can be monitored. Our goal is estimate the optimal therapeutic strategies (i.e., the optimal tacrolimus concentrations) for patients after kidney transplantation from observational data collected in the electronic medical database.

Under such a context, for any $t \geq 0$, at the $t$-th clinical visit, we denote by $S_t$ the patient's creatinine level, $Z_t$ the assigned dosage of the immunosuppressive drug tacrolimus, $A_t$ the actual effective dosage level (i.e., whole blood tacrolimus concentration), which can be different from $Z_t$ due to compliance/resistance issues, $U_t$ the unobserved confounders, such as the quality of care, and $R_t$ the reward. We introduce the simulation setup in details as follows.

\vskip5pt
\noindent\textbf{Dynamics and Rewards.}  
For simplicity,  we consider $S_t\in \RR$ and $U_t\in \RR$ for any $t$. In the meanwhile, we assume that the IV and action spaces are ternary, i.e., $\cA = \cZ = \{0,1,2\}$, which represents low, medium, and high  dosage/concentration level, respectively. 
Specifically, given the current state $S_t$ and action $A_t$, we assume that the reward $R_t$ and next state $S_{t+1}$ satisfy the following equations, 
\$
R_t = - S_t^2 + (U_t-2)\cdot (A_t-1), \qquad S_{t+1} = S_t + 0.5 \cdot (A_t - 1) + 3\cdot \ind\{S_t> 0\} \cdot (U_t-2), 
\$
where $U_t \sim \mathcal N(2, 0.1)$ is the unobserved confounder at the $t$-th step. The term $- S_t^2$ in defining $R_t$ reflects our clinical knowledge that creatinine levels that are either very high or too low can be harmful for patients. 
Also, we take the initial state $S_0\sim \mathcal N(5, 0.1)$. 

\vskip5pt
\noindent\textbf{IV and Behavior Policy.} 
We assume that the IV $Z_t$ takes a value in $\cZ$ with certain probabilities. Specifically, given the current state $S_t$, the IV $Z_t$ is taken as follows. 
\begin{itemize}
    \item[(i)] if $S_t < -0.3$, we take $Z_t = z$ with probability $p_z$ for any $z\in \cZ$, where $(p_0, p_1, p_2) = (0.1,0.1,0.8)$; 
    \item[(ii)] if $S_t > 0.3$, we take $Z_t = z$ with probability $p_z$ for any $z\in \cZ$, where $(p_0, p_1, p_2) = (0.8,0.1,0.1)$; 
    \item[(iii)] if $-0.3\leq S_t \leq 0.3$, we take $Z_t = z$ with probability $p_z$ for any $z\in \cZ$, where $(p_0, p_1, p_2) = (0.1,0.8,0.1)$. 
\end{itemize}

As for the behavior policy, given $S_t$, $U_t$, and $Z_t$, the action $A_t$ is taken as follows. 
\begin{itemize}
    \item if $U_t > 2$: 
    \begin{itemize}
        \item[(i)] if $Z_t=0$, we take $A_t = a$ with probability $p_a$ for any $a\in \cA$, where $(p_0, p_1, p_2) = (0.8,0.1,0.1)$; 
        \item[(ii)] if $Z_t=1$, we take $A_t = a$ with probability $p_a$ for any $a\in \cA$, where $(p_0, p_1, p_2) = (0.1,0.8,0.1)$; 
        \item[(iii)] if  $Z_t=2$, we take $A_t = a$ with probability $p_a$ for any $a\in \cA$, where $(p_0, p_1, p_2) = (0.1,0.1,0.8)$; 
    \end{itemize}
    \item if $U_t \leq 2$: 
    \begin{itemize}
        \item[(iv)] if $Z_t=0$, we take $A_t = a$ with probability $p_a$ for any $a\in \cA$, where $(p_0, p_1, p_2) = (0.78, 0.11, 0.11)$; 
        \item[(v)] if $Z_t=1$, we take $A_t = a$ with probability $p_a$ for any $a\in \cA$, where $(p_0, p_1, p_2) = (0.05, 0.78, 0.17)$; 
        \item[(vi)] if $Z_t=2$, we take $A_t = a$ with probability $p_a$ for any $a\in \cA$, where $(p_0, p_1, p_2) = (0.11, 0.05, 0.84)$. 
    \end{itemize}
\end{itemize}

\vskip5pt 
\noindent\textbf{Simulation Setup.}
Throughout the experiment, we consider simplex encoding of $\cA$ and $\cZ$ as in \eqref{eq:iv-simplex-encoding}. 
Under the aforementioned setting, we generate $N=1000$ trajectories with a finite horizon $T=100$ following the behavior policy. We take the discount factor $\gamma = 0.9$. 
For the simplicity of simulation, we parameterize $\cV$, $\cW$, and $\Pi$ as follows, 
\#\label{eq:iv-linear-param}
& \cV = \left\{ v_{\omega_{\textsf{v}}}(\cdot) = \psi(\cdot)^\top \omega_{\textsf{v}} \colon \omega_{\textsf{v}} \in \RR^5 \right\},  \\
& \cW = \left\{ g_{\omega_{\textsf{g}}}(\cdot) = \psi(\cdot)^\top \omega_{\textsf{g}} \colon \omega_{\textsf{g}} \in \RR^5, \|\omega_{\textsf{g}}\|_\infty \leq 1 \right\},  \\ 
& \Pi = \left\{ \pi_{\omega_{\textsf{pi}}} \colon \pi_{\omega_{\textsf{pi}}}(a\given \cdot) \propto \exp(\psi(\cdot)^\top \omega_{a, \textsf{pi}}) \text{ where }  \omega_{a, \textsf{pi}} \in \RR^5 \text{ for any $a\in \cA$} \right\}, 
\#
where $\psi(s) = (1, s, s^2, s^3, s^4)^\top$ is the feature vector for any $s\in \RR$. For the simplicity of the simulation, we assume that there exists an oracle that gives $\Delta^*(s,a)$, $\Theta^*(s,z)$, and $\PP(A=a\given S=s)$ for any $(s,z,a)\in \cS\times \cZ\times \cA$. Such an oracle can be achieved by logistic regression. 
Under the linear parameterization in \eqref{eq:iv-linear-param}, we construct (pessimistic) estimators of the expected total reward using the following methods.

\begin{figure}
    \centering
    \includegraphics[width=0.7\textwidth]{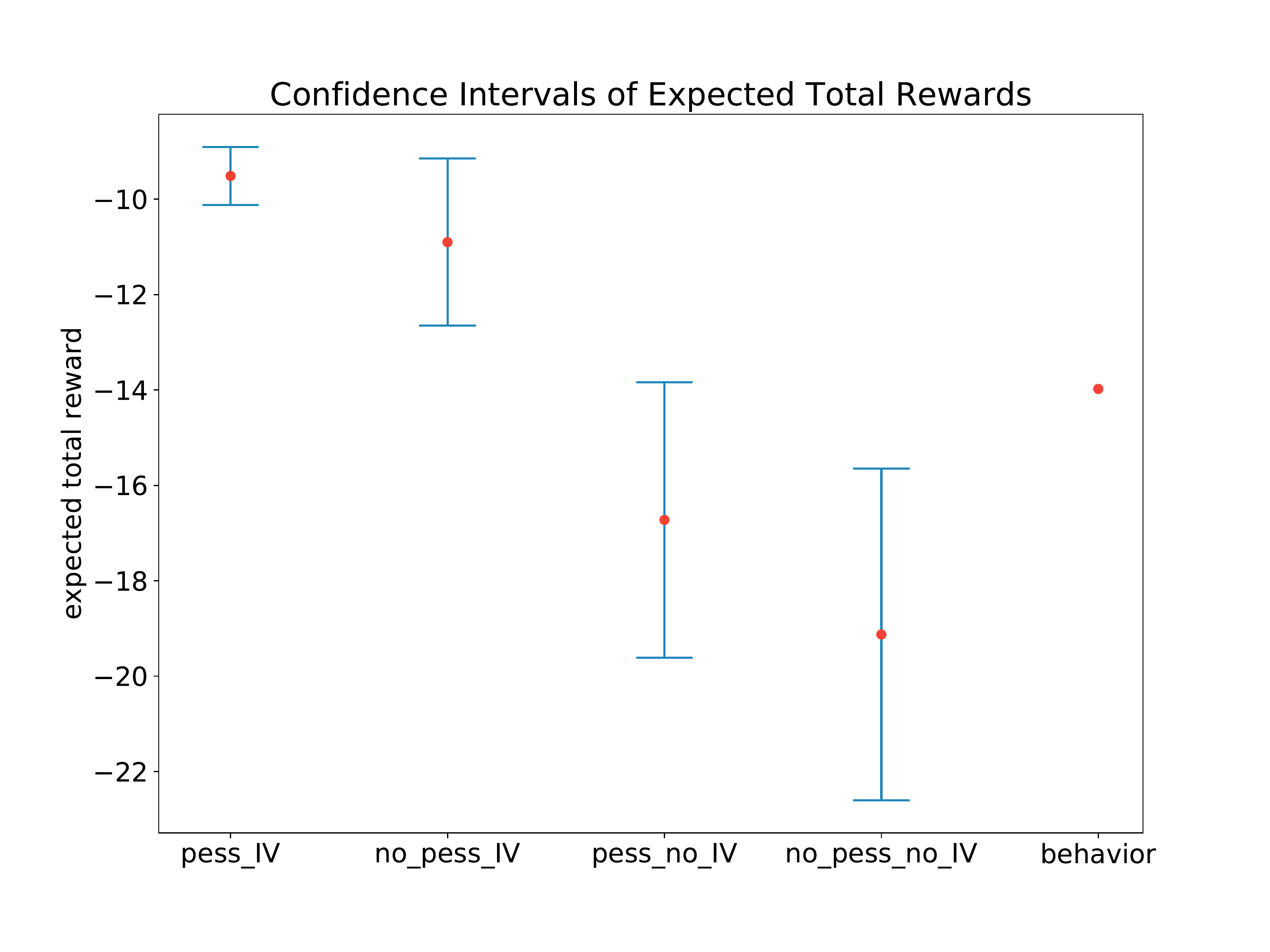}
    \caption{Expected total rewards by using \texttt{pess\_IV}, \texttt{no\_pess\_IV}, \texttt{pess\_no\_IV}, and \texttt{no\_pess\_no\_IV}. Here, the red dots represent the means, while the blue intervals represent the 95\% confidence intervals generated by 25 seeds. }
    \label{fig:exp-iv}
\end{figure}

\begin{itemize}
    \item \texttt{pess\_IV}. For any $\pi$, by adding random noise to the rewards, we formulate the following optimization problems, 
    \$
    \omega^j_{\textsf{v}} \in \argmin_{v\in \cV} \max_{g\in \cW} \frac{1}{NT} \sum_{i = 1}^N  \sum_{t = 0}^{T-1} g(S_t^i) \left ( \frac{Z_t^{i \top} A_t^i \pi(A_t^i\given S_t^i)}{\Delta^*(S_t^i,A_t^i) \Theta^*(S_t^i,Z_t^i)} \left(\tilde R_t^{i,j} + \gamma v(S_{t+1}^i)\right) - v(S_t^i) \right ), 
    \$
    where $\tilde R_t^{i,j} = R_t^i + \chi_t^{i,j}$ and $\chi_t^{i,j} \sim \mathcal N(0, 0.1)$ for any $(t,i,j)\in [T]\times [N]\times [10]$. Then we take $(1-\gamma)\cdot \EE_{S_0\sim \nu}[v_{\omega^{j^*}_{\textsf{v}}}(S_0)]$ as a pessimistic estimator of the total expected reward, where $j^*$ is taken such that $\EE_{S_0\sim \nu}[v_{\omega^{j^*}_{\textsf{v}}}(S_0)]$ achieves the minimum among all $j\in [10]$. 
    \item \texttt{no\_pess\_IV}. For any $\pi$, we solve the following optimization problem, 
    \$
    \omega_{\textsf{v}} \in \argmin_{v\in \cV} \max_{g\in \cW} \frac{1}{NT} \sum_{i = 1}^N  \sum_{t = 0}^{T-1} g(S_t^i) \left ( \frac{Z_t^{i \top} A_t^i \pi(A_t^i\given S_t^i)}{\Delta^*(S_t^i,A_t^i) \Theta^*(S_t^i,Z_t^i)} \left(R_t^i + \gamma v(S_{t+1}^i)\right) - v(S_t^i) \right ). 
    \$
    Then we take $(1-\gamma)\cdot \EE_{S_0\sim \nu}[v_{\omega_{\textsf{v}}}(S_0)]$ as an estimator of the total expected reward. 
    \item \texttt{pess\_no\_IV}. For any $\pi$, without using IVs, by adding random noise to the rewards, we formulate the following optimization problems, 
    \$
    \omega^j_{\textsf{v}} \in \argmin_{v\in \cV} \max_{g\in \cW} \frac{1}{NT} \sum_{i = 1}^N  \sum_{t = 0}^{T-1} g(S_t^i) \left ( \frac{\pi(A_t^i\given S_t^i)}{\PP(A_t^i\given S_t^i)} \cdot \left(\tilde R_t^{i,j} + \gamma v(S_{t+1}^i)\right) - v(S_t^i) \right ), 
    \$
    where $\tilde R_t^{i,j} = R_t^i + \chi_t^{i,j}$ and $\chi_t^{i,j} \sim \mathcal N(0, 0.1)$ for any $(t,i,j)\in [T]\times [N]\times [10]$. Then we take $(1-\gamma)\cdot \EE_{S_0\sim \nu}[v_{\omega^{j^*}_{\textsf{v}}}(S_0)]$ as a pessimistic estimator of the total expected reward, where $j^*$ is taken such that $\EE_{S_0\sim \nu}[v_{\omega^{j^*}_{\textsf{v}}}(S_0)]$ achieves the minimum among all $j\in [10]$. 
    \item \texttt{no\_pess\_no\_IV}. For any $\pi$, without using IVs, we solve the following optimization problem, 
    \$
    \omega_{\textsf{v}} \in \argmin_{v\in \cV} \max_{g\in \cW} \frac{1}{NT} \sum_{i = 1}^N  \sum_{t = 0}^{T-1} g(S_t^i) \left ( \frac{\pi(A_t^i\given S_t^i)}{\PP(A_t^i\given S_t^i)} \cdot \left(R_t^i + \gamma v(S_{t+1}^i)\right) - v(S_t^i) \right ). 
    \$
    Then we take $(1-\gamma)\cdot \EE_{S_0\sim \nu}[v_{\omega_{\textsf{v}}}(S_0)]$ as an estimator of the total expected reward. 
\end{itemize}

Finally, by zero-th order optimization, we update the policy to maximize the above estimators. We repeat the above procedure for 25 times, and plot the 95\% confidence intervals of the total expected rewards of the output policies in Figure \ref{fig:exp-iv}. According to the figure, we observe that with the presence of IVs, \texttt{pess\_IV} and \texttt{no\_pess\_IV} are capable of learning better policies than the behavior policy, with higher and more stable expected total rewards in \texttt{pess\_IV} due to the use of pessimism; while without IVs, \texttt{pess\_no\_IV} and \texttt{no\_pess\_no\_IV} even fail to beat the behavior policy.

\end{document}